\documentclass{article}%
\usepackage[T1]{fontenc}%
\usepackage{lmodern}%
\usepackage{textcomp}%
\usepackage{lastpage}
\usepackage{float}%
\usepackage[dvipsnames]{xcolor}
\usepackage{geometry}
\usepackage{graphicx}%
\usepackage[colorlinks=true,linkcolor=blue,citecolor=blue,pdfencoding=auto]{hyperref}%
\usepackage{caption}%
\usepackage{natbib}%
\usepackage{amsmath}%
\usepackage{amssymb}%
\usepackage{mathtools}%
\usepackage{booktabs}%
\usepackage{epstopdf}%
\usepackage{adjustbox}%
\usepackage{array}%
\usepackage{lscape}%
\usepackage{titling}%
\usepackage{amsthm}
\usepackage{xargs}
\usepackage{xcolor}
\usepackage{authblk}
\usepackage{enumerate}
\usepackage{subcaption}
\usepackage[colorinlistoftodos,prependcaption,textsize=small,color=red!25]{todonotes}

\usepackage[linesnumbered,ruled,vlined]{algorithm2e}

\DeclarePairedDelimiter\norm{\lVert}{\rVert}

\graphicspath{{images/}}

\DeclareMathOperator{\p}{\mathbb{P}}
\DeclareMathOperator{\X}{\mathcal{X}}
\DeclareMathOperator{\N}{\mathcal{N}}

\DeclareMathOperator{\bE}{\mathbb{E}}

\DeclareMathOperator{\Pp}{\mathbb{P}}
\DeclareMathOperator{\var}{Var}
\DeclareMathOperator{\one}{\mathbb{I}}

\newcommand{\indep}{\perp\!\!\!\perp}
\newcommand{\bbracket}[1]{\ensuremath{\left({#1}\right)}}
\newcommand{\E}[2]{\ensuremath{\mathbb{E}_{#1}\left[{#2}\right]}}
\newcommand{\condE}[3]{\ensuremath{\mathbb{E}_{#1}\left[\left.{#2}\right|{#3}\right]}}
\newcommand{\Var}[2]{\ensuremath{\mathrm{Var}_{#1}\left({#2}\right)}}
\newcommand{\Cov}[2]{\ensuremath{\mathrm{Cov}\left({#1}, {#2}\right)}}

\newcommand{\EE}{\mathbb{E}}
\newcommand{\hGamma}{\widehat{\Gamma}}

\theoremstyle{plain}
\newtheorem{proposition}{Proposition}
\newtheorem{lemma}{Lemma}
\newtheorem{theorem}{Theorem}

\newtheorem{assumption}{Assumption}

\newcounter{remark}
\newenvironment{remark}[1][]{\refstepcounter{remark}\par\medskip
   \noindent \textsc{Remark~\theremark. #1} \rmfamily}{\medskip}
   
\newcounter{example}[section]
\newenvironment{example}[1][]{\refstepcounter{example}\par\medskip
   \noindent \textsc{Example~\theexample. #1} \rmfamily}{\medskip}

\newcommand{\StableVar}{\texttt{StableVar} }
\newcommand{\minvar}{\texttt{MinVar} }

\title{Off-Policy Evaluation via Adaptive Weighting with Data from Contextual Bandits~\thanks{This paper has been accepted in the Proceedings of the 27th ACM SIGKDD Conference on Knowledge Discovery and Data Mining (KDD '21) on May 17, 2021. The published version includes a short appendix with sketches of the proofs, while the appendix in this version includes complete proofs.}
}

\makeatletter
\newcommand{\printfnsymbol}[1]{%
  \textsuperscript{\@fnsymbol{#1}}%
}
\makeatother

\author{Ruohan Zhan\thanks{Institute for Computational and Mathematical Engineering, Stanford University.} ~~ Vitor Hadad\thanks{Graduate School of Business, Stanford University.} ~~ David A. Hirshberg\printfnsymbol{3} ~~ Susan Athey\printfnsymbol{3}
}

\begin{document}

\maketitle

\begin{abstract}
\noindent It has become increasingly common for data to be collected adaptively, for example using contextual bandits. Historical data of this type can be used to evaluate
other treatment assignment policies to guide future innovation or experiments. However, policy evaluation is challenging if the target policy differs from the one used to collect data, and popular estimators, including doubly robust (DR) estimators, can be plagued by bias, excessive variance, or both. 
In particular, when the pattern of treatment assignment in the collected data looks little like the pattern generated by the policy to be evaluated,
the importance weights used in DR estimators explode, leading to excessive variance. 

In this paper, we improve the DR estimator by adaptively weighting observations to control its variance.
We show that a $t$-statistic based on our improved estimator is asymptotically normal under certain conditions, 
allowing us to form confidence intervals and test hypotheses. 
Using synthetic data and public benchmarks, we provide empirical evidence for our estimator's improved accuracy and inferential properties 
relative to existing alternatives. 
\end{abstract}


\section{Introduction}\label{sec:introduction}
Off-policy evaluation is the problem of estimating the benefits of one treatment assignment policy using historical data that was collected using another. For example, in personalized healthcare we may wish to use historical data to evaluate how particular groups of patients will respond to a given  treatment regime for the design of future clinical trials \citep{murphy2003optimal}; in targeted advertising one may want to understand how alternative advertisements perform for different consumer segments \citep{li2011unbiased}. The estimation challenge arises since an individual's outcome is observed only for the assigned treatments, so counterfactual outcomes for alternative treatments are not observed. 
There is a large literature on this problem in the cases where historical observations, which are collected by one or multiple policies, are independent from one another~\citep{dudik2011doubly,imbens2015causal,agarwal2017effective,kallus2020optimal}. However, it has been increasingly common for data to be collected in adaptive experiments, for example using \emph{contextual bandit} algorithms \citep[e.g.,][]{agrawal2013thompson,lattimore2018bandit, russo2017tutorial}. Contextual bandits trade off exploration and exploitation in an attempt to learn the policy that best targets the treatment assignment to an observation's context (e.g., a patient's characteristics). In these experiments, assignment probabilities depend on the context in a way that is updated over time as the algorithm learns from past data. In this paper, we focus on off-policy evaluation using data collected by contextual bandits.   

\begin{figure}[t]
    \centering
    \includegraphics[width=.6\linewidth]{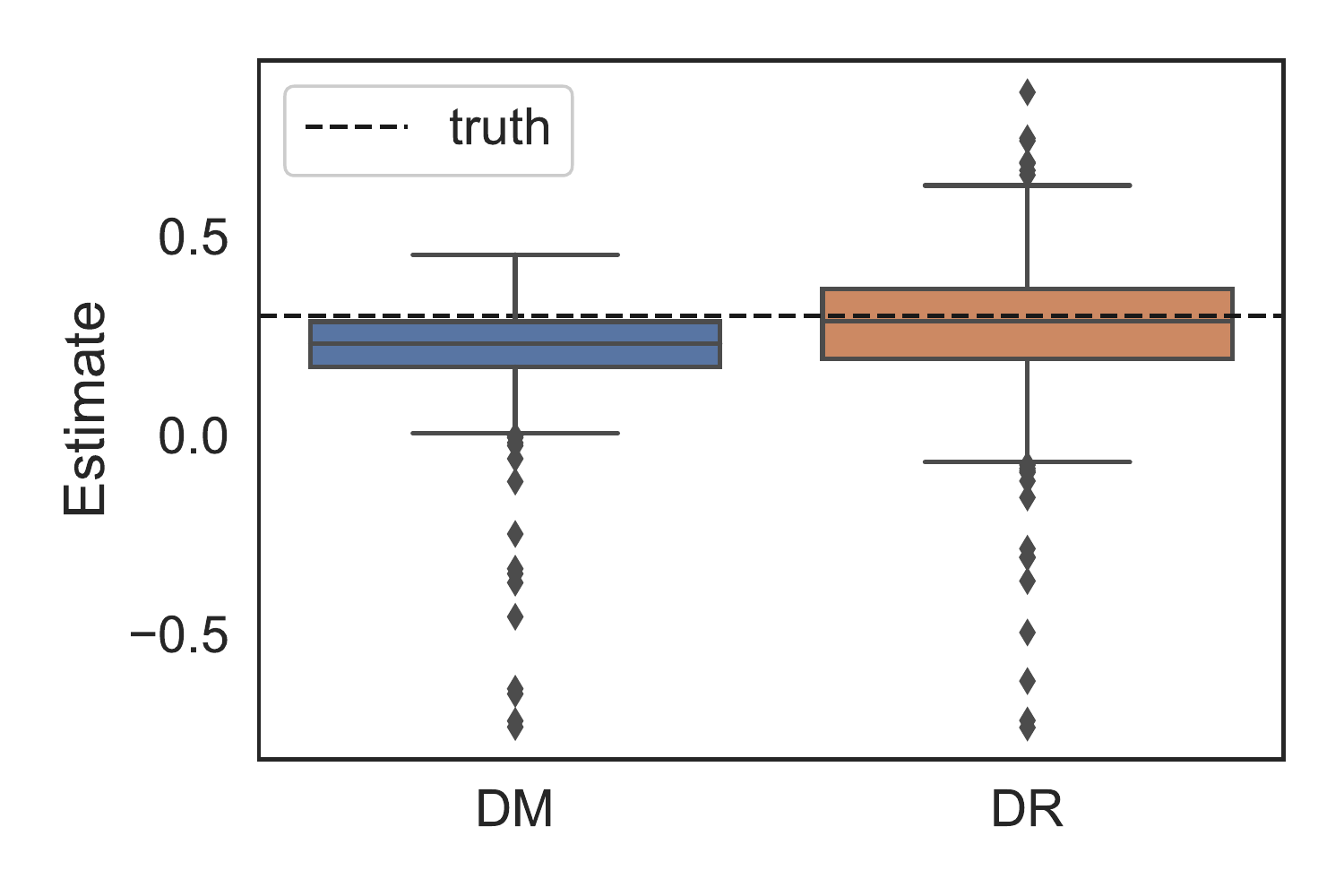}
    \caption{Distribution of DM and DR policy value estimates in Example~\ref{example} over $1000$ simulations. The interquartile range (IQR) box represents the middle $50\%$ of the data, and the whiskers are extended by $1.5$ of the IQR past the low and high quartiles. DM is highly skewed and has large downward bias due to adaptive data collection, while DR has high variance and heavy tails due to large importance weights from the poor overlap between the data-collection mechanism and the target policy. 
    }
    \label{fig:intro}
\end{figure}

One common approach to off-policy evaluation, often called the \emph{direct method} (DM), is to use the historical data to estimate a regression model that
predicts outcomes for each treatment/context pair, then averages this model's predictions for treatments assigned by the target policy. 
However, regression models fit to adaptively collected data tend to be biased \citep{villar2015multi,nie2017adaptively,shin2019sample,shin2020conditional},
and this bias is inherited by DM estimators. In particular, downward bias tends to arise when fitting data generated by a bandit, 
as treatments for which we observe random downward fluctuations early on tend to be sampled less afterward than those for which we observe upward ones.
Therefore, downward fluctuations get corrected less than upward ones when more data arrives. Consider the following example.

\begin{example}
\label{example}
We run a contextual bandit experiment with a Thompson sampling agent  \citep{thompson1933likelihood}, collecting $1000$ observations for one simulation. At each time $t$, a context $X_t$ is sampled from the context space $\{x_0, x_1, x_2, x_3, x_4\}$ with probability $[0.6, 0.1, 0.1, 0.1, 0.1]$. The agent chooses one $W_t=w_j$ out of five treatments  $\{w_0, w_1, w_2, w_3, w_4\}$ based on assignment probabilities that are computed using previous observations, and observes the outcome $\one\{i=j\} + \varepsilon_t$,  where $\varepsilon_t$ is standard normal.  The target policy always assigns treatment $W_t=w_0$.
\end{example}

\noindent A natural implementation of the direct method estimates the target policy value by $T^{-1}\sum_{t=1}^{T}  \hat\mu_0(X_t)$, where $\hat\mu_0(x)$ is the sample mean calculated from observed outcomes under treatment $w_0$ on context $x$. Figure~\ref{fig:intro} shows that it has substantial bias.

An unbiased alternative, \emph{inverse propensity weighting} (IPW) \citep{horvitz1952generalization}, weights observations from the collected data to look like treatment were assigned using the target policy. 
However, when the assignments made during data collection differ substantially from those made by the target policy, called \emph{low overlap} between policies,
this requires the weights, and therefore the variance, to be large \citep{imbens2004nonparametric}. 
The doubly robust (DR) estimator \citep{dudik2011doubly}, also unbiased, reduces variance by combining DM and IPW.
It uses a weighted average of regression residuals to correct the bias of the DM estimator,
reducing variance to a degree determined by the accuracy of the regression model's predictions. 
This estimator is optimal in large samples when historical observations are independent and identically distributed \citep{bickel1993efficient}. 
With observations from adaptive experiments, while what is optimal is not well understood, there is strong evidence that DR estimator is outperformed by other estimators in terms of mean squared error \citep{hadad2021confidence}.

This difference in behavior is caused by \emph{drifting overlap} in adaptive experiments. 
As the experimenter shifts the data-collection policy in response to what they observe,
overlap with the target policy can deteriorate. As a result, the variance of both the IPW and DR estimators 
can be unacceptably large. This is exacerbated in contextual settings, as importance weights can grow without bound even when there exists a relatively small set of contexts for which overlap is low.
In Example~\ref{example}, the set of contexts for which overlap worsens has total probability $0.4$, 
but, as shown in Figure~\ref{fig:intro}, this is enough to inflate the variance of the DR estimator.

Recently, shrinking importance weights toward one has been explored as a variance reduction strategy for the DR estimator \citep{charles2013counterfactual,wang2017optimal,su2019cab,su2020doubly}. This can substantially reduce variance by introducing a potentially small bias if the amount of shrinkage is chosen well,
but choosing this is a challenging tuning problem. An alternative approach focuses on local stabilization. 
For example,  when data is collected in batches, \cite{zhang2020inference} proposes the use of 
a hypothesis test for the policy value based on the average of batch-specific studentized statistics. \citeauthor{zhang2020inference} shows that the variance is stabilized
within each batch using this hypothesis test. \cite{luedtke2016statistical} propose a locally studentized version of the DR estimator that does not require batched collection,
using a rolling-window estimate of a term in the DR estimator's standard deviation. 
\cite{hadad2021confidence} refine this approach, observing that the optimal way to aggregate 
observations with different variances is not an inverse standard deviation weighted (standardized/studentized) average
 but an inverse variance weighted one. Focusing on the case of multi-armed bandits, 
they show that locally inverse variance weighted averages can reduce asymptotic variance while retaining its desirable inferential properties
if a particular approximate rolling-window variance estimate is used.
We view our method as an extension of \cite{hadad2021confidence} to contextual bandits,
where the story of local stabilization is more complex 
because the relevant notion of local variance depends on context as well as treatment.


\section{Setting}

We begin by formalizing the problem of off-policy evaluation in contextual bandits and introducing some notation.
We use potential outcome notation \citep{imbens2015causal}, denoting by $Y_t(w)$ a random variable representing the outcome that would be observed 
if at time $t$ an individual were assigned to a treatment $w$ from a finite set of options $\{1 \ldots K\}$.
In any given experiment, this individual can be assigned only one treatment $W_t$, 
so we observe only one realized outcome $Y_t = Y_t(W_t)$. Associated to each individual is
a context $X_t$; a policy $\pi$ assigns individuals with context $x$ to a treatment $w$ with probability $\pi(x,w)$.

We focus on the ``stationary'' environment where individuals, represented by 
a context $X_t  \in \mathcal{X}$ and a vector of potential outcomes $\big(Y_t(1),$ $\ldots, Y_t(K)\big)$, are independent and identically distributed.
However, our observations $(X_t, W_t, Y_t)$ are neither independent nor identically distributed because the assignment $W_t$ depends on the observations before it.
We assume that an individual's treatment assignment is randomized, with nonzero probability of receiving each treatment,\footnote{We do, however, permit the assignment probability $e_t(x,w)$ to decay to zero as $t$ grows. This is typically what happens in bandit experiments when 
treatment $w$ is, on average, suboptimal for individuals with context $x$.}
as a function of their context and previous observations. We write $H_t = \{(X_s,W_s,Y_s): s \le t\}$ for the history
of observations up to time $t$; $e_t(x,w)=\Pp(W_t=w \mid X_t=x, H_{t-1})$ for the conditional randomization probability, also known as the propensity score;
and $\mu_t(x,w)=\bE[Y_t(w) \mid X_t=x]$ for the regression of the outcome on context. We summarize these assumptions 
as follows.


\begin{assumption} For all periods $t$, \\ 
\label{assu:data}
\vspace{-.4cm}
 \begin{enumerate}[(a)]
     \item $(X_t, Y_t(1) \ldots Y_t(K))$  are independent and identically distributed;
     \item $(Y_t(1) \ldots Y_t(K)) \indep W_t \mid X_t,\ H_{t-1}$
     \item $e_t(x,w)>0$ for all $(x,w)$.

 \end{enumerate}
 \end{assumption}

Our goal is to estimate the quantity 
\begin{equation}
    Q(\pi) 
    := \sum_{w=1}^K  \bE[\pi(X_t, w) Y_t(w)],
\end{equation}
which is the average outcome attained when individuals are assigned treatment in accordance with the policy $\pi$, typically called the \emph{average policy value}.

Because Assumption~\ref{assu:data}(a) implies $\mu_t(x,w)$ is constant in $t$, we will write $\mu(x,w)$ in its place from here on.
To simplify notation, we will also write \smash{$\sum_{w} f(\cdot, w)$} for \smash{$\sum_{w=1}^K f(\cdot, w)$}, write $||f(z)||_{\infty}$
for the sup-norm  $\sup_{z \in \mathcal{Z}} |f(z)|$ of a function $f: \mathcal{Z} \to \mathbb{R}$,
and let $Q(x, \pi) := \sum_{w} \E{}{\pi(X_t, w) Y_t(w)|X_t=x}$.

\section{Doubly robust estimator}
\label{sec:dr}

The doubly robust  (DR) estimator considered here takes the form
\begin{equation}
    \label{eq:qdr}
    \widehat{Q}_T^{DR}(\pi) := \frac{1}{T} \sum_{t=1}^{T} \widehat{\Gamma}_t(X_t, \pi),
\end{equation}
\noindent where the objects being averaged are \emph{doubly-robust scores},
\begin{equation*}
    \label{eq:aipw}
    \widehat{\Gamma}_t(X_t, \pi) 
    := \sum_{w} \pi(X_t,w) \Big( \widehat{\mu}_t(X_t, w) + \frac{\one\{W_t=w\}}{e_t(X_t,w)}(Y_t-\widehat{\mu}_t(X_t, w))\Big).
\end{equation*}
Above, $\widehat{\mu}_t(x,w)$ is an estimate of the outcome regression $\mu(x, w)$ based only on the history $H_{t-1}$. 
Similar estimators are widely used, including in bandit settings \citep{dudik2011doubly, howard2018uniform, van2008construction,hadad2021confidence}.

Recalling that the assignment probabilities $e_t(X_t, w)$ 
and the estimated regression $\widehat{\mu}_{t}(X_t, w)$ are functions of the history, 
it is straightforward to show that the DR scores \smash{$\widehat\Gamma_t(\pi)$} are unbiased estimators of the policy value $Q(\pi)$ conditional on past history 
and, by the tower property, their average $\widehat Q_T^{DR}(\pi)$  is unbiased.
\begin{proposition}[Unbiasedness] For any $\pi\in\Pi$, we have
\label{prop:aipw_unbiased}
\[ \bE\big[\widehat{\Gamma}_t(X_t, \pi)|H_{t-1}, X_t\big] = \sum_w \pi(X_t, w)\mu(X_t, w) \ \mbox{ and } 
\ \bE\big[\widehat{Q}_T^{DR}(\pi)\big]
    =
    \bE\Big[\frac{1}{T} \sum_{t=1}^{T}
        \bE\big[\widehat{\Gamma}_t(X_t, \pi)|H_{t-1}\big]
        \Big]
    = Q(\pi).
\]
\end{proposition}

\noindent We refer readers to Appendix~\ref{sec:proof_aipw_unbiased} for proof details. The variance of the DR estimator is largely governed by the similarity between the policy being evaluated and the assignment probabilities. This phenomenon is expected since we should be better able to evaluate policies that assign treatments in a similar manner as that used when collecting data. The next proposition, which quantifies this intuition, is convenient when considering variance-minimizing schemes in the next section. 

\begin{proposition}[Dominant Term of Variance]
\label{prop:aipw_var} Under Assumption~\ref{assu:data}, suppose there are positive, finite constants $C_0$, $C_1$ such that for all $w$ and $t$,  $\var(Y_t(w)\mid X_t)\in[C_0, C_1]$. Suppose there is a finite constant $C_2$ such that 
$\norm{\mu}_{\infty} < C_2$ and $\norm{\widehat{\mu}_t}_{\infty} < C_2$ for all $t$. Then, there exist positive, finite constants $L$, $U$ depending only on $C_0, C_1, C_2$ such that for any $\pi\in\Pi$ and all $t$,
\begin{equation}
    \label{eq:gamma_condvar_x}
    L\leq \Var{}{\widehat{\Gamma}_t(X_t, \pi)|H_{t-1}, X_t} \Big/ \bE\Big[\sum_{w}\frac{\pi^2(X_t,w)}{e_t(X_t,w)}\Big|H_{t-1}, X_t\Big] \leq U.
\end{equation}
\end{proposition}
\noindent The proof is deferred to Appendix~\ref{sec:proof_aipw_var}. Averaging over the contexts in  \eqref{eq:gamma_condvar_x}, we derive a variant for the context-free conditional variance.
\begin{equation}
    \label{eq:gamma_condvar}
    L\leq \Var{}{\widehat{\Gamma}_t(X_t, \pi)|H_{t-1}} \Big/ \bE\Big[\sum_{w}\frac{\pi^2(X_t,w)}{e_t(X_t,w)}\Big|H_{t-1}\Big]  \leq U.
\end{equation}

\section{Adaptive weighting}
\label{sec:noncontextual_weighting}

When overlap between the evaluation policy and the assignment mechanism deteriorates as $t$ grows
(in the sense that $\sum_w \pi^2(X_t, w) / e_t(X_t, w)$ increases), 
the DR estimator \eqref{eq:qdr} can have high variance and heavy tails with large samples. This phenomenon is demonstrated in Figure~\ref{fig:intro}. 
In light of Proposition~\ref{prop:aipw_var}, the problem is that the DR estimator is a uniformly-weighted average of scores \smash{$\widehat \Gamma_t$} 
with variances that can differ dramatically. We can improve on the DR estimator by averaging those scores with non-uniform weights $h_t$ that stabilize variance \smash{$\var(\widehat{\Gamma}_t(X_t, \pi) \mid H_{t-1})$} term-by-term.
\begin{equation}
\label{eq:qnc}
    \widehat{Q}^{NC}_T(\pi) := 
    \sum_{t=1}^T 
        \frac{h_t}{\sum_{s=1}^T h_s}
        \widehat{\Gamma}_t(X_t, \pi),
\end{equation}

As stabilizing weights, we consider both weights that approximately standardize the terms \smash{$\widehat \Gamma_t$} 
and approximate inverse variance weights. Both rely on the variance proxy from \eqref{eq:gamma_condvar}:
\begin{equation}
\label{eq:hnc_infeasible}
    h_t := \phi\Big(
      \bE\Big[
        \sum_{w}
            \frac{\pi^2(X_t;w)}{e_t(X_t;w)}\Big|H_{t-1}\Big]
            \Big),
\end{equation}
where different functions $\phi$ yields weights that have different properties:
\begin{enumerate}[(i)]
    \item \texttt{StableVar}: $\phi(v)=\sqrt{1/v}$ yields  weights $h_t$ that approximately standardize the terms of $\widehat{Q}^{NC}_T$.
    \item \texttt{MinVar}: $\phi(v)=1/v$ yields weights $h_t$ that approximately minimize the variance of $\widehat{Q}^{NC}_T$.
\end{enumerate}

We refer to both weights in \eqref{eq:hnc_infeasible} as \emph{non-contextual} weights, since the realized context $X_t$ is integrated out. In practice, one may not compute oracle weights \eqref{eq:hnc_infeasible} exactly. We thus use a feasible approximation in our simulation,
\begin{equation}
\label{eq:hnc_feasible}
    \widetilde h_t := \phi\Big(
      \frac{1}{t-1}
      \sum_{s=1}^{t-1}
        \sum_{w}
            \frac{\pi^2(X_s;w)}{e_t(X_s;w)}
            \Big).
\end{equation}

The idea of \minvar weighting scheme is similar to the weighted IPW estimator proposed in \cite{agarwal2017effective}, where the authors use the (context-independent) inverse variance to weight samples from multiple logging policies. 
The essential difference is that, in the adaptive setting we consider, each observation is from a new logging policy, so we must estimate inverse variance weights for each observation instead of for a relatively small number of logging policies.

The \texttt{StableVar} weights are less effective at reducing variance than \texttt{MinVar},
 but would allow us to construct asymptotically valid confidence intervals if we had access to an oracle that would allow us to compute the conditional expectation in \eqref{eq:hnc_infeasible}.  
 To start, we introduce additional assumptions on the data generating process and the nuisance estimator $\hat{\mu}_t$.

\begin{assumption}
\label{assu:clt_condition}
Suppose that the following conditions are satisfied:
\begin{enumerate}[(a)]
    \item $\var(Y_t(w)|X_t=x)\geq C_0$ and $\bE[Y_t^4(w)|X_t=x]\leq C_1$ for all $w, x$ and some positive constants $C_0, C_1$.
    \item $e_t(x,w) \geq Ct^{-\alpha}$ 
for all $w, x$ and some constants $C$ and $ \alpha\in[0,\frac{1}{2})$.
    \item $\sup_{w,x}|\mu(x,w)|$ and $\sup_{t,w,x}|\hat\mu_t(x,w)|$  are uniformly bounded and $\hat \mu_t$ is convergent in the sense that,
    for some function $\mu_\infty$,
    $\sup_{w,x}|\hat\mu_t(x, w)-\mu_\infty(x, w)| \to 0$ almost surely.
\end{enumerate} 
\end{assumption}
Assumption~\ref{assu:clt_condition}(a) is satisfied for various types of non-constant random variables including, for example, bounded rewards. Assumption~\ref{assu:clt_condition}(b) requires that exploration in the adaptive experiments should not be reduced faster than some polynomial decay rate, and is also enforced when analyzing the asymptotics of adaptive weighting for multi-armed bandits in \cite{hadad2021confidence}. Assumption~\ref{assu:clt_condition}(c) is standard in the literature on DR estimators \citep{imbens2004nonparametric,chernozhukov2016double}. We state a limit theorem,
proven in Appendix~\ref{sec:proof_noncontextual_clt}, for for $\hat{Q}_T^{NC}$ using \StableVar weighting. 

\begin{theorem}
\label{thm:non_contextual} 
Under Assumptions~\ref{assu:data} and~\ref{assu:clt_condition}, estimator $\widehat{Q}_T^{NC}(\pi)$ with \texttt{StableVar} weights is consistent for true policy value $Q(\pi)$. Furthermore,  suppose that  assignment probabilities $e_t$ satisfy
\begin{equation}
\label{eq:e_convergence}
    \sup_{x,w}\bigg|\frac{e_t^{-1}(x,w)}{\bE[e_t^{-1}(x,w)]} - 1\bigg|\rightarrow0, \quad a.s.
\end{equation}
Then non-contextual \texttt{StableVar} weighting yields an asymptotically normal studentized statistic 
\begin{equation}
\label{eq:clt_noncontextual}
    \frac{\widehat{Q}_T^{NC}(\pi)-Q(\pi)}{\big(\widehat{V}^{NC}_T(\pi)\big)^{1/2}}
    \xrightarrow{d}\N(0,1), 
    \  \text{ where } \ 
    \widehat{V}_T^{NC}(\pi) := \frac{\sum_{t=1}^T h_t^2\big(\widehat{\Gamma}_t(X_t, \pi) -\widehat{Q}_T^{NC}(\pi) \big)^2 }{\big(\sum_{t=1}^T h_t \big)^2}.
\end{equation}
\end{theorem}

Condition \eqref{eq:e_convergence} says that the inverse of assignment probabilities, which essentially represents the variance of estimating the value of each arm, is asymptotically equivalent to its expectation. To understand it, consider the thought experiment of running multiple simulations of a bandit setting in which arms are identical. If the bandit algorithm is ``too greedy'', spurious fluctuations in arm value estimates at the beginning can decide the algorithm's long-run behavior. In some of these simulations, $e_t(x, w)$ may increase to one while in others it may decay to zero. Condition \eqref{eq:e_convergence}
requires that the long-run behavior of the assignment algorithm is  stable. For many bandit algorithms, 
it will hold if, for every point of the context space, the signal from the best arm is strong enough to have it discovered eventually.
For some tractable contextual bandit algorithms, for example \cite{krishnamurthy2020tractable},
it holds without qualification.




\section{Adaptive contextual weighting}
\label{sec:contextual_weighting}

We can control variance better by using context-dependent weights $h_t(x)$. Thinking of the DR estimator $\widehat Q_T^{DR}(\pi)$ as a double sum \smash{$T^{-1}\sum_{t=1}^T \sum_{x \in \X} \mathbb{I}\{X_t=x\} \widehat \Gamma_t(x,\pi)$}, it is natural to stabilize terms in the inner sum rather than the outer, as proposed in Section~\ref{sec:noncontextual_weighting}. 

Figure~\ref{fig:arm0} demonstrates how, in Example~\ref{example},
the contextual variance proxy from \eqref{eq:gamma_condvar_x} varies with context. Recall that the target policy always assigns treatment $w_0$: $\pi(\cdot, w_0)=1$. For context $x_0$, treatment $w_0$ is optimal, so $e_t(x_0, w_0)$ increases over time as the experimenter learns to assign it. Thus,  the variance proxy $\sum_{w}\pi^2(x_0,w)/e_t(x_0,w)$ decreases. For other contexts, $w_0$ is suboptimal 
and the variance proxy increases, which
drives the increase in the non-contextual variance proxy (dashed black line in Figure~\ref{fig:arm0}) used in \eqref{eq:hnc_infeasible}.

\begin{figure}[H]
    \centering
    \includegraphics[width=.8\linewidth]{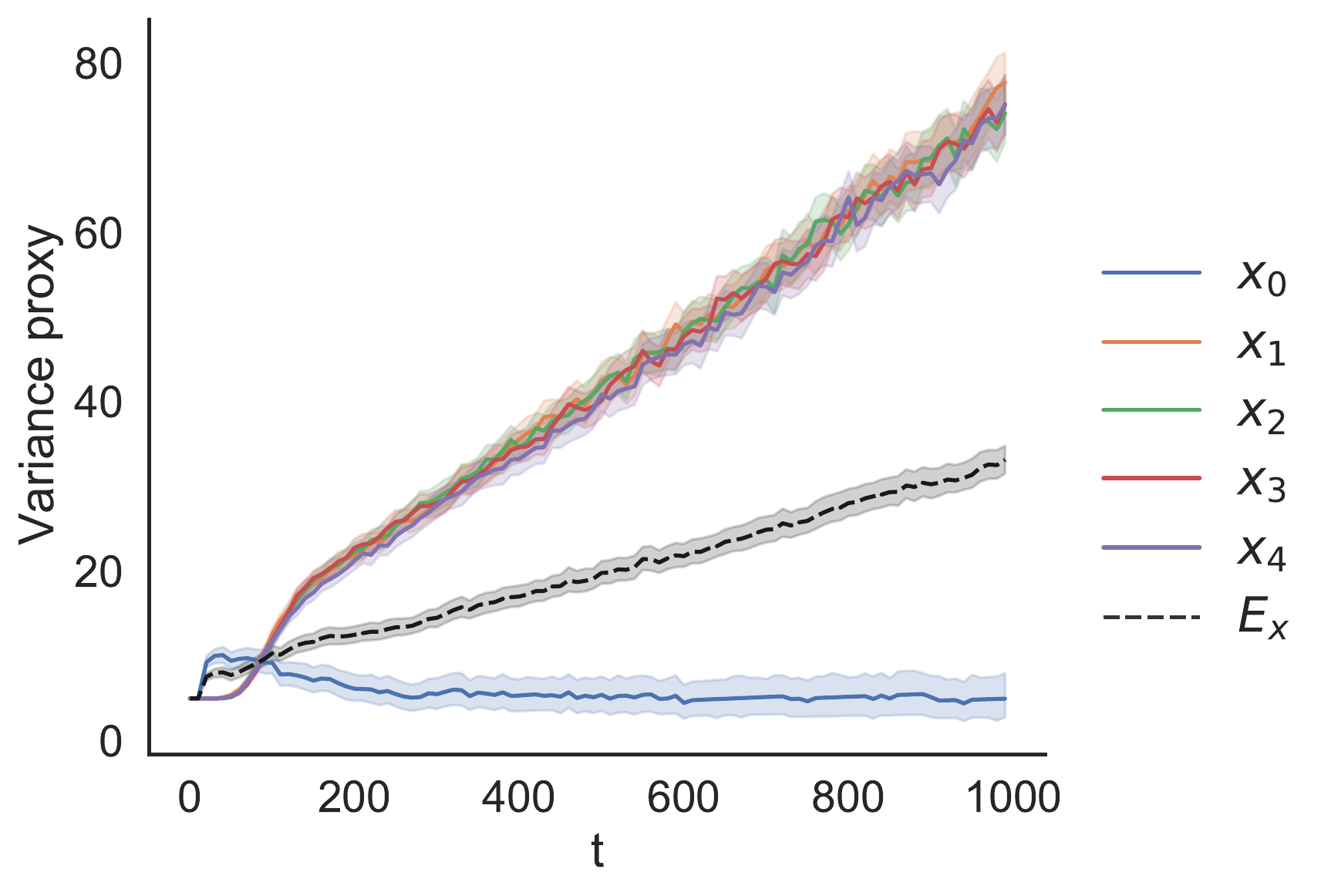}
    \caption{In Example~\ref{example}, the contextual variance proxy $\sum_{w}\pi^2(X_t,w)/e_t(X_t,w)$ from Proposition~\ref{prop:aipw_var} varies significantly by context. We plot
    the median and bands indicating $2.5$th and $97.5$th percentiles over $1000$ simulations.}
    \label{fig:arm0}
\end{figure}

We thus propose \emph{contextual} weighting  to observe such variation of score variance in context:
\begin{equation}
\label{eq:qc}
    \widehat{Q}_T^C(\pi) = \sum_{t=1}^T \frac{h_t(X_t)}{\sum_{s=1}^T h_s(X_t)}\widehat{\Gamma}_t(X_t, \pi).
\end{equation}
Here, adaptive weight $h_t(\cdot)$  is a function of context $x\in\X$. The normalized term \smash{$h_t(X_t)/\sum_{s=1}^T h_s(X_t)$} weights the score \smash{$\widehat{\Gamma}_t$.} We define the contextual adaptive weight $h_t(x)$ as a function of the context-specific variance proxy from \eqref{eq:gamma_condvar_x}, 
\begin{equation}
    \label{eq:hc}
    h_t(x) = \phi\Big(\sum_{w}\frac{\pi^2(x,w)}{e_t(x,w)}\Big), \quad x\in \X,
\end{equation}
where, as in Section~\ref{sec:noncontextual_weighting}, 
$\phi(v)=\sqrt{1/v}$
for \texttt{StableVar} weights and $\phi(v)=1/v$ for \texttt{MinVar} weights. 




Like its counterpart with non-contextual weights, the estimator  \eqref{eq:qc} admits a central limit theorem when contextual \texttt{StableVar} weights are used. However, our result requires more fragile assumptions on the data generating process and assignment mechanism. Again, we focus on evaluating a fixed-arm policy that always assigns a pre-specified treatment $w$.

\begin{theorem}
\label{thm:contextual} 
Suppose Assumptions~\ref{assu:data} and~\ref{assu:clt_condition} hold and the context space $\X$ is  discrete, estimator $\widehat{Q}^C_T(\pi)$ with \StableVar weights is consistent for the true policy value $Q(\pi)$.
Further, suppose that  assignment probabilities $e_t$ satisfy that,
\begin{equation}
    \label{eq:e_pair_convergence}
    \begin{aligned}
   \sup_{x,w} \bigg|\frac{\bE[e_t^{-1}(x,w)]}{e_t^{-1}(x,w)}-1 \bigg|\rightarrow 0, \mbox{ and }\sup_{x,x',w,w'} \bigg|\frac{\bE[e_t^{-1}(x,w) e_t^{-1}(x',w')]}{e_t^{-1}(x,w) e_t^{-1}(x',w')} -1 \bigg|\rightarrow 0, \quad \mbox{a.s. and in }L_1.
    \end{aligned}
\end{equation}
Then contextual \texttt{StableVar} weighting \eqref{eq:hc} yields an asymptotically normal studentized statistic:
\begin{equation}
\label{eq:clt_contextual}
\begin{aligned}
       &\frac{\widehat{Q}_T^{C}(\pi)-Q(\pi)}{\big(\widehat{V}^{C}_T(\pi)\big)^{1/2}}\xrightarrow{d}\N(0,1),\mbox{ where } \\
       \widehat{V}_T^{C}(\pi) =  \sum_{t=1}^T \Big(&\frac{h_t(X_t)}{\sum_{s=1}^T h_s(X_t)}\widehat{\Gamma}_t(X_t, \pi) - \sum_{s=1}^T\frac{h_t(X_s)h_s(X_s)}{\big(\sum_{s'=1}^T h_{s'}(X_s)\big)^2} \widehat{\Gamma}_s(X_s, \pi) \Big)^2 . 
\end{aligned}
\end{equation}
\end{theorem}

\noindent We defer the proof to Appendix~\ref{sec:proof_contextual_clt}. Condition \eqref{eq:e_pair_convergence} requires the similar stability from assignment probabilities as  Condition \eqref{eq:e_convergence}, while this time it in addition requires a more strict assumption---a pair of inverse of $e_t$ needs to be asymptotically equivalent to its expectation.

\section{Experimental results}
\label{sec:simulation}

This section provides an empirical investigation of different estimators for off-policy evaluation in contextual bandits.\footnote{Reproduction code can be found at \url{https://github.com/gsbDBI/contextual_bandits_evaluation}.} We consider four estimators: (i) the  DM estimator 
\begin{equation}
    \widehat{Q}^{DM}_T(\pi)=T^{-1}\sum_{t=1}^T\sum_{w}\widehat{\mu}_{T}(X_t,w)\pi(X_t, w),
\end{equation}
 (ii) the DR estimator $\smash{T^{-1} \sum_{t=1}^T \widehat{\Gamma}_t}(X_t, \pi)$, (iii) the adaptively-weighted estimator with non-contextual \texttt{MinVar} and \texttt{StableVar} weights as in Section~\ref{sec:noncontextual_weighting}, and  (iv) the adaptively-weighted estimator with contextual \texttt{MinVar} and \texttt{StableVar} weights  as in Section~\ref{sec:contextual_weighting}.
 In all estimators, $\hat \mu_t(\cdot,w)$
 is a linear model fit via least squares on the observations $\{(X_s,Y_s) : W_s=w,\ s < t\}$.

We consider two setups. In Section~\ref{sec:synth}, we have a synthetic data-generating process (DGP). In Section~\ref{sec:class}, the DGP is adapted from multi-class classification datasets on OpenML \citep{OpenML2013}. In both cases, we show that adaptive weighting reduces variance and MSE significantly and that  contextual weighting outperforms non-contextual weighting.

\paragraph{Target policies.} We estimate the contrast
$\Delta = Q(\pi_1)-Q(\pi_2)$ between the best contextual policy $\pi_1$ 
and the best non-contextual policy $\pi_2$. The former
 assigns individuals with context $x$ to the treatment $w^*(x)=\arg\max_w \bE[Y_t(w) \mid X=x]$; the latter assigns all individuals to $w^*=\arg\max_w \bE[Y_t(w)]$. To construct this estimator $\widehat{\Delta}_T$ and corresponding estimated variance $\widehat{V}_T$, we re-visit the discussion starting from Section~\ref{sec:dr}, defining the doubly robust score for the contrast as the difference in doubly robust scores for the two policies. We use variants of the adaptive weights discussed in Sections~\ref{sec:noncontextual_weighting} \&~\ref{sec:contextual_weighting} with the variance proxy $\sum_w (\pi_1(x,w)-\pi_2(x,w))^2/e_t(x,w)$. 

\paragraph{Metrics.} In Section~\ref{sec:synth}, we evaluate the performance of each estimator in terms of its root mean squared error (RMSE), bias, and the radius and coverage of $95\%$ confidence intervals based on the approximate normality of
studentized estimates, as in Theorems~\ref{thm:non_contextual} and~\ref{thm:contextual}. In Section~\ref{sec:class}, we focus on RMSE, bias, and standard errors.


\subsection{Synthetic datasets}
\label{sec:synth}


\begin{figure}
    \centering
+    \includegraphics[width=\linewidth]{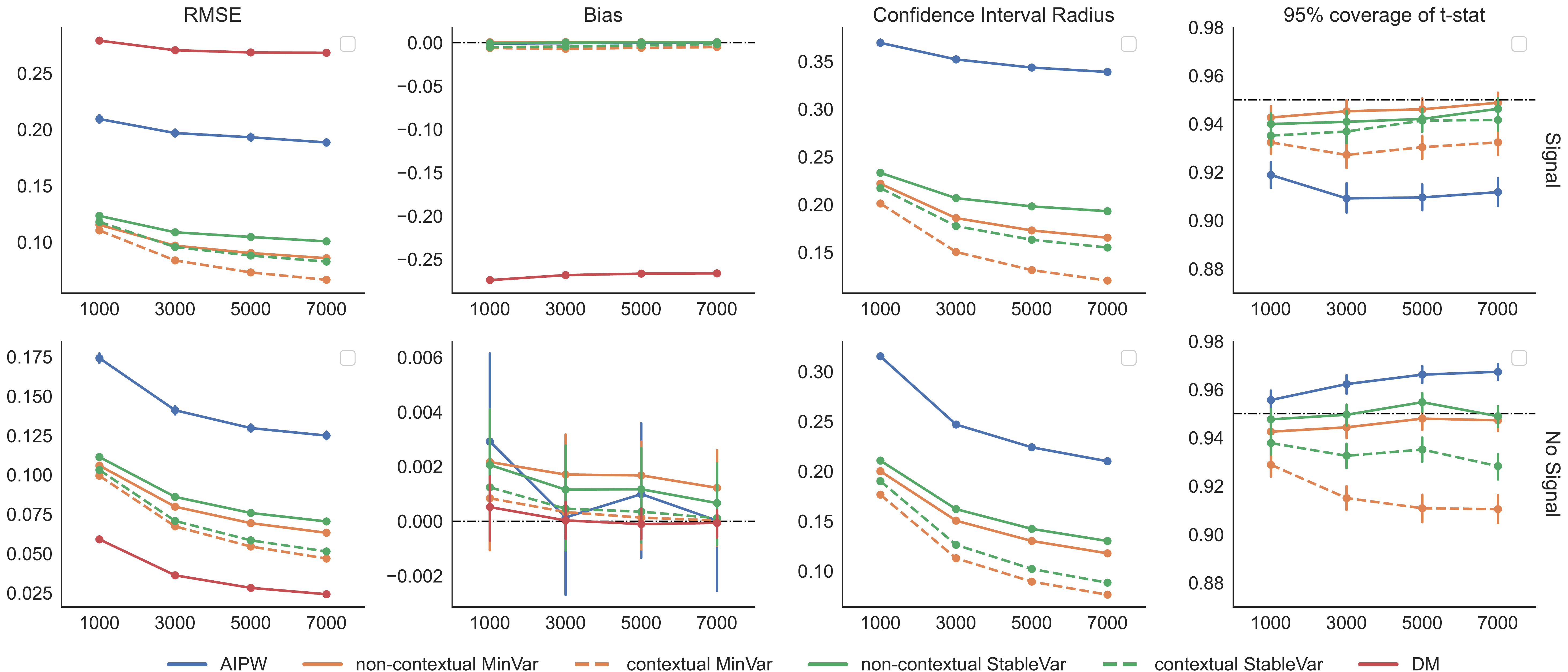}
    \caption{Evolution of estimates of contrast value between best contextual policy and best non-contextual policy with synthetic data in Section \ref{sec:synth}. X-axis is the sample size. Error bars indicate $95\%$ confidence intervals based on averaging over $10^5$ replications of the simulation.} 
    \label{fig:region}
\end{figure}
\begin{figure}
    \centering
    \includegraphics[width=\linewidth]{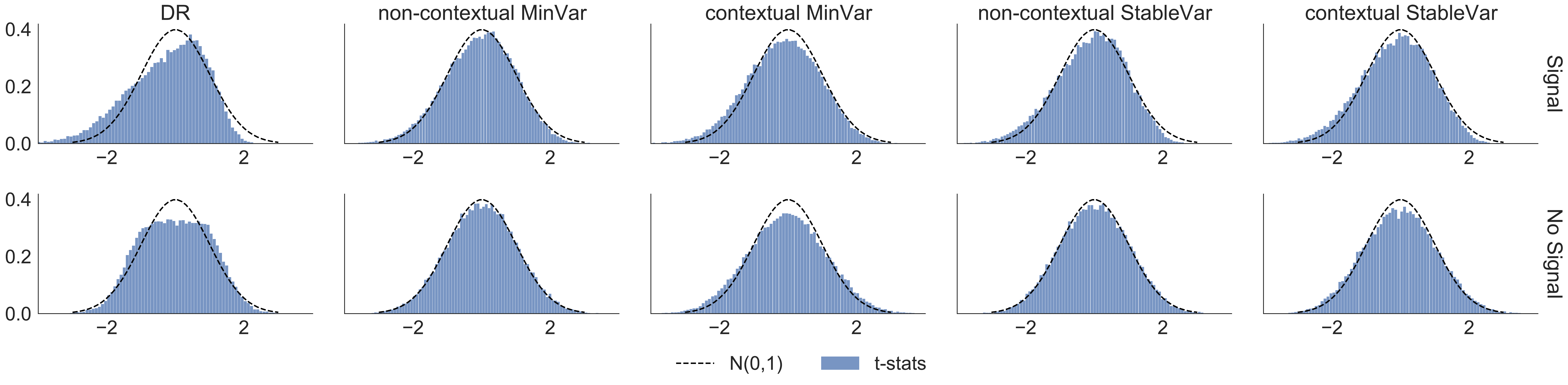}
    \caption{Histogram of studentized statistics of the form $(\widehat{\Delta}_T - \Delta) / \widehat{V}_T^{1/2}$ for  estimating the contrast value between best contextual policy and best non-contextual policy at $T=7000$ with synthetic data in Section \ref{sec:synth}. Numbers aggregated over $10^5$ replications.} 
    \label{fig:region_hist}
\end{figure}

\paragraph{Data-generating process} We consider a four-arm contextual bandit setting. Each time step, the algorithm observes a vector of context $X_t \sim \N(0, I_3)$ drawn independently. We consider two outcome models. In the \emph{no-signal} setting, $\mu(x,w)=0$ for all contexts and treatments. In the \emph{signal} setting, we divide the context space into four regions, depending on whether the first two coordinates $X_{t, 1}$ and $X_{t, 2}$ are above or below $0.5$. In each region, one of the treatments has  $\mu(x,w)=0.7$, where all others have $\mu(x,w)=0$, so each region is associated with a distinct treatment with positive mean outcome. In both settings, we observe $Y_t = \mu(X_t,W_t) + \varepsilon_t$ where $\varepsilon_t \sim \N(0,1)$. 

\paragraph{Assignment probabilities}  A modified Thompson sampling algorithm interacts with the environment, collecting data in batches  of size $100$. During the first batch, the agent assigns arms uniformly at random. For each subsequent batch, the agent first discretizes context space using PolicyTree \citep{sverdrup2020policytree} based on previous observations. Next, for each subspace, the algorithm computes tentative Thompson sampling assignment probabilities. Finally, a lower bound is imposed on these probabilities, so that final assignment probabilities are guaranteed to satisfy $e_t\geq t^{-0.8}/4$.

\begin{itemize}
    \item \textbf{RMSE.} Figure \ref{fig:region} shows that in both scenarios, adaptive weighting reduces RMSE significantly as compared to DR, and  contextual weighting has a larger improvement than non-contextual weighting. When there is signal, DM is outperformed by all other methods as a result of substantial bias. When there is no signal, it performs best, as the systematic biases in evaluating each of the two policies---which generate identical rewards in the case of no signal---cancel out.
    
    \item \textbf{Inference.} Figures \ref{fig:region} shows that
    in the signal scenario, adaptive weighting has nominal coverage with improved power relative to DR; in the no-signal scenario, only non-contextual weighting has nominal coverage. Such inferential properties are also verified by the t-statistics shown in Figure \ref{fig:region_hist}. This is in line with our theorems---Theorem ~\ref{thm:contextual} for contextual weighting requires higher-order stability of assignment probabilities as compared to Theorem~\ref{thm:non_contextual}. DM tends to be much more biased than variable, so confidence intervals based on the assumption of asymptotic normality with negligible bias and variance of approximately $\widehat{V}^{DM}_T(\pi) = T^{-2}\sum_{t=1}^T\big(\widehat{Q}^{DM}_T(\pi)-\sum_{w}\widehat{\mu}_{T}(X_t,w)\pi(X_t,w)\big)^2$ do not yield much coverage in our simulations; thus, we omit it in our coverage plots.
\end{itemize}

\subsection{Sequential classification with bandit feedback}
\label{sec:class}

\begin{table}[t]
    \centering
    \begin{tabular}{|cc|cc|cc|}
     \hline
      Classes   & Count &Features   & Count &Observations  & Count \\
     \hline
     $2$   & 66  &  $\leq 10$   & 46 & $\leq 1$k   & 44\\
     $>2$, $\leq 8$ & 16 &$>10$, $\leq 50$ & 36 & $>1$k, $\leq 10$k & 32\\
     $>8$ & 4 & $>50$ & 4 &   $>10$k & 10\\
  \hline
    \end{tabular}
    \hfill
\caption{Characteristics of $86$ public datasets used for sequential classification in Section~\ref{sec:public_datasets}.}
    \label{tab:dataset}
\end{table}

\begin{figure}[tb]
    \centering
    \includegraphics[width=\linewidth]{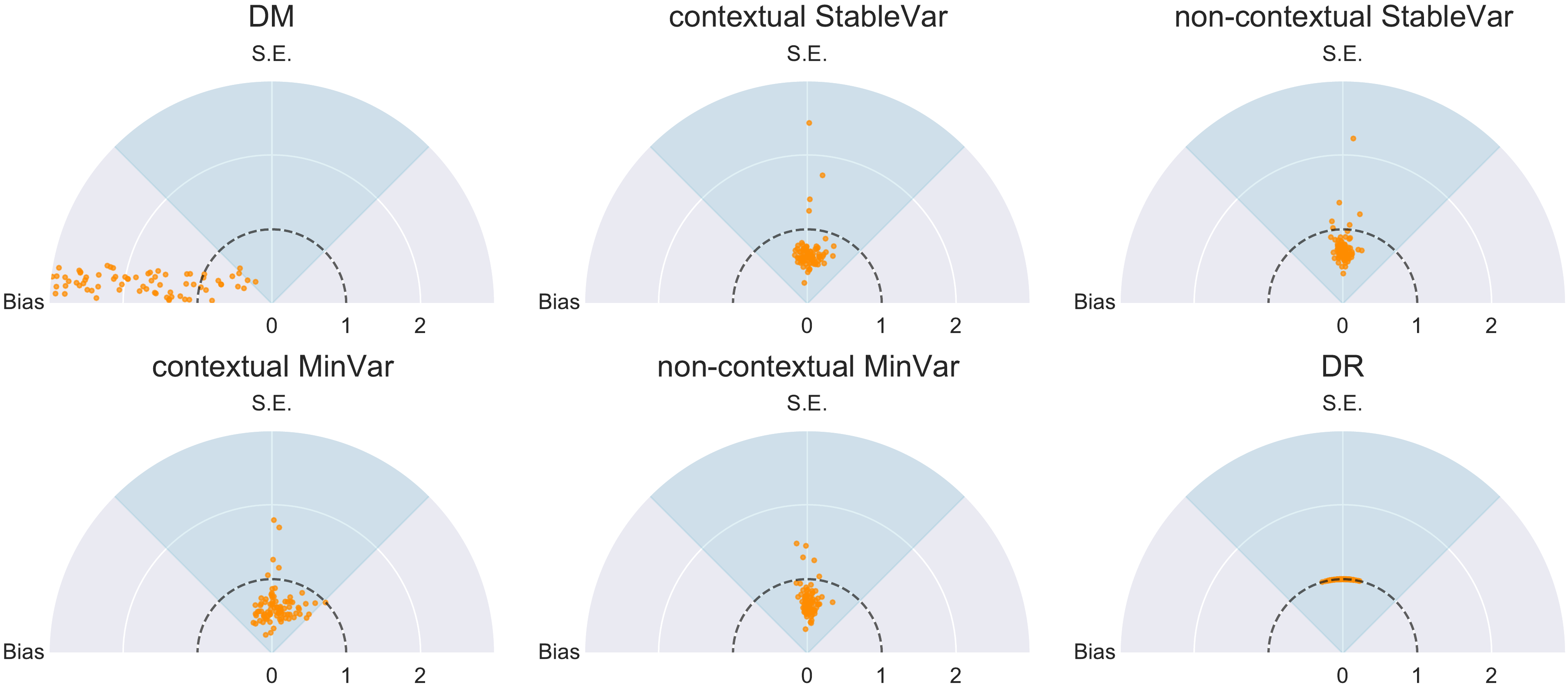}
    \caption{Comparison among different estimators when evaluating the contrast between optimal contextual policy and best fixed-arm policy across $86$ datasets. Each point corresponds to one dataset, and its cartesian coordinates are bias and standard error normalized by RMSE of DR estimate for the same dataset. Note that the distance of each point from zero is its relative RMSE.}
    \label{fig:dataset}
\end{figure}

\begin{figure}
    \centering
    \begin{subfigure}[t]{\textwidth}
    \centering
    \includegraphics[width=\linewidth]{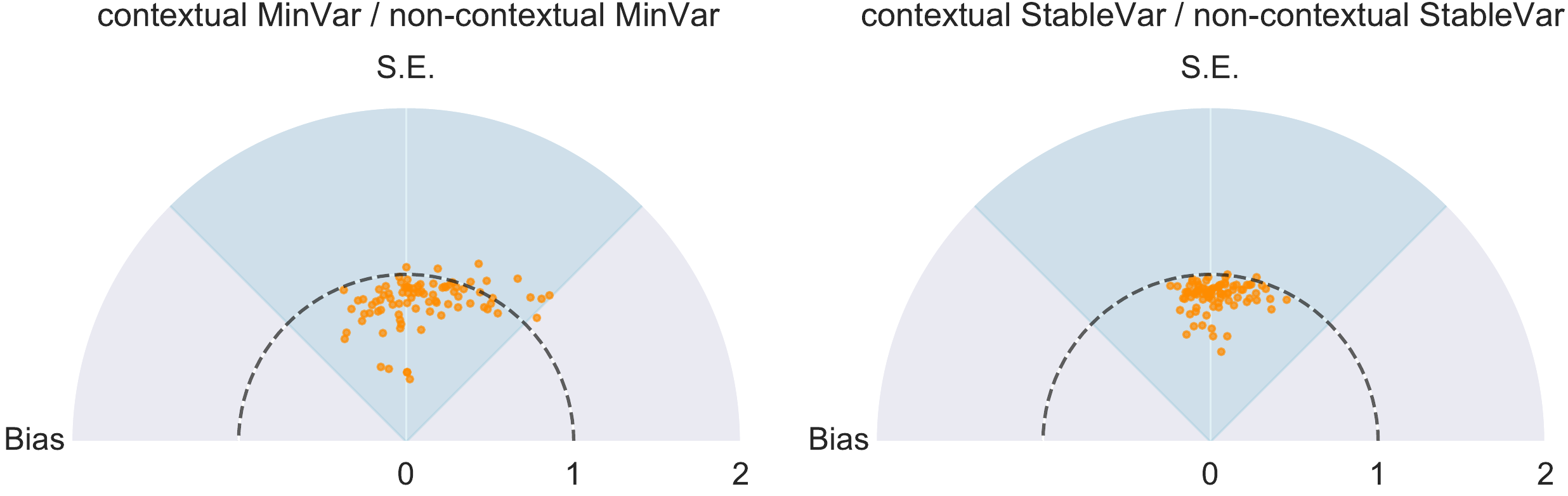}
    \caption{Contextual weighting \emph{v.s.} non-contextual weighting. Each point corresponds to one dataset, and its cartesian coordinates are bias and S.E. of the contextual weighting normalized by the RMSE of its non-contextual counterpart. }
    \label{fig:dataset_compare}
    \end{subfigure}
    
    \vspace{0.5cm}
    
    \begin{subfigure}[t]{\textwidth}
    \includegraphics[width=\linewidth]{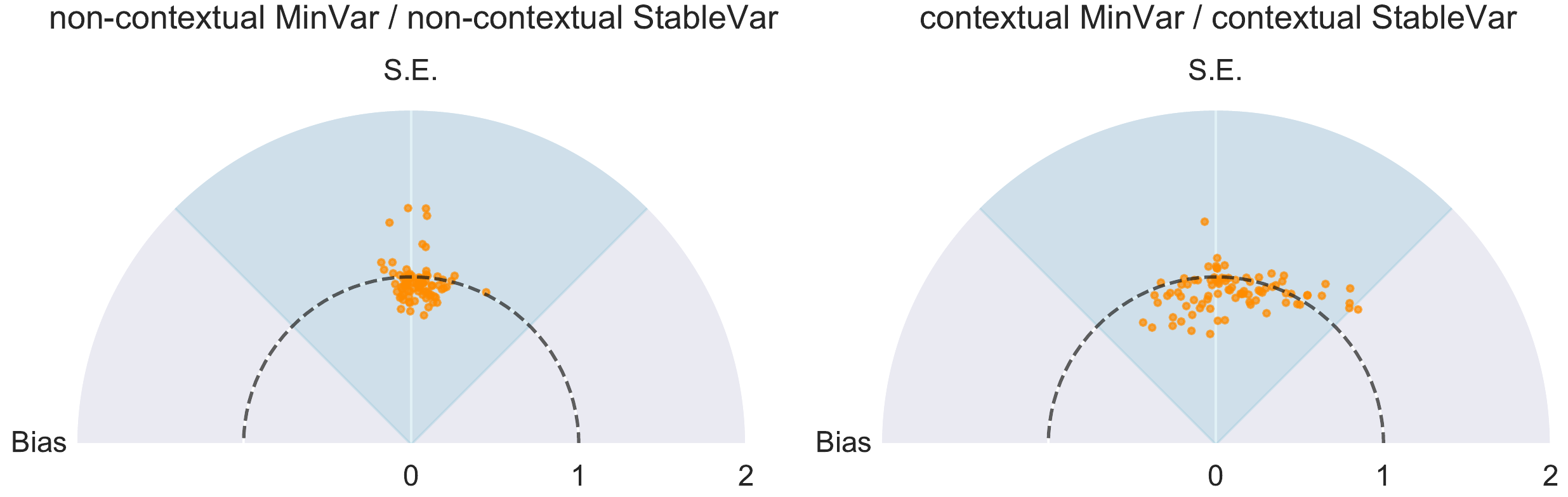}
    \caption{\texttt{MinVar} \emph{v.s.} \texttt{StableVar}. Each point corresponds to one dataset, and its cartesian coordinates are bias and standard error of the \texttt{MinVar} weighting normalized by the RMSE of the \texttt{StableVar} counterpart. }
    \label{fig:dataset_compare_2}
    \end{subfigure}
    \caption{Pairwise comparison when evaluating the contrast between best contextual policy and best non-contextual policy across $86$ datasets.}
\end{figure}

\paragraph{Data generating process} In this section we follow prior work \citep{dudik2011doubly,dimakopoulou2017estimation,su2020doubly,wang2017optimal} and use multi-class classification datasets to generate bandit problems. The idea is to transform a sequential classification problem into a bandit problem, by taking features as contexts and classes as arms. At each time step, the bandit algorithm is shown a vector of features sampled uniformly from the dataset and must choose a corresponding class. If it manages to label the observation correctly, the bandit algorithm receives reward $Y_t = 1 + \varepsilon_t$ for $\varepsilon_t \sim \N(0,1)$; otherwise, $Y_t = \varepsilon_t$. Here, we consider $86$ public datasets from OpenML \citep{OpenML2013} (see Appendix~\ref{sec:dataset} for the list of datasets), which vary in number of classes, number of features, and number of observations, as summarized in Table~\ref{tab:dataset}.

\paragraph{Data collection} A linear Thompson sampling agent \citep{agrawal2013thompson} is used to collect data with assignment probability floor $t^{-0.5}/K$, where $K$ is number of arms/classes. 
\label{sec:public_datasets}

\paragraph{Performance} For each dataset, we calculate bias and standard error (S.E.) of all estimators based on $100$ simulations. 


\begin{itemize}
    \item \textbf{Adaptive weighting \emph{v.s.} DR.} Figure~\ref{fig:dataset} uses DR as a baseline and shows its comparison with adaptive weighting and DM on various datasets. Each subplot  corresponds to an estimation method, wherein each point denotes a dataset, and its $(x,y)$ coordinates are bias and standard error normalized by RMSE of DR estimate on the same dataset. It can be seen that all adaptive weighting schemes reduce RMSE as compared to DR in most datasets. In many, RMSE is reduced by half. We also note that across various datasets, standard error accounts for most of the  RMSE with adaptive weighting and DR estimates, while bias is the main factor for DM. 
    \item \textbf{Non-contextual  \emph{v.s.} contextual.} Figure~\ref{fig:dataset_compare} proceeds to investigate within adaptive weighting, how contextual weighting performs as compared to its non-contextual counterpart. In the left panel, each point denotes a dataset, and its $(x,y)$ coordinates are bias and standard error of contextual \texttt{MinVar} normalized by RMSE of non-contextual \texttt{MinVar} on the same dataset. The same visualization approach is applied to contextual and non-contextual \texttt{StableVar}  in the right panel. We see a clear improvement in RMSE of  contextual weighting over  the non-contextual. 
    \item \textbf{\texttt{MinVar} \emph{v.s.} \texttt{StableVar}.} We compare the two adaptive weighting choices \texttt{MinVar} and \texttt{StableVar} in Figure~\ref{fig:dataset_compare_2}. As before, each point corresponds to a dataset. The $(x,y)$ coordinates are bias and standard error of \texttt{MinVar} normalized by its \texttt{StableVar} counterpart in forms of both non-contextual weighting (left panel) and contextual weighting (right panel). We can see that \texttt{MinVar} typically reduces RMSE relative to \texttt{StableVar}. 
\end{itemize}

\section{Discussion}
Off-policy evaluation on adaptive data can be challenging, as observations are dependent and overlap between the target policy and data-collection policy often deteriorates over the course of data collection.
Many estimators suffer from bias or high variance. In this paper, we propose a set of heuristics that address the issue by modifying the DR estimator through averaging doubly robust scores using adaptive, non-uniform weights. When chosen appropriately, these weights are able to reduce variance, at the cost of a small finite-sample bias. We prove that our estimators have consistency and asymptotic normality under certain conditions. Empirically, our estimators outperform existing alternatives with lower RMSE and expected coverage.

A number of important research directions remain. For example, it would be of interest to know if some of our more stringent sufficient conditions for asymptotic normality can be weakened, allowing for inference based on approximate normality in more general settings. Another interesting avenue of research would be to consider whether adaptive weights as presented here could be beneficial for learning optimal policies using logged data from bandit algorithms.

\section*{Acknowledgements}
The authors would like to thank Stefan Wager, Sanath Kumar Krishnamurthy, and Kristine Koutout for  helpful discussions. The authors are also grateful for the generous support provided by Golub Capital Social Impact Lab. S.A. acknowledges  support from the Office of Naval Research grant N00014-19-1-2468 and Schmidt Futures. R.Z. acknowledges support from the Total Graduate Fellowship and PayPal Graduate Fellowship. D.A.H. acknowledges support from the Stanford Institute for Human-Centered Artificial Intelligence.

\bibliographystyle{apalike}
\bibliography{reference}

\clearpage
\appendix

\section{Additional proofs}
With some abuse of notation, starting from this section, we will write
\begin{equation}
\label{eq:dr_single}
\widehat{\Gamma}_t(X_t, w) 
:=  \widehat{\mu}_t(X_t, w) + \frac{\one\{W_t=w\}}{e_t(X_t,w)}(Y_t-\widehat{\mu}_t(X_t, w)),
\end{equation}
so that $\widehat{\Gamma}_t(X_t, \pi) \equiv \sum_{w} \pi(X_t, w)\widehat{\Gamma}_t(X_t, w)$. 

\subsection{Proof of Proposition~\ref{prop:aipw_unbiased}}
\label{sec:proof_aipw_unbiased}
Recall that $\widehat{\mu}(X_t, w)$ and $e_t(X_t, w)$ are functions of past data $H_{t-1}$ and the current context $X_t$. Therefore, the rightmost term of \eqref{eq:dr_single} has conditional expectation
\begin{equation*}
    \begin{aligned}
        &
        \bE\Big[
        \frac{\one\{W_t=w\}}{e_t(X_t,w)}(Y_t-\widehat{\mu}_t(X_t, w))
        \Big|H_{t-1}, X_t\Big] \\
         \stackrel{\mbox{(i)}}{=}~&  \bE\Big[
            \frac
                {(Y_t(w)-\widehat{\mu}_t(X_t, w))}
                {e_t(X_t,w)}
            \Big|H_{t-1}, X_t, W_{t}=w \Big]e_t(X_t,w) \\
       =~ &\bE\big[Y_t(w)|H_{t-1}, X_t, W_{t}=w\big]
        -\widehat{\mu}_t(X_t, w) \\
        \stackrel{\mbox{(ii)}}{=}~ &\condE{}{Y_t(w)}{X_t}
        -\widehat{\mu}_t(X_t, w) \\
        =~&\mu(X_t, w) -\widehat{\mu}_t(X_t, w),
    \end{aligned}
\end{equation*}

\noindent where in (i) we used the definition of $e_t(X_t, w)$ and in (ii) we used Assumption~\ref{assu:data} to drop the conditioning on $H_{t-1}$ and $W_t$. This implies that
\begin{equation*}
    \begin{aligned}
    \bE\big[\widehat{\Gamma}_t(X_t, \pi)\big|H_{t-1}, X_t\big]
    = \sum_{w} \pi(X_t, w) \bE\big[\widehat{\Gamma}_t(X_t, w) \big|H_{t-1}, X_t\big] = \sum_{w} \pi(X_t, w)\mu(X_t, w).
    \end{aligned}
\end{equation*}
Take expectation with respect to context, we have
\begin{equation*}
    \bE\big[\widehat{\Gamma}_t(X_t, \pi)\big|H_{t-1}\big] = \bE\big[\sum_{w} \pi(X_t, w)\mu(X_t, w)\big] = Q(\pi).
\end{equation*}


\subsection{Proof of Proposition~\ref{prop:aipw_var}}
\label{sec:proof_aipw_var}



\noindent We first prove the lower bound of the conditional variances in \eqref{eq:gamma_condvar_x}. Expanding the definition of $\widehat{\Gamma}_t$ from \eqref{eq:aipw} and substituting \eqref{eq:dr_single}, we have that 

\begin{equation}    
    \label{eq:cond_var_expand}
    \begin{split}
       \Var{}{\widehat{\Gamma}_t(X_t, \pi)|H_{t-1}, X_t}=&    \sum_{w}\pi^2(X_t, w)\Var{}{\widehat{\Gamma}_t(X_t, w) |H_{t-1}, X_t}\\
     & + \sum_{w_1\neq w_2}  \pi(X_t, w_1)\pi(X_t, w_2)\Cov{\widehat{\Gamma}_t(w_1) }{\widehat{\Gamma}_t(w_2) |H_{t-1}, X_t}
    \end{split}
\end{equation}

\noindent Let us analyze these sums. A routine calculation reveals that the conditional variances decompose as
\begin{equation}
\label{eq:var_decomp}
\begin{aligned}
    \Var{}{\widehat{\Gamma}_t(X_t, w) |H_{t-1}, X_t} = &
    \frac{\Var{}{Y_t(w) | X_t}}{e_t(X_t, w)}
    +  \Big(\frac{1}{e_t(X_t, w)} - 1\Big)
    \left(\widehat{\mu}_t(X_t, w) - \mu(X_t, w) \right)^2,
\end{aligned}
\end{equation}
while the conditional covariances reduce to
\begin{equation}
 \label{eq:covar_decomp}
    \begin{aligned}
         &\Cov{\widehat{\Gamma}_t(w_1) }{\widehat{\Gamma}_t(w_2) |H_{t-1}, X_t}=\widehat{\mu}_t(X_t, w_1){\mu}_t(X_t, w_2) + {\mu}_t(X_t, w_1)\widehat{\mu}_t(X_t, w_2) - {\mu}_t(X_t, w_1){\mu}_t(X_t, w_2).
    \end{aligned}
\end{equation}

\noindent  Substituting \eqref{eq:var_decomp} and \eqref{eq:covar_decomp} in \eqref{eq:cond_var_expand},
\begin{equation}
    \label{eq:var_decomp_a}
    \begin{split}
        \Var{}{\widehat{\Gamma}_t(X_t, \pi)|H_{t-1}, X_t}
        = \sum_w\frac{\pi^2(X_t, w)\Var{}{Y_t(w)|X_t}}{e_t(X_t, w)} + A,
    \end{split}
\end{equation}

\noindent where $A$ collects all the terms involving $\widehat{\mu}$ and $\mu$, and after some algebra can be shown to be
\begin{align*}
        A := \sum_w\frac{\pi^2(X_t, w)(\hat{\mu}_t(X_t, w)-\mu(X_t, w))^2}{e_t(X_t, w)} -  \Big(\sum_w \pi(X_t, w)\big(\hat{\mu}_t(X_t, w)-\mu(X_t, w)\big)\Big)^2.
\end{align*}
Note that $A \geq 0$, since by the Cauchy-Schwartz inequality and recalling that  $\sum_{w}e_t(X_t, w)=1$,
\begin{align*}
      & \Big(\sum_w\frac{\pi^2(X_t, w)(\hat{\mu}_t(X_t, w)-\mu(X_t, w))^2}{e_t(X_t, w)}\Big)\times \Big(\sum_{w}e_t(X_t, w)\Big) \geq \Big(\sum_w \pi(X_t, w)\bbracket{\hat{\mu}_t(X_t, w)-\mu(X_t, w)}\Big)^2.
\end{align*}

\noindent Thus  from \eqref{eq:var_decomp_a} and the fact that $A \geq 0$, we have
\begin{equation*}
\begin{split}
    \Var{}{\widehat{\Gamma}_t(X_t, \pi)|H_{t-1}, X_t}&\geq \sum_w\frac{\pi^2(X_t, w)\Var{}{Y_t(w)|X_t}}{e_t(X_t, w)}\geq C_0\sum_w\frac{\pi^2(X_t, w)}{e_t(X_t, w)},
\end{split}
\end{equation*}
wherein the last inequality is due to the assumption that  $\Var{}{Y_t(w)|X_t}$ is uniformly bounded  below by $C_0 > 0$.

To prove the upper bound on conditional variances in \eqref{eq:gamma_condvar_x}, 
\begin{equation*}
    \begin{split}
       \Var{}{\widehat{\Gamma}_t(X_t, \pi)|H_{t-1}, X_t}
      \leq & \sum_w\frac{\pi^2(X_t, w)\Var{}{Y_t(w)|X_t}}{e_t(X_t, w)} + \sum_w\frac{\pi^2(X_t, w)(\hat{\mu}_t(X_t, w)-\mu(X_t, w))^2}{e_t(X_t, w)}\\
      \leq & \sum_w\frac{C_1\pi^2(X_t, w)}{e_t(X_t, w)} +  \sum_w\frac{4C_2^2\pi^2(X_t, w)}{e_t(X_t, w)},
    \end{split}
\end{equation*}
where for the last inequality, we use the assumption that $\Var{}{Y_t(w)|X_t}$ is uniformly  bounded above by some constant $C_1$ and that $\widehat{\mu}_t, \mu$ are uniformly bounded above by $C_2$. This justifies the result \eqref{eq:gamma_condvar_x}.

\section{Auxiliary theorems}

Our proof is established on martingale central limit theorems. Particularly, we invoke the following result  multiple times throughout the proof.
\begin{proposition}[Martingale CLT \cite{hall2014martingale}]
\label{prop:martingale_clt}
Let $\{\xi_{t}, \mathcal{F}_{T,t}\}_{t=1}^T$ be a martingale difference sequence (MDS) with bounded $4^{th}$ moments. Suppose that the following two conditions are satisfied,
\begin{itemize}
    \item variance convergence: $\sum_{t=1}^T \E{}{\xi_{T,t}^2|H_{t-1}}\xrightarrow{p} 1$;
    \item moment decay: $\sum_{t=1}^T\condE{}{\xi_{T,t}^4}{H_{t-1}}\xrightarrow{p}0$.
\end{itemize}
Then, $\sum_{t=1}^T \xi_{T,t}\xrightarrow{d}\N(0, 1)$.
\end{proposition}

The following proposition is also useful in our proof.
\begin{proposition}[Convergence of weighted arrays of random variables, Lemma 10, \cite{hadad2021confidence}]
\label{prop:weighted_convergence}
Let $a_{T,t}$ be a triangular sequence of nonnegative weight vectors satisfying $\mathrm{plim}_{T\rightarrow\infty}\max_{1\leq t\leq T} a_{T,t}=0$ and $\mathrm{plim}_{T\rightarrow\infty} \sum_{t=1}^T a_{T,t}\leq C$ for some constant $C$. Also, let $x_t$ be a sequence of random variables satisfying $x_t\xrightarrow[t\rightarrow\infty]{a.s.}0$. Then $\sum_{t=1}^T a_{T,t}x_t\xrightarrow[p]{T\rightarrow\infty}0$.

\end{proposition}

A similar proposition will also be used. 
\begin{proposition}[Convergence of weighted arrays of deterministic numbers]
\label{prop:weighted_convergence_2}
Let $a_{T,t}$ be a triangular sequence of nonnegative weight vectors satisfying $\lim_{T\rightarrow\infty}\max_{1\leq t\leq T} a_{T,t}=0$ and $\lim_{T\rightarrow\infty} \sum_{t=1}^T a_{T,t}\leq C$ for some constant $C$. Also, let $x_t$ be a sequence of numbers satisfying $x_t\xrightarrow{t\rightarrow\infty}0$. Then $\sum_{t=1}^T a_{T,t}x_t\xrightarrow{T\rightarrow\infty}0$.
\end{proposition}

\begin{proof}
    For any given $\epsilon>0$, since $x_t\rightarrow 0$, there exists a $T_0$ such that when $t>T_0$, $|x_t|<\epsilon$.
    \begin{equation}
        \begin{split}
            \bigg|\sum_{t=1}^T a_{T,t}x_t\bigg|& = \bigg| \sum_{t=1}^{T_0} a_{T,t}x_t +\sum_{t=T_0+1}^{T} a_{T,t}x_t \bigg|\\
           &\leq  \bigg| \sum_{t=1}^{T_0} a_{T,t}x_t \bigg|+ \bigg|\sum_{t=T_0+1}^{T} a_{T,t}x_t \bigg|\\
           &\leq T_0\max_{1\leq t\leq T_0} |a_{T,t}| \max_{1\leq t\leq T_0} |x_t| + \epsilon \sum_{t=1}^T a_{T,t}\\
           &\leq T_0\max_{1\leq t\leq T_0} |x_t|\max_{1\leq t\leq T} |a_{T,t}|  + \epsilon \sum_{t=1}^T a_{T,t},\\
        \end{split}
    \end{equation}
\end{proof}
where $\max_{1\leq t\leq T} |a_{T,t}|$ vanishes as $T\rightarrow\infty$, and $\sum_{t=1}^T a_{T,t}$ bounded, we thus have $ \bigg|\sum_{t=1}^T a_{T,t}x_t\bigg| = O(\epsilon)$ for large $T$, concluding the proof.

\section{Limit theorems of non-contextual weighting}
\label{sec:proof_noncontextual_clt}

In this section, we establish Theorem~\ref{thm:non_contextual}. At a high level, Theorem~\ref{thm:non_contextual} is an extension of limiting theorems for adaptive weighting in multi-armed bandits from \cite{hadad2021confidence}. Note that the non-contextual adaptive weight $h_t$ proposed in \eqref{eq:hnc_infeasible} is independent of the realized context $X_t$; when proving its asymptotics, we can view context as  exogenous  randomness and follow a similar martingale analysis framework as that in \cite{hadad2021confidence}. 

Recall that our goal is to evaluate a policy  $\pi$, whose policy value is $Q(\pi)=\bE[\sum_{w}\mu(x,w)]$. Our policy value estimate and the associated variance estimate are 
\begin{align*}
    \widehat{Q}_T^{NC}(\pi) =  \frac{\sum_{t=1}^T h_t \widehat{\Gamma}_t(X_t, \pi) }{\sum_{t=1}^T h_t}\quad \mbox{and}\quad \widehat{V}_T^{NC}(\pi)  = \frac{\sum_{t=1}^T h_t^2 (\widehat{\Gamma}_t(X_t, \pi)  - \widehat{Q}_T^{NC}(\pi) )^2 }{(\sum_{t=1}^T h_t)^2},
\end{align*}
with \StableVar weights $h_t = \sqrt{1 / \bE[\sum_{w}\pi^2(X_t, w)/e_t(X_t, w)|H_{t-1}]}$.

To start, we define an auxiliary martingale difference sequence (MDS) $\{\xi_{T,t}\}_{t=1}^T$ for   estimator $\widehat{Q}^{NC}_T$,
\begin{equation}
\label{eq:def_xi}
        \xi_{T, t} = \frac{h_t\big(\widehat{\Gamma}_t(X_t, \pi) - Q(\pi)\big)}{\sqrt{\sum_{t=1}^T \bE\big[h_t^2\big(\widehat{\Gamma}_t(X_t, \pi) -Q(\pi)\big)^2 \big]} }, 
\end{equation}
which martingale difference property is justified since $\bE[\xi_{T,t}|H_{t-1}]=0$ by Proposition~\ref{prop:aipw_unbiased}, which immediately yields $\bE[(\sum_{t=1}^T\xi_{T,t})^2] = 1$: 
\begin{equation}
    \begin{aligned}
   \bE[(\sum_{t=1}^T\xi_{T,t})^2] =   \bE[\sum_{t=1}^T\xi^2_{T,t}] = \frac{\sum_{t=1}^T \bE\big[h_t^2\big(\widehat{\Gamma}_t(X_t, \pi) -Q(\pi)\big)^2 \big] }{
   \sum_{t=1}^T \bE\big[h_t^2\big(\widehat{\Gamma}_t(X_t, \pi) -Q(\pi)\big)^2 \big]
   } = 1.
\end{aligned}
\label{eq:xi_2}
\end{equation}
We proceed to show the following steps in details.

\begin{itemize}
    \item \emph{Step I: consistency of $\widehat{Q}^{NC}_T(\pi) $,} that is, $\widehat{Q}^{NC}_T(\pi) \xrightarrow{p} Q(\pi)$.
    \item \emph{Step II: CLT of MDS $\{\xi_{T,t}\}_{t=1}^T$,} that is, $\sum_{t=1}^T \xi_{T,t}\xrightarrow{d}\N(0,1)$.
    \item \emph{Step III: CLT of $\widehat{Q}^{NC}_T(\pi) $,} that is,  $\big(\widehat{Q}^{NC}_T(\pi)  - Q(\pi)\big) / \big(\widehat{V}_T^{NC}(\pi) \big)^{1/2}\xrightarrow{d}\N(0, 1)$.
\end{itemize}

\subsection{Consistency of \texorpdfstring{$\widehat{Q}^{NC}_T$}{Qhat-NC}}
We now  show that $\widehat{Q}^{NC}_T(\pi) $ converges to true policy value $Q(\pi)$. We have
    \begin{equation}
         \begin{aligned}
       \big|\widehat{Q}^{NC}_T(\pi)  - Q(\pi)\big| & = \Big|\frac{\sum_{t=1}^T h_t(\widehat{\Gamma}_t(X_t, \pi)  - Q(\pi))}{\sum_{t=1}^T h_t}\Big|\\
       & =  \Big|\frac{\sum_{t=1}^T h_t(\widehat{\Gamma}_t(X_t, \pi)  - Q(\pi))}{\sqrt{\sum_{t=1}^T \bE\big[h_t^2\big(\widehat{\Gamma}_t(X_t, \pi) -Q(\pi)\big)^2}\big]}\Big|\cdot \frac{\sqrt{\sum_{t=1}^T \bE\big[h_t^2\big(\widehat{\Gamma}_t(X_t, \pi) -Q(\pi)\big)^2}\big]}{\sum_{t=1}^T h_t}\\
        & = \Big|\sum_{t=1}^T \xi_{T,t}\Big|\cdot \frac{\sqrt{\sum_{t=1}^T \bE\big[h_t^2\big(\widehat{\Gamma}_t(X_t, \pi) -Q(\pi)\big)^2}\big]}{\sum_{t=1}^T h_t}\\
         & = \Big|\sum_{t=1}^T \xi_{T,t}\Big|\cdot \frac{\sqrt{\sum_{t=1}^T \bE\big[h_t^2\bE\big[\big(\widehat{\Gamma}_t(X_t, \pi) -Q(\pi)\big)^2|H_{t-1}\big]\big]}}{\sum_{t=1}^T h_t}\\
          &\stackrel{\mbox{(i)}}{\leq}\Big|\sum_{t=1}^T \xi_{T,t}\Big|\cdot \frac{\sqrt{\sum_{t=1}^T \bE\big[(\bE[\sum_w \pi(X_t, w)/e_t(X_t, w)|H_{t-1}])^{-1}\cdot U \bE[\sum_w \pi(X_t, w)/e_t(X_t, w)|H_{t-1}]\big]}}{\sum_{t=1}^T (\bE[\sum_w \pi(X_t, w)/e_t(X_t, w) |H_{t-1}])^{-1/2}}\\
                    &=\Big|\sum_{t=1}^T \xi_{T,t}\Big|\cdot \frac{\sqrt{\sum_{t=1}^T U}}{\sum_{t=1}^T (\bE[\sum_w \pi(X_t, w)/e_t(X_t, w)|H_{t-1}])^{-1/2}}\\
         &\stackrel{\mbox{(ii)}}{\lesssim} \Big|\sum_{t=1}^T \xi_{T,t}\Big| \cdot \frac{\sqrt{T}}{\sum_{t=1}^T t^{-\alpha/2} } =\Big|\sum_{t=1}^T \xi_{T,t}\Big|\cdot O(T^{(\alpha-1)/2}),
    \end{aligned}
    \end{equation}
where in (i), we use \eqref{eq:gamma_condvar} and the definition of \StableVar weights; in (ii), we use Assumption~\ref{assu:clt_condition}b that $e_t\geq C t^{-\alpha}$. Therefore, for any $\epsilon>0$,
\begin{align*}
    \p\big(\big|\widehat{Q}^{NC}_T(\pi)  - Q(\pi)\big| >\epsilon \big) &\leq \epsilon^{-2} \bE\big[ \big|\widehat{Q}^{NC}_T(\pi)  - Q(\pi)\big|^2 \big] \\
   & \leq  \epsilon^{-2} \bE\big[ \big(\sum_{t=1}^T \xi_{T,t}\big)^2 \big]O(T^{\alpha-1})  
   = \epsilon^{-2}O(T^{\alpha-1})\rightarrow 0,
\end{align*}
where we use the fact that $\bE\big[ \big(\sum_{t=1}^T \xi_{T,t}\big)^2 \big]=\bE\big[\sum_{t=1}^T \xi_{T,t}^2 \big]=1$ by \eqref{eq:xi_2}.

\subsection{CLT of martingale difference sequence \texorpdfstring{$\{\xi_{T,t}\}_{t=1}^T$}{MDS-NC}}
\label{appendix:clt_xi}
We show the convergence of  $\sum_{t=1}^T\xi_{T,t}$ by verifying two martingale CLT conditions stated in Proposition~\ref{prop:martingale_clt}. 

\subsubsection{Variance convergence.} 
\label{sec:variance-convergence-noncontextual}

We want to show that the following ratio converges in probability to 1. 
\begin{align}
    \label{eq:xi2}
    \sum_{t=1}^T\condE{}{\xi^2_{T,t}}{H_{t-1}} 
    = \frac{\sum_{t=1}^T h_t^2 \bE\big[\big(\widehat{\Gamma}_t(X_t, \pi) - Q(\pi)\big)^2| H_{t-1}\big]
     }{\sum_{t=1}^T \bE\big[h_t^2\big(\widehat{\Gamma}_t(X_t, \pi) -Q(\pi)\big)^2\big]}.
\end{align}

We will later show that \eqref{eq:xi2} can be rewritten as $Z_T/\bE[Z_T]$, where
\begin{equation}
    \label{eq:zt_num}
    Z_T = 
     \sum_{t=1}^{T} \EE\Big[h_t^2 \sum_w \frac{C_1(X_t, w)\pi^2(X_t, w)}{e_t(X_t, w)} | H_{t-1}\Big] + C_2\sum_{t=1}^{T} h_t^2 +   o_p\left( \sum_{t=1}^{T} \EE\Big[h_t^2\Big\{ \sum_w \frac{C_1(X_t, w)\pi^2(X_t, w)}{e_t(X_t, w)} + C_2\Big\} | H_{t-1}\Big] \right),
\end{equation}
and
\begin{equation}
    \label{eq:zt_denom}
    \EE[Z_T] = 
   \sum_{t=1}^{T} \EE\Big[h_t^2 \sum_w \frac{C_1(X_t, w)\pi^2(X_t, w)}{e_t(X_t, w)} \Big] + C_2\sum_{t=1}^{T} \EE[h_t^2] +   o\left( \sum_{t=1}^{T} \EE\Big[h_t^2 \Big\{\sum_w \frac{C_1(X_t, w)\pi^2(X_t, w)}{e_t(X_t, w)}+ C_2\Big\}\Big] \right).
\end{equation}
Above,  $C_1$ is a function of $X_t$ and $w$ but not the history, bounded above and away from zero; $C_2$ is a finite constant. This characterization implies that that \eqref{eq:xi2} is asymptotically equivalent to the following ratio of dominant terms in \eqref{eq:zt_num} and \eqref{eq:zt_denom}, which we denote as 
\begin{equation}
\label{eq:xi2_target}
    \frac{A_T}{\EE A_T} := 
    \frac{  \sum_{t=1}^{T} \EE\Big[h_t^2 \sum_w C_1(X_t, w)\frac{\pi^2(X_t, w)}{e_t(X_t, w)} | H_{t-1}\Big] + C_2\sum_{t=1}^{T} h_t^2  }
         { \sum_{t=1}^{T} \EE\Big[h_t^2 \sum_w C_1(X_t, w)\frac{\pi^2(X_t, w)}{e_t(X_t, w)} \Big] + C_2\sum_{t=1}^{T} \EE[h_t^2]}.
\end{equation}

For now, let's assume that we have characterized \eqref{eq:xi2} as $Z_T/\EE Z_T$ and show that \eqref{eq:xi2_target} converges in probability to one. We will show that both the numerator and the denominator of \eqref{eq:xi2_target}  are equivalent to $M_T$ asymptotically, where
\begin{equation}
    \label{eq:M}
    M_T := \sum_{t=1}^T \EE\Big[\sum_w \frac{C_1(X_t, w)\pi^2(X_t, w)}{e_t(X_t, w)}\Big]\Big/\bE\Big[\sum_w \frac{\pi^2(X_t, w)}{e_t(X_t, w)}\Big] + \sum_{t=1}^TC_2\Big/\bE\Big[\sum_w \frac{\pi^2(X_t, w)}{e_t(X_t, w)}\Big].
\end{equation}    

It will be useful to note that the sequence $M_T$ is 
$\Omega(T)$. This is shown later in \eqref{eq:M_lower}.

\vspace{\baselineskip}

Let's begin by showing that $A_T/M_T \to 1$ almost surely as $T \to \infty$ or, equivalently, that $A_T/M_T - 1 \to 0$ almost surely. This ratio decomposes as follows,
\begin{align}
\label{eq:xi_hCe2_p}
      \frac{A_T}{M_T} - 1 =  & \frac{\sum_{t=1}^{T} \EE\Big[h_t^2 \sum_w\frac{ C_1(X_t, w)\pi^2(X_t, w)}{e_t(X_t, w)} | H_{t-1}\Big] + C_2\sum_{t=1}^{T} h_t^2  }{M_T} - 1
    = \frac{
    \sum_{t=1}^{T}a_t m_t +C_2  b_t n_t }{M_T}\\
\mbox{where}\quad & a_t := \frac{ \EE\Big[h_t^2 \sum_w\frac{ C_1(X_t, w)\pi^2(X_t, w)}{e_t(X_t, w)} | H_{t-1}\Big]}{\EE[\sum_w \frac{C_1(X_t, w)\pi^2(X_t, w)}{e_t(X_t, w)}]\big/\bE[\sum_w \frac{\pi^2(X_t, w)}{e_t(X_t, w)}]} - 1, \quad b_t := \frac{ h_t^2}{\bE[\sum_w \frac{\pi^2(X_t, w)}{e_t(X_t, w)}]^{-1}} -1\nonumber\\
 & m_t := \EE[\sum_w \frac{C_1(X_t, w)\pi^2(X_t, w)}{e_t(X_t, w)}]\big/\bE[\sum_w \frac{\pi^2(X_t, w)}{e_t(X_t, w)}], \quad n_t :=\bE[\sum_w \frac{\pi^2(X_t, w)}{e_t(X_t, w)}]^{-1}.\nonumber 
\end{align}
Below, we will show that $a_t\xrightarrow{a.s.}0$ and $b_t\xrightarrow{a.s.}0$. Consequently, for any $\epsilon>0$, there exists $T_0$ such that $|a_t|<\epsilon$ and $|b_t|<\epsilon$ for $t \ge T_0$ on an event of probability one. We work on this event. Note that $m_t$ is bounded because $C_1$ is bounded; $n_t \in (0,1]$ because $e_t \le 1$ and $\sum_w \pi(x,w)^2 \le \sum_w \pi(x,w) = 1$ for all $x$;
and below, in  \eqref{eq:M_lower},
we show that  $M_T=\Omega(T)$. It follows that $\sum_{t=T_0}^T a_tm_t+C_2 b_tn_t\le \epsilon \sum_{t=T_0}^T (m_t + C_2 n_t) \le \epsilon M_T$. 
Furthermore, because $a_t$ and $b_t$ are convergent, they are bounded, and
it follows that $\sum_{t=1}^{T_0} a_tm_t+C_2 b_tn_t$ is bounded and therefore less than $\epsilon M_T = \Omega(T)$ for large $T$. We have decomposed the numerator of the right side of  \eqref{eq:xi_hCe2_p} into two terms, and shown both are  arbitrarily small relative to the denominator; it follows that  $A_T / M_T-1$
converges to $0$ on the aforementioned event, i.e., almost surely.


\vspace{\baselineskip}

Because the ratio $A_T / M_T$ is bounded, $\bE[A_T/M_T]$ also converges to one as consequence of the above and dominated convergence.  To see that it is bounded, recall that $M_T=\Omega(T)$ and observe that $A_T=O(T)$, as each summand in $A_T$ is $O(1)$. In particular, the summands satisfy
\begin{equation}
    \begin{split}
    &h_t^2\{\bE[\sum_w\frac{ C_1(X_t, w)\pi^2(X_t, w)}{e_t(X_t, w)} | H_{t-1}] + C_2\} \\
\leq& h_t^2 (\sup_{x,w}|C_1(x,w)|+|C_2|)\bE[\sum_w\frac{ \pi^2(X_t, w)}{e_t(X_t, w)} | H_{t-1}]\\
=&(\sup_{x,w}|C_1(x,w)|+|C_2|) < \infty,
\end{split}
\end{equation}
as $h_t^2$ is the inverse of the conditional expectation in the second line above. 

\vspace{\baselineskip}

Note that $M_T$ is not random, so this implies that $\EE[A_T] / M_T \to 1$.
Thus, as $A_T/M_T \xrightarrow{a.s.} 1$,
it follows that $A_T / E A_T = (A_T / M_T) / (\EE[A_T] / M_T) \xrightarrow{a.s.} 1$.

\vspace{\baselineskip}

Next, let's show that  $b_t$ defined in \eqref{eq:xi_hCe2_p} converges to $0$ almost surely. We'll show that the equivalent is true:
\begin{equation}
\label{eq:eX_converged}
\begin{split}
       &\bigg| \frac{\bE\big[\sum_w\frac{\pi^2(X_t, w)}{e_t(X_t,w)} |H_{t-1}\big]}{\bE\big[\sum_w\frac{\pi^2(X_t, w)}{e_t(X_t,w)}\big]} - 1 \bigg| \\
       \stackrel{(i)}{=}&
           \bigg| \frac{\int_{x}\sum_w\frac{\pi^2(x, w)}{e_t(x,w)} d\p_{x}}{\bE\big[\sum_w\frac{\pi^2(X_t, w)}{e_t(X_t,w)}\big]} - 1 \bigg| 
           \\
                 \stackrel{(ii)}{=}&    \bigg| \frac{\int_{x}\sum_w\frac{\pi^2(x, w)}{e_t(x,w)}  - \bE[\sum_w\frac{\pi^2(x, w)}{e_t(x,w)} ] d\p_{x}}{\bE\big[\sum_w\frac{\pi^2(X_t, w)}{e_t(X_t,w)}\big]}  \bigg|  
                \\
           \leq&     \frac{\int_{x}|\sum_w\frac{\pi^2(x, w)}{e_t(x,w)}  - \bE[\sum_w\frac{\pi^2(x, w)}{e_t(x,w)} ]| d\p_{x}}{\bE\big[\sum_w\frac{\pi^2(X_t, w)}{e_t(X_t,w)}\big]}     =  \frac{ \int_{x}  \Big|\frac{\sum_w\frac{\pi^2(x, w)}{e_t(x,w)}}{\bE[\sum_w\frac{\pi^2(x, w)}{e_t(x,w)}]} -1
       \Big|\bE[\sum_w\frac{\pi^2(x, w)}{e_t(x,w)}]d\p_{x}  }{
       \bE\big[\sum_w\frac{\pi^2(X_t, w)}{e_t(X_t,w)}\big]
       }\\  
       \leq & \frac{\sup_{x}\Big|\frac{\sum_w\frac{\pi^2(x, w)}{e_t(x,w)}}{\bE[\sum_w\frac{\pi^2(x, w)}{e_t(x,w)}]} -1
       \Big| \int_{x} \bE[\sum_w\frac{\pi^2(x, w)}{e_t(x,w)}]d\p_{x}  }{
       \bE\big[\sum_w\frac{\pi^2(X_t, w)}{e_t(X_t,w)}\big]
       }  \\
       = & \frac{\sup_{x}\Big|\frac{\sum_w\frac{\pi^2(x, w)}{e_t(x,w)}}{\bE[\sum_w\frac{\pi^2(x, w)}{e_t(x,w)}]} -1
       \Big| \cdot \bE\big[\sum_w\frac{\pi^2(X_t, w)}{e_t(X_t,w)}\big]  }{
       \bE\big[\sum_w\frac{\pi^2(X_t, w)}{e_t(X_t,w)}\big]
       }  \\
       = & \sup_{x}\Big|\frac{\sum_w\frac{\pi^2(x, w)}{e_t(x,w)}}{\bE[\sum_w\frac{\pi^2(x, w)}{e_t(x,w)}]} -1\Big|\\
        \stackrel{(iii)}{\leq} & \sup_x \bigg|\sum_w \frac{\frac{\pi^2(x, w)}{e_t(x,w)} }{\bE[\frac{\pi^2(x, w)}{e_t(x,w)}]}-1\bigg| \leq  \sup_x \sum_w \bigg|\frac{\frac{\pi^2(x, w)}{e_t(x,w)} }{\bE[\frac{\pi^2(x, w)}{e_t(x,w)}]}-1\bigg|\\
    \stackrel{(iv)}{=}& \sup_{x}\sum_w  \bigg|\frac{e_t^{-1}(x,w) }{\bE[e_t^{-1}(x,w)]}-1\bigg|\leq K \sup_{x,w} \bigg|\frac{e_t^{-1}(x,w) }{\bE[e_t^{-1}(x,w)]}-1\bigg|\xrightarrow{a.s.}0
       ,
\end{split}
\end{equation}
where in (i), we rewrite $ \bE[\sum_w\frac{\pi^2(X_t, w)}{e_t(X_t,w)} |H_{t-1}]$ as $\int_{x}\sum_w\frac{\pi^2(x, w)}{e_t(x,w)} d\p_{x}$ since $ \bE[\sum_w\frac{\pi^2(X_t, w)}{e_t(X_t,w)} |H_{t-1}]$ is an expectation over $X_t=x$ conditioning on history $H_{t-1}$; in (ii), we rewrite $ \bE[\sum_w\frac{\pi^2(X_t, w)}{e_t(X_t,w)}]$ as  
$\int_{x}\bE[\sum_w\frac{\pi^2(x, w)}{e_t(x,w)} ] d\p_{x}$, since
$ \bE[\sum_w\frac{\pi^2(X_t, w)}{e_t(X_t,w)}]$ is an expectation over $(X_t, H_{t-1})$, and here, we use law of iterated expectations to integrate out $H_{t-1}$ first; in (iii), we substitute a smaller denominator for each term in the numerator then use the triangle inequality;
(iv),  the equation is true since we can pull the $\pi^2(x,w)$ out of the expectation with fixed $(x,w)$. This concludes that $b_t\xrightarrow{a.s.}0$.

\vspace{\baselineskip}


\vspace{\baselineskip}

Next we are proving $a_t\xrightarrow{a.s.}0$. Note that $(a_t+1) = (b_t+1) \cdot \frac{\bE\big[\sum_w\frac{C_1(X_t, w)\pi^2(X_t, w)}{e_t(X_t,w)} |H_{t-1}\big]}{\bE\big[\sum_w\frac{C_1(X_t, w)\pi^2(X_t, w)}{e_t(X_t,w)}\big]}$. We already have $b_t\xrightarrow{a.s.}0$, to show $a_t\xrightarrow{a.s.}0$, we only need to show that the latter factor converges to one. We use an argument similar to the one above.
\begin{equation}
\label{eq:xi_hCe2}
    \begin{split}
     &\bigg| \frac{\bE\big[\sum_w\frac{C_1(X_t, w)\pi^2(X_t, w)}{e_t(X_t,w)} |H_{t-1}\big]}{\bE\big[\sum_w\frac{C_1(X_t, w)\pi^2(X_t, w)}{e_t(X_t,w)}\big]} - 1 \bigg|  \\
     &= \bigg| \frac{\int_{x}\sum_w\frac{C_1(x, w)\pi^2(x, w)}{e_t(x,w)} d\p_{x}}{\bE\big[\sum_w\frac{C_1(X_t, w)\pi^2(X_t, w)}{e_t(X_t,w)}\big]} - 1 \bigg|\\
    &= \bigg| \frac{\int_{x}\sum_w\frac{C_1(x, w)\pi^2(x, w)}{e_t(x,w)} - \bE[\sum_w\frac{C_1(x, w)\pi^2(x, w)}{e_t(x,w)}] d\p_{x} }{\bE\big[\sum_w\frac{C_1(X_t, w)\pi^2(X_t, w)}{e_t(X_t,w)}\big]}  \bigg|\\
    &\leq  \frac{\int_{x}\Big|\frac{\sum_w\frac{C_1(x, w)\pi^2(x, w)}{e_t(x,w)}}{ \bE[\sum_w\frac{C_1(x, w)\pi^2(x, w)}{e_t(x,w)}]}-1\Big|\bE[\sum_w\frac{C_1(x, w)\pi^2(x, w)}{e_t(x,w)}] d\p_{x} }{\bE\big[\sum_w\frac{C_1(X_t, w)\pi^2(X_t, w)}{e_t(X_t,w)}\big]} \\
       &\leq \frac{\sup_{x}\Big|\frac{\sum_w\frac{C_1(x, w)\pi^2(x, w)}{e_t(x,w)}}{\bE\big[\sum_w\frac{C_1(x, w)\pi^2(x, w)}{e_t(x,w)}\big]} -1
       \Big| \int_{x}\bE[ \sum_w\frac{C_1(x, w)\pi^2(x, w)}{e_t(x,w)}]d\p_{x}  }{
       \bE\big[\sum_w\frac{C_1(X_t, w)\pi^2(X_t, w)}{e_t(X_t,w)}\big]
       }\\
       & = \sup_{x}\Big|\frac{\sum_w\frac{C_1(x, w)\pi^2(x, w)}{e_t(x,w)}}{\bE\big[\sum_w\frac{C_1(x, w)\pi^2(x, w)}{e_t(x,w)}\big]} -1
       \Big|\\
       &\leq \sup_x  \bigg|\sum_w\frac{\frac{C_1(x,w)\pi^2(x, w)}{e_t(x,w)} }{\bE[\frac{C_1(x,w)\pi^2(x, w)}{e_t(x,w)}]}-1\bigg| \leq \sup_x \sum_w \bigg|\frac{\frac{C_1(x,w)\pi^2(x, w)}{e_t(x,w)} }{\bE[\frac{C_1(x,w)\pi^2(x, w)}{e_t(x,w)}]}-1\bigg|\\
   & =\sup_{x}\sum_w  \bigg|\frac{e_t^{-1}(x,w) }{\bE[e_t^{-1}(x,w)]}-1\bigg|\leq K \sup_{x,w} \bigg|\frac{e_t^{-1}(x,w) }{\bE[e_t^{-1}(x,w)]}-1\bigg|\xrightarrow{a.s.}0.
\end{split}
\end{equation}

This completes the proof that $a_t$ converges to one almost surely. 

\clearpage

Next we'll show  \eqref{eq:xi2} can be rewritten as $Z_T/\bE[Z_T]$. Let's start by understanding the behavior the quantity in the numerator of \eqref{eq:xi2}. 

\begin{equation}
\label{eq:decompose_numerator_a}
    \begin{aligned}
        \bE[(\hGamma_t(X_t,\pi) - Q(\pi))^2|H_{t-1}, X_t] & =  \bE[(\hGamma_t(X_t,\pi) - \sum_w \pi(X_t, w)\mu(X_t, w))^2|H_{t-1}, X_t] + (\sum_w \pi(X_t, w)\mu(X_t, w)-Q(\pi))^2.
    \end{aligned}
\end{equation}

By derivation in Appendix \ref{sec:proof_aipw_var} (see Equation~\ref{eq:var_decomp_a}) , we have
\begin{equation}
\label{eq:decompose_numerator_b}
\begin{aligned}
    &\bE[(\hGamma_t(X_t,\pi) - \sum_w \pi(X_t, w)\mu(X_t, w))^2|H_{t-1}, X_t] =   \Var{}{\widehat{\Gamma}_t(X_t, \pi)|H_{t-1}, X_t} \\
    &= \sum_w\frac{\pi^2(X_t, w)\Var{}{Y_t(w)|X_t}}{e_t(X_t, w)} + \sum_w\frac{\pi^2(X_t, w)(\hat{\mu}_t(X_t, w)-\mu(X_t, w))^2}{e_t(X_t, w)} -  \Big(\sum_w \pi(X_t, w)\big(\hat{\mu}_t(X_t, w)-\mu(X_t, w)\big)\Big)^2.
\end{aligned}
\end{equation}

Combing \eqref{eq:decompose_numerator_a} and \eqref{eq:decompose_numerator_b}, we have
\begin{equation}
\label{eq:decompose_numerator}
    \begin{aligned}
        &\bE[(\hGamma_t(X_t,\pi) - Q(\pi))^2|H_{t-1}] =   \bE[\bE[(\hGamma_t(X_t,\pi) - Q(\pi))^2|X_t, H_{t-1}] |H_{t-1}]\\
    = &\bE\Big[
    \sum_w\frac{\pi^2(X_t, w)[\Var{}{Y_t(w)|X_t} + (\hat{\mu}_t(X_t, w)-\mu(X_t, w))^2  ]}{e_t(X_t, w)}  |H_{t-1} \Big]\\
    &\quad - \bE\Big[\Big(\sum_w \pi(X_t, w)\big(\hat{\mu}_t(X_t, w)-\mu(X_t, w)\big)\Big)^2 |H_{t-1}\Big] +\bE[(\sum_w \pi(X_t, w)\mu(X_t, w)-Q(\pi))^2] \\
     = & \bE\Big[
    \sum_w C_1(X_t, w)\frac{\pi^2(X_t, w)}{e_t(X_t, w)}  |H_{t-1} \Big] + C_2 + \bE[m_t(X_t)|H_{t-1}] 
    \end{aligned}
\end{equation}
where $C_1(X_t, w) = \Var{}{Y_t(w)|X_t} + (\mu_\infty(X_t, w)-\mu(X_t, w))^2  $, $C_2 =\bE[(\sum_w \pi(X_t, w)\mu(X_t, w)-Q(\pi))^2]- \bE[(\sum_w \pi(X_t, w)(\mu_\infty(X_t, w)-\mu(X_t, w)))^2 ] $. Note neither $C_1$ or $C_2$ depend on the history $H_{t-1}$. The variable $m_t$ collects the following terms:
\begin{equation}
    \label{eq:mt_def}
\begin{split}
    m_t(X_t) =& \sum_{w}\frac{\pi^2(X_t, w)}{e_t(X_t, w)} (\hat{\mu}_t(X_t, w)-\mu_\infty(X_t, w))^2 -(\sum_w \pi(X_t, w)(\hat{\mu}_t(X_t, w)-\mu_\infty(X_t, w)))^2\\
    &+ 2\sum_{w}\frac{\pi^2(X_t, w)}{e_t(X_t, w)}(\hat{\mu}_t(X_t, w)-\mu_\infty(X_t, w))(\mu_\infty(X_t, w)-\mu(X_t, w)) 
    \\&-2(\sum_w \pi(X_t, w)(\hat{\mu}_t(X_t, w)-\mu_\infty(X_t, w)))(\sum_w \pi(X_t, w)(\mu_\infty(X_t, w)-\mu(X_t, w))).
\end{split}
\end{equation}
Below in \eqref{eq:mt_upper_deriv} we will prove that $m_t$ is upper bounded as
\begin{equation}
    \label{eq:mt_upper}
    \begin{aligned}
   | m_t(X_t) |
  \lesssim & \sup_{x, w}|\hat{\mu}_t(x, w)-\mu_\infty(x, w)|\sum_{w}\frac{\pi^2(X_t, w)}{e_t(X_t, w)}.
\end{aligned}
\end{equation}
For now, let's take this as fact and show that $Z_T$ and $\EE Z_T$ decompose as claimed in \eqref{eq:zt_num} and \eqref{eq:zt_denom}.

To show that the negligible terms in \eqref{eq:zt_num} are indeed negligible, we consider the ratio between these terms and the dominant term in $Z_T$,
\begin{equation}
\label{eq:negligible_x}
\begin{split}
 & \frac{\sum_{t=1}^{T} \EE[h_t^2 m_t(X_t)| H_{t-1}]}
     { \sum_{t=1}^{T} \EE\Big[h_t^2 \Big\{\sum_w\frac{ C_1(X_t, w)\pi^2(X_t, w)}{e_t(X_t, w)}+C_2\Big\} | H_{t-1}\Big] } \\
     \lesssim &
           \frac{\sum_{t=1}^{T} \sup_{x, w}|\hat{\mu}_t(x, w)-\mu_\infty(x, w)| \EE[h_t^2 \sum_{w}\frac{\pi^2(X_t, w)}{e_t(X_t, w)}  |H_{t-1}]}
     { \sum_{t=1}^{T} \inf_{x,w}\var(Y_t(w)|x)\EE[h_t^2 \sum_w \frac{\pi^2(X_t, w)}{e_t(X_t, w)} | H_{t-1}] } \\
     = &\frac{\sum_{t=1}^{T} \sup_{x, w}|\hat{\mu}_t(x, w)-\mu_\infty(x, w)| }
     { \inf_{x,w}\var(Y_t(w)|x)T }\xrightarrow{a.s}0,
\end{split}
\end{equation}
where in the first inequality we substituted $m_t$ by its upper bound \eqref{eq:mt_upper} in the numerator, and in the denominator we used a lower bound shown below in \eqref{eq:C1C2}. In the equality we used the definition of $h_t$. The limit is due to the fact that $\sup_{x,w} |\hat{\mu}_t(x,w) - \mu_{\infty}(x,w)|$ converges to zero almost surely, and therefore so does $(1/T) \sum_{t} \sup_{x,w} |\hat{\mu}_t(x,w) - \mu_{\infty}(x,w)|$.
\vspace{\baselineskip}

Similarly, to show that the negligible terms in \eqref{eq:zt_denom} are indeed negligible, we consider the following ratio,
\begin{equation}
\begin{aligned}
              \frac{\sum_{t=1}^{T} \EE[h_t^2 m_t(X_t)]}
             { \sum_{t=1}^{T} \EE\Big[h_t^2 \Big\{\sum_w \frac{C_1(X_t, w)\pi^2(X_t, w)}{e_t(X_t, w)}+C_2\Big\}  \Big] } 
             &\lesssim
                   \frac{\sum_{t=1}^{T} \bE\big[\sup_{x, w}|\hat{\mu}_t(x, w)-\mu_\infty(x, w)| h_t^2\EE[ \sum_{w}\frac{\pi^2(X_t, w)}{e_t(X_t, w)} |H_{t-1} ]\big]}
             { \sum_{t=1}^{T} \inf_{x,w}\var(Y_t(w)|x)\EE[h_t^2 \sum_w \frac{\pi^2(X_t, w)}{e_t(X_t, w)}] }\\
            &=\frac{\sum_{t=1}^{T} \bE\big[\sup_{x, w}|\hat{\mu}_t(x, w)-\mu_\infty(x, w)| \big]}
             {  \inf_{x,w}\var(Y_t(w)|x) T }
\end{aligned}
\end{equation}
Since $\hat{\mu}_t, \mu_\infty$ is bounded and $\sup_{x, w}|\hat{\mu}_t(x, w)-\mu_\infty(x, w)|\xrightarrow{a.s}0$ by Assumption \ref{assu:clt_condition}, we have $\sup_{x, w}|\hat{\mu}_t(x, w)-\mu_\infty(x, w)|\xrightarrow{L_1}0$ and thus $  \bE[\sup_{x, w}|\hat{\mu}_t(x, w)-\mu_\infty(x, w)|]\rightarrow 0$; this yields
\begin{equation}
  \frac{\sum_{t=1}^{T} \EE[h_t^2 m_t(X_t)]}
             { \sum_{t=1}^{T} \EE\Big[h_t^2 \sum_w C_1(X_t, w)\frac{\pi^2(X_t, w)}{e_t(X_t, w)} \Big] } \rightarrow 0.
\end{equation}

Finally, to conclude the proof, let's establish some auxiliary results used above. First, a lower bound on the dominant terms in $Z_T$. By expanding the definition of $C_1$ and $C_2$ from the discussion following \eqref{eq:decompose_numerator},
\begin{equation}
    \label{eq:C1C2}
  \begin{aligned}
          \bE\Big[\sum_w \frac{C_1(X_t,w)\pi^2(X_t, w)}{e_t(X_t, w)}\Big] + C_2 
        &=\bE\Big[\sum_w \frac{[\Var{}{Y_t(w)|X_t}\pi^2(X_t, w)}{e_t(X_t, w)} \Big] \\
          &+ \EE \Big[(\mu_\infty(X_t, w)-\mu(X_t, w))^2 \frac{\pi^2(X_t, w)}{e_t(X_t, w)}\Big]\\
          &+ \bE[(\sum_w \pi(X_t, w)\mu(X_t, w)-Q(\pi))^2] \\ 
          &- \bE[(\sum_w \pi(X_t, w)(\mu_\infty(X_t, w)-\mu(X_t, w)))^2 ]
    \end{aligned}
\end{equation}

The third term of this decomposition is non-negative. The second term is larger than the fourth in magnitude, since by Cauchy-Schwartz inequality,
\begin{equation*}
    \begin{aligned}
    \EE[(\sum_w \pi(X_t, w)(\mu_\infty(X_t, w)-\mu(X_t, w)))^2 ]
    &=
    \EE[(\sum_w \frac{\sqrt{e_t(X_t, w)} \pi(X_t, w)}{\sqrt{e_t(X_t, w)}}(\mu_\infty(X_t, w)-\mu(X_t, w)))^2  ] \\
    &\leq
    \EE\left[\sum_w \frac{\pi^2(X_t, w)}{e_t(X_t, w)}(\mu_\infty(X_t, w) - \mu(X_t, w))^2 \left( \sum_w e_t(X_t, w) \right) \right] \\
    &= \EE\left[\sum_w \frac{\pi^2(X_t, w)}{e_t(X_t, w)}(\mu_\infty(X_t, w) - \mu(X_t, w))^2 \right].
    \end{aligned}
\end{equation*}

This implies the last three terms sum to a non-negative number, and hence we have the following lower bound:
\begin{equation}
      \bE\Big[\sum_w \frac{C_1(X_t,w)\pi^2(X_t, w)}{e_t(X_t, w)}\Big] + C_2 
    \geq \bE\Big[\sum_w \frac{[\Var{}{Y_t(w)|X_t}\pi^2(X_t, w)}{e_t(X_t, w)} \Big] \gtrsim \inf_{x,w} \Var{}{Y_t(w)|x}.
\end{equation}

The result above implies that that the sequence $M_T=\Omega(T)$ since
\begin{equation}
\begin{split}
    \label{eq:M_lower}
    M_T &=\sum_{t=1}^T \EE\Big[\sum_w \frac{C_1(X_t, w)\pi^2(X_t, w)}{e_t(X_t, w)} + C_2\Big]\Big/\bE\Big[\sum_w \frac{\pi^2(X_t, w)}{e_t(X_t, w)}\Big] \\
    &\geq  \sum_{t=1}^T \EE\Big[\sum_w \frac{\var(Y_t(w)|X_t)\pi^2(X_t, w)}{e_t(X_t, w)}\Big]\Big/\bE\Big[\sum_w \frac{\pi^2(X_t, w)}{e_t(X_t, w)}\Big]\\
    &\geq \inf_{x,w}\var(Y_t(w)|x) T.
\end{split}
\end{equation}

The upper bound on $m_t$ referred to in \eqref{eq:mt_upper} can be derived as follows,
\begin{equation}
    \label{eq:mt_upper_deriv}
    \begin{split}
         | m_t(X_t) | =& \bigg|
         \sum_{w}\frac{\pi^2(X_t, w)}{e_t(X_t, w)} (\hat{\mu}_t(X_t, w)-\mu_\infty(X_t, w))^2 -(\sum_w \pi(X_t, w)(\hat{\mu}_t(X_t, w)-\mu_\infty(X_t, w)))^2\\
    &+ 2\sum_{w}\frac{\pi^2(X_t, w)}{e_t(X_t, w)}(\hat{\mu}_t(X_t, w)-\mu_\infty(X_t, w))(\mu_\infty(X_t, w)-\mu(X_t, w)) 
    \\&-2(\sum_w \pi(X_t, w)(\hat{\mu}_t(X_t, w)-\mu_\infty(X_t, w)))(\sum_w \pi(X_t, w)(\mu_\infty(X_t, w)-\mu(X_t, w)))
         \bigg|\\
    \leq & \bigg| \sum_{w}\frac{\pi^2(X_t, w)}{e_t(X_t, w)} (\hat{\mu}_t(X_t, w)-\mu_\infty(X_t, w))^2\bigg| + \bigg|(\sum_w \pi(X_t, w)(\hat{\mu}_t(X_t, w)-\mu_\infty(X_t, w)))^2\bigg|\\
    &+2\bigg|\sum_{w}\frac{\pi^2(X_t, w)}{e_t(X_t, w)}(\hat{\mu}_t(X_t, w)-\mu_\infty(X_t, w))(\mu_\infty(X_t, w)-\mu(X_t, w)) \bigg|
    \\&+ 2\bigg|(\sum_w \pi(X_t, w)(\hat{\mu}_t(X_t, w)-\mu_\infty(X_t, w)))(\sum_w \pi(X_t, w)(\mu_\infty(X_t, w)-\mu(X_t, w)))
         \bigg|\\
    \leq & 2\bigg| \sum_{w}\frac{\pi^2(X_t, w)}{e_t(X_t, w)} (\hat{\mu}_t(X_t, w)-\mu_\infty(X_t, w))^2\bigg|\\
    &+2\sup_{x,w}|\mu_\infty(x,w)-\mu(x,w)|\cdot \sup_{x, w}|\hat{\mu}_t(x, w)-\mu_\infty(x, w)| \sum_{w}\frac{\pi^2(X_t, w)}{e_t(X_t, w)}
    \\&+ 2\sup_{x,w}|\mu_\infty(x,w)-\mu(x,w)|\cdot \sup_{x, w}|\hat{\mu}_t(x, w)-\mu_\infty(x, w)|\\
   \leq &\sup_{x, w}|\hat{\mu}_t(x, w)-\mu_\infty(x, w)| \cdot \sum_{w}\frac{\pi^2(X_t, w)}{e_t(X_t, w)} \cdot  \Big\{2+ 4 \sup_{x,w}|\mu_\infty(x,w)-\mu(x,w)|\Big\}  \\
  \lesssim & \sup_{x, w}|\hat{\mu}_t(x, w)-\mu_\infty(x, w)|\sum_{w}\frac{\pi^2(X_t, w)}{e_t(X_t, w)}.
    \end{split}
\end{equation}

\clearpage

\begin{remark}
By Proposition~\ref{prop:aipw_var}, $\sum_{t=1}^T \condE{}{\xi_{T,t}^2}{H_{t-1}}\leq \frac{U}{L}$, we thus have $\sum_{t=1}^T \condE{}{\xi_{T,t}^2}{H_{t-1}}\xrightarrow{L^p}1$ for any $p\geq 1$.
\end{remark}

\subsubsection{Moment decay.}
\label{sec:moment-decay-noncontextual}
Now consider the conditional fourth moment of $\xi_{T,t}$, we have
\begin{align}
         &\bE\big[\big(\widehat{\Gamma}_t(X_t, \pi)  - Q(\pi)\big)^4| H_{t-1}, X_t\big]\nonumber\\
        = & \bE\big[\big(\widehat{\Gamma}_t(X_t, \pi) -\sum_w \pi(X_t, w)\mu(X_t, w)+  \sum_w \pi(X_t, w)\mu(X_t, w)- Q(\pi)\Big)^4|H_{t-1}, X_t\Big]\nonumber\\
        = & \bE\big[\big(\widehat{\Gamma}_t(X_t, \pi) -  \sum_w \pi(X_t, w)\mu(X_t, w)\big)^4|H_{t-1}, X_t\big]\label{eq:xi_fourth_i} \\ 
       & + 4\big(\sum_w \pi(X_t, w)\mu(X_t, w)- Q(\pi)\big)\bE\big[\big(\widehat{\Gamma}_t(X_t, \pi) -  \sum_w \pi(X_t, w)\mu(X_t, w)\big)^3|H_{t-1}, X_t\big]\label{eq:xi_fourth_ii}\\
        &+ 6\big(\sum_w \pi(X_t, w)\mu(X_t, w)- Q(\pi)\big)^2\bE\big[\big(\widehat{\Gamma}_t(X_t, \pi) -  \sum_w \pi(X_t, w)\mu(X_t, w)\big)^2|H_{t-1}, X_t\big]\label{eq:xi_fourth_iii}\\
        &+ 4\big(\sum_w \pi(X_t, w)\mu(X_t, w)- Q(\pi)\big)^3\bE\big[\big(\widehat{\Gamma}_t(X_t, \pi) -  \sum_w \pi(X_t, w)\mu(X_t, w)\big)|H_{t-1}, X_t\big]\label{eq:xi_fourth_iv}\\
        & + \big(\sum_w \pi(X_t, w)\mu(X_t, w)- Q(\pi)\big)^4\label{eq:xi_fourth_v}
\end{align}
We analyze each term respectively. 
\begin{enumerate}[(i)]
    \item For term \eqref{eq:xi_fourth_i}, we have
   \[
     \begin{split}
  \eqref{eq:xi_fourth_i}   =~ & 
    \sum_w e_t(X_t, w)\Big(\frac{\pi(X_t, w)}{e_t(X_t, w)}(Y_t(w)-\hat{\mu}_t(X_t, w)) +  \sum_{w'}\pi(X_t, w')(\hat{\mu}_t(X_t, w')-\mu(X_t, w')) \Big)^4
  \\
     \leq~ & c_1\Big(1 + \sum_w \frac{\pi^2(X_t, w)}{e_t(X_t, w)} 
     + \sum_w \frac{\pi^3(X_t, w)}{e_t^2(X_t, w)}
     + \sum_w \frac{\pi^4(X_t, w)}{e_t^3(X_t, w)}\Big),
     \end{split}  
    \]
    for some constant $c_1$.
    \item Similarly, for term \eqref{eq:xi_fourth_ii}, there exists a constant $c_2$ such that
    \[
          \eqref{eq:xi_fourth_ii}
     \leq~  c_2\Big(1 +1 + \sum_w \frac{\pi^2(X_t, w)}{e_t(X_t, w)} 
     + \sum_w \frac{\pi^3(X_t, w)}{e_t^2(X_t, w)} \Big).
     \]
    \item For term \eqref{eq:xi_fourth_iii}, by Proposition~\ref{prop:aipw_var}, there exists $U>0$ such that
    \[
    \bE\big[\big(\widehat{\Gamma}_t(X_t, \pi) -  \sum_w \pi(X_t, w)\mu(X_t, w)\big)^2|H_{t-1}, X_t\big] = \mbox{Var}(\widehat{\Gamma}_t|H_{t-1}, X_t) \leq U \sum_w \frac{\pi^2(X_t, w)}{e_t(X_t, w)} .
    \]
    \item By Proposition~\ref{prop:aipw_unbiased}, term $\eqref{eq:xi_fourth_iv}=0$.
    \item Term \eqref{eq:xi_fourth_v} is bounded by a constant by Assumption~\ref{assu:clt_condition}.
\end{enumerate}
Thus we have
\begin{align*}
        \sum_{t=1}^T \condE{}{\xi_{T,t}^4}{H_{t-1}} &= \frac{\sum_{t=1}^T h_t^4 \bE\big[\big(\widehat{\Gamma}_t(X_t, \pi)  - Q( \pi)\big)^4| H_{t-1}\big]}{\bE\big[h_t^2  \big(\widehat{\Gamma}_t(X_t,  \pi)  - Q( \pi)\big)^2\big]^2}\\
        &\lesssim  \frac{\sum_{t=1}^T h_t^4  \bE\big[ 1 + \sum_w \frac{\pi^2(X_t, w)}{e_t(X_t, w)} 
     + \sum_w \frac{\pi^3(X_t, w)}{e_t^2(X_t, w)}
     + \sum_w \frac{\pi^4(X_t, w)}{e_t^3(X_t, w)} \big] }{
        \big(\sum_{t=1}^T \bE\big[h_t^2\sum_w \frac{\pi^2(X_t, w)}{e_t(X_t, w)}\big]\big)^2
        }\\
        &\stackrel{\mbox{(i)}}{\leq} \frac{\sum_{t=1}^T h_t^4  \bE\big[ \sum_w \frac{\pi^4(X_t, w)}{e_t^3(X_t, w)} |H_{t-1}\big] }{
        \big(\sum_{t=1}^T  \bE\big[ h_t^2\sum_w \frac{\pi^2(X_t, w)}{e_t(X_t, w)}\big]\big)^2}\\
        &\stackrel{\mbox{(ii)}}{=} \frac{\sum_{t=1}^T 
       \bE\big[  \sum_w \frac{\pi^4(X_t, w)}{e_t^3(X_t, w)} |H_{t-1}\big]
       \big/
       \bE\big[\sum_w \frac{\pi^2(X_t, w)}{e_t(X_t, w)}|H_{t-1}\big]^2
        }{T^2}\\
        &\stackrel{\mbox{(iii)}}{
        \lesssim}  \frac{\sum_{t=1}^T 
       \bE\big[  \sum_w \frac{\pi^2(X_t, w)}{e_t(X_t, w)} t^{2\alpha} |H_{t-1}\big]
       \big/
       \bE\big[\sum_w \frac{\pi^2(X_t, w)}{e_t(X_t, w)}|H_{t-1}\big]
        }{T^2}=\frac{\sum_{t=1}^T t^{2\alpha}  }{T^2}=O(T^{2\alpha-1})\xrightarrow 0 , 
\end{align*}
where in (i), we use Lemma \ref{lemma: pi_e}
; in (ii), we use the definition of non-contextual \StableVar weights that $h_t=\big(\bE\big[ \sum_w \frac{\pi^2(X_t, w)}{e_t(X_t, w)}|H_{t-1} \big]\big)^{-1/2}$; in (iii), we use the condition that $1\geq e_t\geq C t^{-\alpha}$.

\begin{lemma}
\label{lemma: pi_e}
For any $x,w, t$, we have 
\begin{equation}
    1\leq \sum_w \frac{\pi^2(X_t, w)}{e_t(X_t, w)}\leq \sum_w \frac{\pi^3(X_t, w)}{e_t^2(X_t, w)}\leq \sum_w \frac{\pi^4(X_t, w)}{e_t^3(X_t, w)}.
\end{equation}
\end{lemma}
\begin{proof}
    We have 
    \begin{align*}
        \sum_w \frac{\pi^2(X_t, w)}{e_t(X_t, w)} = \sum_w \frac{\pi^2(X_t, w)}{e_t(X_t, w)} \cdot \sum_w e_t(X_t, w) \geq 1
    \end{align*}
    by Cauchy-Schwartz inequality. Similarly,
      \begin{align*}
        \sum_w \frac{\pi^3(X_t, w)}{e_t^2(X_t, w)} = \sum_w \frac{\pi^3(X_t, w)}{e_t^2(X_t, w)} \cdot \sum_w \pi_t(X_t, w) \geq  \sum_w \Big(\frac{\pi^2(X_t, w)}{e_t(X_t, w)}\Big)^2\geq\frac{\pi^2(X_t, w)}{e_t(X_t, w)} .
    \end{align*}
    Also,
    \begin{align*}
         \sum_w \frac{\pi^4(X_t, w)}{e_t^3(X_t, w)} = \sum_w \frac{\pi^4(X_t, w)}{e_t^3(X_t, w)}\cdot \sum_w e_t(X_t, w) \geq  \Big(  \sum_w \frac{\pi^2(X_t, w)}{e_t(X_t, w)}  \Big)^2\geq \sum_w \frac{\pi^2(X_t, w)}{e_t(X_t, w)}.
    \end{align*}
    Lastly, 
    \begin{equation}
       \Big(\sum_w \frac{\pi^4(X_t, w)}{e_t^3(X_t, w)}\Big)^2 \geq \Big(\sum_w \frac{\pi^4(X_t, w)}{e_t^3(X_t, w)}\Big)\Big(\sum_w \frac{\pi^2(X_t, w)}{e_t(X_t, w)}\Big)\geq  \Big(\sum_w \frac{\pi^3(X_t, w)}{e_t^2(X_t, w)}\Big)^2.
    \end{equation}
\end{proof}

Collectively, $\xi_{T,t}$ satisfies the two  conditions for martingale central limit theorem in Proposition~\ref{prop:martingale_clt}, and thus we have $\sum_{t=1}^T \xi_{T,t}\xrightarrow{d}\N(0,1)$.

\subsection{CLT of \texorpdfstring{$\widehat{Q}^{NC}_T$}{Qhat-NC}}
\label{appendix:clt_nc_final}
We now establish Theorem~\ref{thm:non_contextual}. We have
\begin{equation}
\label{eq:t-stat-nc}
    \begin{split}
        \frac{\widehat{Q}^{NC}_T(\pi)  - Q(\pi)}{\sqrt{\widehat{V}_T^{NC}(\pi) }} = \frac{
        \sum_{t=1}^T h_t(\widehat{\Gamma}_t(X_t, \pi)  - Q(\pi))
        }{\sqrt{ \sum_{t=1}^T h_t^2(\widehat{\Gamma}_t(X_t, \pi)  - \widehat{Q}_T^{NC}(\pi)  )^2}} = \sum_{t=1}^T \xi_{T,t} \cdot \frac{\sqrt{\bE[  \sum_{t=1}^T h_t^2(\widehat{\Gamma}_t(X_t, \pi)  - Q(\pi) )^2] } }{\sqrt{   \sum_{t=1}^T h^2_t(\widehat{\Gamma}_t(X_t, \pi)  -  \widehat{Q}_T^{NC}(\pi) )^2} }.
    \end{split}
\end{equation}
We have shown that $\sum_{t=1}^T \xi_{T,t}$ is asymptotically standard normal, so by Slutsky's theorem it suffices to show that the rightmost factor in \eqref{eq:t-stat-nc} converges to $1$ in probability. We have
\begin{equation*}
    \begin{split}
        \frac{ \sum_{t=1}^T h^2_t(\widehat{\Gamma}_t(X_t, \pi)  -  \widehat{Q}_T^{NC}(\pi) )^2  }{
        \bE[ \sum_{t=1}^T h^2_t(\widehat{\Gamma}_t(X_t, \pi)  - Q(\pi))^2]
        } = \underbracket[0.4pt]{\frac{ \sum_{t=1}^T h^2_t(\widehat{\Gamma}_t(X_t, \pi)  -  Q(\pi))^2  }{
        \bE[ \sum_{t=1}^T h^2_t(\widehat{\Gamma}_t(X_t, \pi)  - Q(\pi))^2]
        }}_{\mbox{(a)}} 
        + \underbracket[0.4pt]{ \frac{ \sum_{t=1}^T h^2_t(Q(\pi)-  \widehat{Q}_T^{NC}(\pi) )^2  }{
        \bE[ \sum_{t=1}^T h^2_t(\widehat{\Gamma}_t(X_t, \pi)  - Q(\pi))^2]
        } }_{\mbox{(b)}}
        \\+ \underbracket[0.4pt]{2~\frac{ \sum_{t=1}^T h^2_t(\widehat{\Gamma}_t(X_t, \pi)  - Q(\pi))(Q(\pi)- \widehat{Q}_T^{NC}(\pi) )  }{
        \bE[ \sum_{t=1}^T h^2_t(\widehat{\Gamma}_t(X_t, \pi)  - Q(\pi))^2]
        }
        }_{\mbox{(c)}}.
    \end{split}
\end{equation*}
We shall show that term(a)$\xrightarrow{p}1$, term(b)$\xrightarrow{p}0$, and term(c)$\xrightarrow{p}0$, which concludes the proof. 
\begin{itemize}
    \item \emph{Term (a).}
We have 
\begin{align*}
    \frac{ \sum_{t=1}^T h^2_t(\widehat{\Gamma}_t(X_t, \pi)  -  Q(\pi))^2  }{
        \bE[ \sum_{t=1}^T h^2_t(\widehat{\Gamma}_t(X_t, \pi)  - Q(\pi))^2]
        } = \frac{ \sum_{t=1}^T h^2_t\big((\widehat{\Gamma}_t(X_t, \pi)  -  Q(\pi))^2 - \bE[h^2_t(\widehat{\Gamma}_t(X_t, \pi)  -  Q(\pi))^2|H_{t-1}] \big) } {
        \bE[ \sum_{t=1}^T h^2_t(\widehat{\Gamma}_t(X_t, \pi)  - Q(\pi))^2]
        } 
        \\+ \frac{ \sum_{t=1}^T h^2_t\bE[h^2_t(\widehat{\Gamma}_t(X_t, \pi)  -  Q(\pi))^2|H_{t-1} ]}{
        \bE[ \sum_{t=1}^T h^2_t(\widehat{\Gamma}_t(X_t, \pi)  - Q(\pi))^2]
        }
\end{align*}
The first term converges to $0$ (to be shown shortly), and the second term is $\sum_{t=1}^T \bE[\xi_{T,t}^2|H_{t-1}]\xrightarrow{p}1$. The first term is the sum of a martingale difference sequence, and  we have
\begin{align*}
 & \bE\Big[ \Big| \frac{ \sum_{t=1}^T h^2_t\big((\widehat{\Gamma}_t(X_t, \pi)  -  Q(\pi))^2 - \bE[h^2_t(\widehat{\Gamma}_t(X_t, \pi)  -  Q(\pi))^2|H_{t-1}] \big) } {
        \bE[ \sum_{t=1}^T h^2_t(\widehat{\Gamma}_t(X_t, \pi)  - Q(\pi))^2]
        } \Big|^2\Big]\\
         =~ &\bE\Big[\sum_{t=1}^T \bE\Big[ 
         \frac{ h^4_t\big((\widehat{\Gamma}_t(X_t, \pi)  -  Q(\pi))^4 - \bE[h^2_t(\widehat{\Gamma}_t(X_t, \pi)  -  Q(\pi))^2|H_{t-1}] \big)^2 }{
        \bE[ \sum_{t=1}^T h^2_t(\widehat{\Gamma}_t(X_t, \pi)  - Q(\pi))^2]^2} 
        \Big|H_{t-1}\Big]\Big]\\
        \leq~ &\bE\Big[\sum_{t=1}^T \bE\Big[ 
         \frac{ h^4_t(\widehat{\Gamma}_t(X_t, \pi)  -  Q(\pi))^4  }{
        \bE[ \sum_{t=1}^T h^2_t(\widehat{\Gamma}_t(X_t, \pi)  - Q(\pi))^2]^2} 
        \Big|H_{t-1}\Big]\Big] = \bE\Big[\sum_{t=1}^T \bE\big[\xi_{T,t}^4 |H_{t-1}\big]\Big] = O(T^{2\alpha-1})\rightarrow 0.\\
\end{align*}
In the last step, we've used a fourth moment bound from Section~\ref{sec:moment-decay-noncontextual}.

    \item \emph{Term (b).} This term vanishes because $\widehat{Q}_T^{NC}\xrightarrow{p}Q$. The ratio
\begin{align*}
    \frac{ (Q(\pi)-  \widehat{Q}_T^{NC}(\pi) )^2 \sum_{t=1}^T h^2_t  }{
        \bE[ \sum_{t=1}^T h^2_t(\widehat{\Gamma}_t(X_t, \pi)  - Q(\pi))^2]
        }
\end{align*}
goes to zero because (i) the squared \StableVar weights  $h_t^2 = 1/\bE[{\sum_{w}\frac{\pi^2(X_t, w)}{e_t(X_t, w)}}|H_{t-1}]$ are each less than one and their sum, therefore, is less than $T$
and (ii) each term in the denominator is bounded away from zero, as \eqref{eq:gamma_condvar} implies that
the inverse of $h_t^2$ is on the order of $\bE[(\widehat{\Gamma}_t(X_t, \pi)  - Q(\pi))^2|H_{t-1}]$, so the denominator itself is $\Omega(T)$.

    \item \emph{Term (c).} This term also vanishes by  Cauchy-Schwartz inequality. We have 
\begin{align*}
    & \Big(\frac{ \sum_{t=1}^T h^2_t(\widehat{\Gamma}_t(X_t, \pi)  - Q(\pi))(Q(\pi)- \widehat{Q}_T^{NC}(\pi) )  }{
        \bE[ \sum_{t=1}^T h^2_t(\widehat{\Gamma}_t(X_t, \pi)  - Q(\pi))^2]
        }\Big)^2 \\\leq &\frac{ \sum_{t=1}^T h^2_t(\widehat{\Gamma}_t(X_t, \pi)  - Q(\pi))^2  }{
        \bE[ \sum_{t=1}^T h^2_t(\widehat{\Gamma}_t(X_t, \pi)  - Q(\pi))^2]
        }
        \times \frac{ \sum_{t=1}^T h^2_t(Q(\pi)- \widehat{Q}_T^{NC}(\pi) )^2  }{
        \bE[ \sum_{t=1}^T h^2_t(\widehat{\Gamma}_t(X_t, \pi)  - Q(\pi))^2]
        }\xrightarrow{p}0.
\end{align*}
\end{itemize}

\section{Limit theorems of contextual weighting}
\label{sec:proof_contextual_clt}

In this section, we establish Theorem~\ref{thm:contextual}. For a target policy $\pi$, the  policy value is $Q(\pi)=\bE[\sum_w\pi(x,w)\mu(x,w)]$.   Our goal is to show that, with contextual \StableVar weights $h_t(x)=1/\sqrt{\sum_w\frac{\pi^2(x,w)}{e_t(x,w)}}$, the estimate $\widehat{Q}_T^C$ is consistent and asymptotically normal. In what follows, we shall use $P_x=\p(X=x)$ and $Q_x=Q(x,\pi)=\sum_w\pi(x,w)\mu(x,w)$ for notational convenience.

\vspace{\baselineskip}

The following lemma characterizes the convergence of  \StableVar weights for each context and each pair of contexts, which be will invoked multiple times throughout the proof. 
\begin{lemma}
\label{lemma:weight_convergence}
Suppose that for any $x, w$, almost surely and in $L_1$,
\begin{equation}
    \begin{aligned}
    \frac{\bE[e_t^{-1}(x,w)]}{e_t^{-1}(x,w)}\to 1 \quad \mbox{ and } \quad \frac{\bE[e_t^{-1}(x,w) e_t^{-1}(x',w')]}{e_t^{-1}(x,w) e_t^{-1}(x',w')}\to 1
    \end{aligned}
\end{equation} Given a policy $\pi$, for any context $x$, its \StableVar weights $h_t(x) =1/ \sqrt{\sum_{w}\frac{\pi^2(x,w)}{e_t(x,w)}} $ satisfies
\begin{equation}
\label{eq:ht_convergence}
    \frac{\sum_{t=1}^T h_t(x)}{\bE\big[\sum_{t=1}^T h_t(x)\big]}\xrightarrow{p}1 
\end{equation}
and
\begin{equation}
\label{eq:ht_cross_convergence}
    \frac{\sum_{t=1}^T h_t(x)h_t(x') }{\bE\big[\sum_{t=1}^Th_t(x)h_t(x')\big]}\xrightarrow{p}1.
\end{equation}

\end{lemma}

\begin{proof}
   Fix a context $x$, we want to show that \eqref{eq:ht_convergence}. This holds if the numerator can be written as 
   \begin{equation}
       \label{eq:ht_num} 
       \sum_{t=1}^T h_t(x) = o_p\Big(\sum_{t=1}^T \bE\Big[\sum_{w}\frac{\pi^2(x,w)}{e_t(x,w)}\Big]^{-1/2}\Big) + \sum_{t=1}^T \bE\Big[\sum_{w}\frac{\pi^2(x,w)}{e_t(x,w)}\Big]^{-1/2},
   \end{equation}
   and the denominator can be written as 
   \begin{equation}
       \label{eq:ht_denom} 
       \sum_{t=1}^T \bE[h_t(x)] = o\Big(\sum_{t=1}^T \bE\Big[\sum_{w}\frac{\pi^2(x,w)}{e_t(x,w)}\Big]^{-1/2}\Big) + \sum_{t=1}^T \bE\Big[\sum_{w}\frac{\pi^2(x,w)}{e_t(x,w)}\Big]^{-1/2}.
   \end{equation}
We will show this is the case. To show  \eqref{eq:ht_num}, we have 
 \begin{equation}
     \begin{split}
       \sum_{t=1}^T h_t(x) - \sum_{t=1}^T \bE\Big[\sum_{w}\frac{\pi^2(x,w)}{e_t(x,w)}\Big]^{-1/2} & =     \sum_{t=1}^T \Big\{\Big(\sum_{w}\frac{\pi^2(x,w)}{e_t(x,w)}\Big)^{-1/2}
              - \bE\Big[\sum_{w}\frac{\pi^2(x,w)}{e_t(x,w)}\Big]^{-1/2}\Big\} \\
              & =    \sum_{t=1}^T \bigg( \sqrt{\frac{\bE [ \sum_{w}\frac{\pi^2(x,w)}{e_t(x,w)}]}{\sum_{w}\frac{\pi^2(x,w)}{e_t(x,w)}}}
              -1 \bigg)\bE\Big[\sum_{w}\frac{\pi^2(x,w)}{e_t(x,w)}\Big]^{-1/2}.
     \end{split}
 \end{equation}
 This parenthesized expression converges to zero. This follows, via the continuous mapping theorem, from the almost-sure convergence of $\bE[e_t^{-1}(x,w)]/e_t^{-1}(x,w)$ to one:
              \begin{equation}
              \label{eq:ht_as_converge}
              \begin{split}
                  \bigg|\frac{\bE [ \sum_{w}\frac{\pi^2(x,w)}{e_t(x,w)}]}{\sum_{w}\frac{\pi^2(x,w)}{e_t(x,w)}} - 1   \bigg|\leq \bigg|\sum_{w} \frac{\bE [ \frac{\pi^2(x,w)}{e_t(x,w)}] - \frac{\pi^2(x,w)}{e_t(x,w)}}{\frac{\pi^2(x,w)}{e_t(x,w)}}   \bigg| \leq \sum_w  \bigg|\frac{\bE [ \frac{\pi^2(x,w)}{e_t(x,w)}] - \frac{\pi^2(x,w)}{e_t(x,w)}}{\frac{\pi^2(x,w)}{e_t(x,w)}}   \bigg| = \sum_w  \bigg| 
                   \frac{\bE[e_t^{-1}(x,w)]}{e_t^{-1}(x,w)} - 1
                   \bigg| \xrightarrow{a.s.}0.
              \end{split}
              \end{equation}

 We have 
    \begin{equation}
    \label{eq:ht_o1}
    \begin{split}
         &\frac{  \sum_{t=1}^T h_t(x) - \sum_{t=1}^T \bE[\sum_{w}\frac{\pi^2(x,w)}{e_t(x,w)}]^{-1/2}}{\sum_{t=1}^T \bE[\sum_{w}\frac{\pi^2(x,w)}{e_t(x,w)}]^{-1/2}} \\
        = &  \frac{   \sum_{t=1}^T ( \sqrt{\frac{\bE [ \sum_{w}\frac{\pi^2(x,w)}{e_t(x,w)}]}{\sum_{w}\frac{\pi^2(x,w)}{e_t(x,w)}}}
              -1 )\bE\Big[\sum_{w}\frac{\pi^2(x,w)}{e_t(x,w)}\Big]^{-1/2} }{\sum_{t=1}^T \bE[\sum_{w}\frac{\pi^2(x,w)}{e_t(x,w)}]^{-1/2}},
    \end{split}
    \end{equation}
    where in the numerator,  $\big|\frac{\bE [ \sum_{w}\frac{\pi^2(x,w)}{e_t(x,w)}]}{\sum_{w}\frac{\pi^2(x,w)}{e_t(x,w)}} - 1   \big|$ goes to zero almost surely;
    and in the denominator,  $ \sum_{t=1}^T \bE[\sum_{w}\frac{\pi^2(x,w)}{e_t(x,w)}]^{-1/2} \geq  \sum_{t=1}^T \bE[\sum_{w}\pi^2(x,w)]^{-1/2} t^{-\frac{\alpha}{2}}$ goes to infinity and thus\\ $\max_{1\leq t\leq T}\bE[\sum_{w}\frac{\pi^2(x,w)}{e_t(x,w)}]^{-1/2} /\sum_{t=1}^T\bE[\sum_{w}\frac{\pi^2(x,w)}{e_t(x,w)}]^{-1/2}\rightarrow 0$. By Proposition \ref{prop:weighted_convergence}, we have  \eqref{eq:ht_o1} goes to zero  in probability, proving \eqref{eq:ht_num}.

\vspace{\baselineskip}   
 To show  \eqref{eq:ht_denom}, condition  $ \frac{\bE[e_t^{-1}(x,w)]}{e_t^{-1}(x,w)}\xrightarrow{L_1} 1$ yields $\sqrt{\frac{\bE [ \sum_{w}\frac{\pi^2(x,w)}{e_t(x,w)}]}{\sum_{w}\frac{\pi^2(x,w)}{e_t(x,w)}}}-1
              \xrightarrow{L_1}0$, which follows from the below argument,
    \begin{equation}
        \begin{split}
            \bE\Big[\Big|
            \sqrt{\frac{\bE [ \sum_{w}\frac{\pi^2(x,w)}{e_t(x,w)}]}{\sum_{w}\frac{\pi^2(x,w)}{e_t(x,w)}}}-1
            \Big|\Big] & =   \bE\Big[\Big|
          \Big(\frac{\bE [ \sum_{w}\frac{\pi^2(x,w)}{e_t(x,w)}]}{\sum_{w}\frac{\pi^2(x,w)}{e_t(x,w)}}-1\Big)\Big/  \sqrt{\frac{\bE [ \sum_{w}\frac{\pi^2(x,w)}{e_t(x,w)}]}{\sum_{w}\frac{\pi^2(x,w)}{e_t(x,w)}}}+1
            \Big|\Big]\\
            &\leq \bE\Big[\Big|
          \frac{\bE [ \sum_{w}\frac{\pi^2(x,w)}{e_t(x,w)}]}{\sum_{w}\frac{\pi^2(x,w)}{e_t(x,w)}}-1
            \Big|\Big] \leq  \bE\Big[\Big|
            \sum_w \frac{\bE[\frac{\pi^2(x,w)}{e_t(x,w)}]}{\frac{\pi^2(x,w)}{e_t(x,w)}}-1
                 \Big|\Big] \\
        &\leq \sum_w  \bE\Big[\Big|
        \frac{\bE[\frac{\pi^2(x,w)}{e_t(x,w)}]}{\frac{\pi^2(x,w)}{e_t(x,w)}}-1
         \Big|\Big] =  \sum_w  \bE\Big[\Big|
        \frac{\bE[e^{-1}_t(x,w)]}{e_t^{-1}(x,w)}-1
         \Big|\Big]\rightarrow 0
        \end{split}
    \end{equation}

     Therefore,
    \begin{equation}
    \label{eq:ht_denom_decompose}
    \begin{aligned}
               \frac{ \bigg|\bE\Big[ \sum_{t=1}^T h_t(x) - \sum_{t=1}^T \bE\Big[\sum_{w}\frac{\pi^2(x,w)}{e_t(x,w)}\Big]^{-1/2} \Big]\bigg|}{\sum_{t=1}^T \bE\Big[\sum_{w}\frac{\pi^2(x,w)}{e_t(x,w)}\Big]^{-1/2} }&\leq \frac{ \bE\Big[\Big| \sum_{t=1}^T h_t(x) - \sum_{t=1}^T \bE\Big[\sum_{w}\frac{\pi^2(x,w)}{e_t(x,w)}\Big]^{-1/2} \Big|\Big]}{\sum_{t=1}^T \bE\Big[\sum_{w}\frac{\pi^2(x,w)}{e_t(x,w)}\Big]^{-1/2} }\\
                & \leq  \frac{\sum_{t=1}^T \bE\Big[\bigg|\sqrt{\frac{\bE [ \sum_{w}\frac{\pi^2(x,w)}{e_t(x,w)}]}{\sum_{w}\frac{\pi^2(x,w)}{e_t(x,w)}}}
              -1 \bigg|\Big]\bE[\sum_{w}\frac{\pi^2(x,w)}{e_t(x,w)}]^{-1/2}}{\sum_{t=1}^T \bE\Big[\sum_{w}\frac{\pi^2(x,w)}{e_t(x,w)}\Big]^{-1/2} },
    \end{aligned}
    \end{equation}
     where in the numerator,  $\bE\big[\big|\frac{\bE [ \sum_{w}\frac{\pi^2(x,w)}{e_t(x,w)}]}{\sum_{w}\frac{\pi^2(x,w)}{e_t(x,w)}} - 1   \big|\big]$ goes to zero almost surely;
    and in the denominator,  $ \sum_{t=1}^T \bE[\sum_{w}\frac{\pi^2(x,w)}{e_t(x,w)}]^{-1/2} \geq  \sum_{t=1}^T \bE[\sum_{w}\pi^2(x,w)]^{-1/2} t^{-\frac{\alpha}{2}}$ goes to infinity and thus\\ $\max_{1\leq t\leq T}\bE[\sum_{w}\frac{\pi^2(x,w)}{e_t(x,w)}]^{-1/2} /\sum_{t=1}^T\bE[\sum_{w}\frac{\pi^2(x,w)}{e_t(x,w)}]^{-1/2}\rightarrow 0$. By Proposition \ref{prop:weighted_convergence_2}, we have  \eqref{eq:ht_denom_decompose} goes to zero. Together with \eqref{eq:ht_o1} converges to zero in probability, we  conclude the proof for \eqref{eq:ht_convergence}.
    
    \vspace{\baselineskip}

  Similarly, for fixed contexts $x, x'$, we want to show that \eqref{eq:ht_cross_convergence}, which can be showed by a similar argument as we did for showing \eqref{eq:ht_convergence}. We shall show that the numerator can be written as 
     \begin{equation}
       \label{eq:h2t_num} 
       \sum_{t=1}^T h_t(x)h_t(x') = o_p(\sum_{t=1}^T \bE[\sum_{w}\frac{\pi^2(x,w)}{e_t(x,w)}\sum_{w}\frac{\pi^2(x',w)}{e_t(x',w)}]^{-1/2}) + \sum_{t=1}^T \bE[\sum_{w}\frac{\pi^2(x,w)}{e_t(x,w)}\sum_{w}\frac{\pi^2(x',w)}{e_t(x',w)}]^{-1/2},
   \end{equation}
   and the denominator can be written as 
   \begin{equation}
       \label{eq:h2t_denom} 
       \sum_{t=1}^T \bE[h_t(x)h_t(x')] = o(\sum_{t=1}^T \bE[\sum_{w}\frac{\pi^2(x,w)}{e_t(x,w)} \sum_{w}\frac{\pi^2(x',w)}{e_t(x',w)} ]^{-1/2}) + \sum_{t=1}^T \bE[\sum_{w}\frac{\pi^2(x,w)}{e_t(x,w)} \sum_{w}\frac{\pi^2(x',w)}{e_t(x',w)}]^{-1/2}.
   \end{equation}
Once we have \eqref{eq:h2t_num} and \eqref{eq:h2t_denom}, we have  the term $\sum_{t=1}^T h_t(x)h_t(x') / \bE[\sum_{t=1}^T h_t(x)h_t(x')]\xrightarrow{p}1$. 

\vspace{\baselineskip}

To show \eqref{eq:h2t_num}, we first have  $ \sqrt{\frac{
        \bE[\sum_{w}\frac{\pi^2(x,w)}{e_t(x,w)}\cdot \sum_{w}\frac{\pi^2(x',w)}{e_t(x',w)}] 
        }{\sum_{w}\frac{\pi^2(x,w)}{e_t(x,w)}\cdot \sum_{w}\frac{\pi^2(x',w)}{e_t(x',w)}  } }- 1
              \xrightarrow{a.s.}0$ by a similar argument as \eqref{eq:ht_as_converge}. Therefore,

\begin{equation}
\label{eq:h2t_inter}
    \begin{split}
 & \sum_{t=1}^T h_t(x)h_t(x') -    \sum_{t=1}^T  \bE[\sum_{w}\frac{\pi^2(x,w)}{e_t(x,w)}\cdot \sum_{w}\frac{\pi^2(x',w)}{e_t(x',w)}]^{-1/2}\\   
 = & \sum_{t=1}^T \Big(  \sqrt{\frac{
        \bE[\sum_{w}\frac{\pi^2(x,w)}{e_t(x,w)}\cdot \sum_{w}\frac{\pi^2(x',w)}{e_t(x',w)}] 
        }{\sum_{w}\frac{\pi^2(x,w)}{e_t(x,w)}\cdot \sum_{w}\frac{\pi^2(x',w)}{e_t(x',w)}  } }- 1\Big)  \bE[\sum_{w}\frac{\pi^2(x,w)}{e_t(x,w)}\cdot \sum_{w}\frac{\pi^2(x',w)}{e_t(x',w)}]^{-1/2}\\
= & o_p\bigg(
\sum_{t=1}^T  \bE[\sum_{w}\frac{\pi^2(x,w)}{e_t(x,w)}\cdot \sum_{w}\frac{\pi^2(x',w)}{e_t(x',w)}]^{-1/2}
\bigg),
    \end{split}
\end{equation}
by a similar argument as we show \eqref{eq:ht_o1} is $o_p(1)$. Thus we have  \eqref{eq:h2t_num} holds.
              
\vspace{\baselineskip}

To show \eqref{eq:h2t_denom}, from condition  $ \frac{\bE[e_t^{-1}(x,w) e_t^{-1}(x',w')]}{e_t^{-1}(x,w) e_t^{-1}(x',w')}\xrightarrow{L_1} 1$, we have $ \sqrt{\frac{
        \bE[\sum_{w}\frac{\pi^2(x,w)}{e_t(x,w)}\cdot \sum_{w}\frac{\pi^2(x',w)}{e_t(x',w)}] 
        }{\sum_{w}\frac{\pi^2(x,w)}{e_t(x,w)}\cdot \sum_{w}\frac{\pi^2(x',w)}{e_t(x',w)}  } }- 1
              \xrightarrow{L_1}0$.  By \eqref{eq:h2t_inter}, we have 
    \begin{equation}
        \bE\Big[ \Big|\sum_{t=1}^T h_t(x)h_t(x') -    \sum_{t=1}^T  \bE[\sum_{w}\frac{\pi^2(x,w)}{e_t(x,w)}\cdot \sum_{w}\frac{\pi^2(x',w)}{e_t(x',w)}]^{-1/2}\Big| \Big] = o\Big(\sum_{t=1}^T  \bE[\sum_{w}\frac{\pi^2(x,w)}{e_t(x,w)}\cdot \sum_{w}\frac{\pi^2(x',w)}{e_t(x',w)}]^{-1/2}\Big),
    \end{equation}
    proving \eqref{eq:h2t_denom}. This concludes the proof for \eqref{eq:ht_cross_convergence}.

\end{proof}

\subsection{Consistency of \texorpdfstring{$\widehat{Q}_T^C$}{Qhat-C}}
\label{appendix:consistency_of_contextual_weighting}
Estimate $\widehat{Q}_T^C$ is in fact a summation of policy value estimates conditional on covariates, that is,
\begin{align*}
   \widehat{Q}_T^C (\pi) = \sum_x (\widehat{P_x  Q_x})_T, \quad\mbox{where}\quad (\widehat{P_x  Q_x})_T := \frac{\sum_{t=1}^T h_t(x)\one\{X_t=x\}\widehat{\Gamma}_t(x, \pi)  }{\sum_{t=1}^T h_t(x)}.
\end{align*}
We have
\begin{align*}
    \widehat{Q}_T^C(\pi)  - Q(\pi)=\sum_x  \sum_{t=1}^T  \frac
        { h_t(x) (\one\{ X_t = x\}\widehat{\Gamma}_t(x, \pi)  - P_x Q_x)}
        {\sum_t h_t(x) } = \sum_x (\widehat{P_x  Q_x})_T - P_xQ_x.
\end{align*}
It thus suffices to show that each $(\widehat{P_x  Q_x})_T$ is consistent for $ P_xQ_x$. To do so, we introduce a context-specific MDS $\eta_{T,t}(x)$ for each covariate $x$,
\begin{equation}
\label{eq:eta_x}
    \eta_{T,t}(x) = \frac{h_t(x)(\one\{X_t=x\}\widehat{\Gamma}_t(x, \pi)  - P_xQ_x)}{\sqrt{\bE[\sum_{s=1}^T  h_t^2(x)(\one\{X_s=x\}\widehat{\Gamma}_s(w)  - P_xQ_x)^2 ]}}.
\end{equation}
which martingale difference property is justified by observing  $\bE [ \eta_{T,t} | H_{t-1} ]=0$ by routine calculation. We immediately have $\bE [ (\sum_{t=1}^T\eta_{T,t})^2 ]= \bE [ \sum_{t=1}^T\eta_{T,t}^2 ]=1$.

\vspace{\baselineskip}

The conditional variance of the numerator in \eqref{eq:eta_x} can be calculated explicitly and dominated by $\sum_w \frac{\pi^2(x,w)}{e_t(x,w)}$, that is, 
\begin{equation}
\label{eq:eta_second_moment}
\begin{aligned}
     &\condE{}{(\one\{X_t=x\}\widehat{\Gamma}_t(x, \pi)  - P_xQ_x)^2}{H_{t-1}} \\
    = & P_x \condE{}{(\widehat{\Gamma}_t(x, w)  - P_xQ_x)^2}{H_{t-1}, X_t=x} + (1-P_x)P^2_xQ^2_x\\
    = & P_x  \condE{}{(\widehat{\Gamma}_t(x,w)  - Q_x)^2}{H_{t-1}, X_t=x} + P_x (1-P_x)^2 Q^2_x + (1-P_x)P_x^2Q^2_x\\
    = &  P_x  \condE{}{(\widehat{\Gamma}_t(x,w)  - Q_x)^2}{H_{t-1}, X_t=x}  + P_x(1-P_x)Q^2_x\\
    = & P_x \var(\widehat{\Gamma}_t(x,w)|H_{t-1}, X_t=x) 
    + P_x(1-P_x)Q^2_x,
\end{aligned}
\end{equation}
which is lower bounded by $c_1 \cdot \sum_w \frac{\pi^2(x,w)}{e_t(x,w)}$ and upper bounded by $c_2 \cdot \sum_w \frac{\pi^2(x,w)}{e_t(x,w)}$ for some constant $c_1, c_2$ by Proposition \ref{prop:aipw_var}. Moving on,  we have,
\begin{align*}
   |(\widehat{P_xQ_x})_T - P_xQ_x|^2 =  &\Big|\frac{\sum_{t=1}^T h_t(x)(\one\{X_t=x\}\widehat{\Gamma}_t(x, \pi) -  P_xQ_x) }{\sum_{t=1}^T h_t(x)}\Big|^2 \\
    = &\Big|\sum_{t=1}^T \eta_{T,t}(x) \Big|^2  \frac{\bE[\sum_{t=1}^T\big(h_t(x)\one\{X_t=x\}\widehat{\Gamma}_t(x, \pi)  - P_xQ_x\big)^2]] }{(\sum_{t=1}^T h_t(x))^2}  \\
    \stackrel{(i)}{\lesssim} &  \Big|\sum_{t=1}^T \eta_{T,t}(x) \Big|^2 \frac{T}{(\sum_{t=1}^T t^{-\alpha/2})^2 } =\Big| \sum_{t=1}^T \eta_{T,t}(x)\Big|^2 O(T^{\alpha-1}),
\end{align*}
where in (i), we use \eqref{eq:eta_second_moment}, $h_t(x)=1/\sqrt{\sum_w \frac{\pi^2(x,w)}{e_t(x,w)}}$, and that $e_t\geq Ct^{-\alpha}$ for some $\alpha\in[0,1/2)$ by Assumption~\ref{assu:clt_condition}. Thus for any $\epsilon>0$,
\begin{align*}
    \p(|(\widehat{P_xQ_x})_T - P_xQ_x|>\epsilon)\leq& \epsilon^{-2} \bE[|(\widehat{P_xQ_x})_T - P_xQ_x|^2]\\
    \leq &\epsilon^2 \bE\Big[\Big|\sum_{t=1}^T \eta_{T,t}(x)\Big|^2\Big]O(T^{\alpha-1})\\
    = & \epsilon^2 O(T^{\alpha-1})\rightarrow 0,
\end{align*}
which concludes the consistency of $(\widehat{P_xQ_x})_T$ and in turn $\widehat{Q}_T^C$.

\subsection{Asymptotic normality of \texorpdfstring{$\widehat{Q}_T^C$}{Qhat-C}}
 Expanding the definitions, we have 
\begin{equation}
    \label{eq:studentized}
    \frac{\widehat{Q}_T^C(\pi)  - Q(\pi)}{\sqrt{\widehat{V}^C_T(\pi) }} = \frac{1}{\sqrt{\widehat{V}^C_T(\pi) }} \sum_{t=1}^T \sum_x  \frac
        { h_t(x) (\one\{ X_t = x\}\widehat{\Gamma}_t(x, \pi)  - P_xQ_x)}
        {\sum_t h_t(x) },
\end{equation}
where the variance estimate given in \eqref{eq:studentized} can be rewritten as
\begin{equation}
    \label{eq:vhatc}
    \widehat{V}_T^C(\pi)  :=   \sum_{t=1}^T 
         \Big( \sum_{x}\frac{
      h_t(x)(\one\{X_t=x\}\widehat{\Gamma}_t(x, \pi)  - (\widehat{P_x  Q_x})_T)
         }{
         \sum_{t=1}^T h_t(x)
         }
         \Big)^2.
\end{equation}

One challenge in proving \eqref{eq:studentized} is asymptotically normal is that the weight normalization term $\sum_t h_t(x)$ in \eqref{eq:studentized} is context-specific, and hence we cannot leverage a single martingale difference sequence (MDS) to establish the central limit theorem as we did in Appendix~\ref{sec:proof_noncontextual_clt} for non-contextual weighting. However, we notice that \eqref{eq:studentized} is related to the sum of context-specific MDSs $\eta_{T,t}(x)$. Therefore, instead of working with \eqref{eq:studentized} directly, we will  show that the following ``auxiliary'' statistic (which is also a MDS) is asymptotically normal:
\begin{equation}
    \label{eq:sum_zeta}
    \begin{split}
    \sum_{t=1}^T \zeta_{T,t} &:= 
    \frac{1}{\sqrt{V_T^C}}
    \sum_{t=1}^T \sum_{x}\alpha_T(x)\eta_{T,t}(x),
    \end{split}
\end{equation}
where the aggregation weight is given by
\begin{equation}
\label{eq:alpha_T}
    \alpha_T(x) =  \frac{\sqrt{\sum_{t=1}^T \bE[h^2_t(x)(\one\{X_t=x\}\widehat{\Gamma}_t(x, \pi)  - P_xQ_x)^2]}}{ \sum_{t=1}^T\E{}{h_t(x)} } .
\end{equation}
and the variance $V_T^C$ is given by
\begin{equation}
    \label{eq:vc}
     \begin{split}
  & \quad V_T^C = \sum_{t=1}^T 
        \bE\Big[  \Big( \sum_{x}\frac{
      h_t(x)(\one\{X_t=x\}\widehat{\Gamma}_t(x, \pi)  - P_xQ_x)
         }{
         \sum_{t=1}^T\E{}{h_t(x)}
         }
         \Big)^2\Big].
    \end{split}
\end{equation}

Once this is established, we then show that we can replace the population quantities by their sample counterparts (e.g., $\widehat{V}_T^C$ by $V_T^C$) without perturbing the asymptotic behavior of the statistic. This is done in the following steps.
\begin{itemize}
    \item \emph{Step I: CLT of context-specific MDS $\eta_{T,t}(x)$.} \label{sec:convergence-xi} For a fixed covariate $x$, this essentially reduces to a non-contextual problem, so we can prove asymptotic normality of $\sum_{t=1}^T \eta_{T, t}(x)$ by  repeating a similar proof to Theorem~\ref{thm:non_contextual}.

    \item \emph{Step II: CLT of the aggregation MDS $\zeta_{T,t}$.} We  shall show that this $\zeta_{T,t}$ satisfies two martingale CLT conditions stated in Proposition~\ref{prop:martingale_clt} leveraging that each $\eta_{T,t}(x)$ satisfies the same conditions; hence $\sum_{t=1}^T \zeta_{T,t}$ is asymptotically standard normal.

\item \emph{Step III: CLT of \eqref{eq:studentized}.} The final step in the proof is to show that we can replace the population quantities $V_T^{C}$ and $\sum_{t=1}^{T} \bE[h_t(X_t)]$ by their sample counterparts.

\end{itemize}

\subsubsection{Step I:  CLT of context-specific MDS \texorpdfstring{$\eta_{T,t}(x)$}{MDS-C-x}}
Fixing a covariate $x$, the context-specific MDS $\eta_{T,t}(x)$ reduces to that constructed in multi-armed bandits, which  has been studied in \cite{hadad2021confidence}. The subtlety here is the indicator function $\one\{X_t=x\}$, but the entire argument from \cite{hadad2021confidence} directly applies. For completeness, we lay out the details and verify two CLT conditions of $\eta_{T,t}(x)$ required in Proposition~\ref{prop:martingale_clt}.

\paragraph{Variance convergence.}
We want to show that 
\begin{equation}
\label{eq:eta2}
\begin{aligned}
    \sum_{t=1}^T \bE[\eta^2_{T,t}(x)|H_{t-1}] = \frac{
   \sum_{t=1}^T h_t(x)^2 \bE[(\one\{X_t=x\}\widehat{\Gamma}_t(x, \pi)  - P_xQ_x)^2|H_{t-1}]
    }{
    \sum_{t=1}^T \bE[h_t(x)^2 (\one\{X_t=x\}\widehat{\Gamma}_t(x, \pi)  - P_xQ_x)^2]
    }\xrightarrow{p}1.
\end{aligned}    
\end{equation}

Each summand in the numerator has been computed analytically in Appendix \ref{sec:proof_aipw_var}, that is,
\begin{equation*}
\begin{aligned}
     &\condE{}{(\one\{X_t=x\}\widehat{\Gamma}_t(x, \pi)  - P_xQ_x)^2}{H_{t-1}} \\
    = & P_x \var(\widehat{\Gamma}_t(X_t,\pi)\mid H_{t-1}, X_t=x)
    + P_x(1-P_x)Q_x^2\\
    = & P_x \sum_w\frac{\pi^2(x, w)\Var{}{Y_t(w)\mid x}}{e_t(x, w)} + P_x \sum_w\frac{\pi^2(x, w)(\hat{\mu}_t(x, w)-\mu(x, w))^2}{e_t(x, w)} \\
    &\quad - P_x \Big(\sum_w \pi(x, w)\big(\hat{\mu}_t(x, w)-\mu(x, w)\big)\Big)^2 + P_x(1-P_x)Q_x^2
\end{aligned}
\end{equation*}
The numerator in \eqref{eq:eta2} can be decomposed as 
\begin{align*}
    Z_T =& \sum_{t=1}^T h_t^2(x) \Big\{P_x \sum_w\frac{\pi^2(x, w)\Var{}{Y_t(w)\mid x}}{e_t(x, w)} + P_x \sum_w\frac{\pi^2(x, w)(\hat{\mu}_t(x, w)-\mu(x, w))^2}{e_t(x, w)} \\
       & - P_x \Big(\sum_w \pi(x, w)\big(\hat{\mu}_t(x, w)-\mu(x, w)\big)\Big)^2 + P_x(1-P_x)Q_x^2\Big\}\\
       = & \sum_{t=1}^T h_t^2(x) \Big\{P_x \sum_w\frac{\pi^2(x, w)[\Var{}{Y_t(w)\mid x} + (\mu_\infty(x,w)-\mu(x,w))^2 ]}{e_t(x, w)} \\
       & - P_x \Big(\sum_w \pi(x, w)\big(\mu_\infty(x, w)-\mu(x, w)\big)\Big)^2 + P_x(1-P_x)Q_x^2\Big\}\\
       &  + \sum_{t=1}^T h_t^2(x)P_x\Big\{\sum_w\frac{\pi^2(x,w)}{e_t(x,w)}[(\hat{\mu}_t(x,w)-\mu_\infty(x,w))^2+2(\hat{\mu}_t(x,w)-\mu_\infty(x,w))(\mu_\infty(x,w)- \mu(x,w))]\\
       & -\big(\sum_w \pi(x,w)(\hat{\mu}_t(x, w) -\mu_\infty(x,w))\big)^2 - 2(\sum_w \pi(x,w)(\hat{\mu}_t(x, w) -\mu_\infty(x,w))(\sum_w \pi(x,w)(\mu_\infty(x, w) -\mu(x,w)) \Big\},
\end{align*}
which can be characterized as (to be shown shortly)
\begin{equation}
\label{eq:eta_ZT}
    \begin{split}
        Z_T =& \sum_{t=1}^T h_t^2(x) \sum_{w}\frac{C_1(x,w)\pi^2(x,w)}{e_t(x,w)} +C_2(x) \sum_{t=1}^T h_t^2(x)  + o_p\Big(\sum_{t=1}^T h_t^2(x) \Big\{\sum_{w}\frac{C_1(x,w)\pi^2(x,w)}{e_t(x,w)} +C_2(x)\Big\} \Big),
    \end{split}
\end{equation}
where $C_1(x,w) = P_x(\Var{}{Y_t(w)\mid x} + (\mu_\infty(x,w)-\mu(x,w))^2) $ and $C_2(x) = - P_x \Big(\sum_w \pi(x, w)\big(\hat{\mu}_t(x, w)-\mu(x, w)\big)\Big)^2 + P_x(1-P_x)Q_x^2$ do not depend on the history $H_{t-1}$. Note that $\sum_w \frac{C_1(x,w)\pi^2(x,w)}{e_t(x,w)}+C_2(x)\geq \sum_w \frac{\var(Y_t(w)|x)\pi^2(x,w)}{e_t(x,w)}$ by a calculation that is similar to the one done in \eqref{eq:C1C2}. Similarly, we have its expectation as 
\begin{equation}
\label{eq:eta_EZT}
    \begin{split}
        \bE[Z_T] =& \sum_{t=1}^T \bE[h_t^2(x) \sum_{w}\frac{C_1(x,w)\pi^2(x,w)}{e_t(x,w)}] +C_2(x) \sum_{t=1}^T \bE[ h_t^2(x)]  + o\Big(\sum_{t=1}^T\bE\Big[ h_t^2(x)\Big\{ \sum_{w}\frac{C_1(x,w)\pi^2(x,w)}{e_t(x,w)}+C_2(x)\Big\}\Big] \Big).
    \end{split}
\end{equation}
If we have \eqref{eq:eta_ZT} and \eqref{eq:eta_EZT}, to show \eqref{eq:eta2}, we only need to show 
\begin{equation}
\label{eq:eta2_target}
    \frac{\sum_{t=1}^T h_t^2(x) \sum_{w}\frac{C_1(x,w)\pi^2(x,w)}{e_t(x,w)} +C_2(x) \sum_{t=1}^T h_t^2(x)  }{\sum_{t=1}^T \bE[h_t^2(x) \sum_{w}\frac{C_1(x,w)\pi^2(x,w)}{e_t(x,w)}] +C_2(x) \sum_{t=1}^T \bE[ h_t^2(x)] }\xrightarrow{p}1.
\end{equation}

By Lemma~\ref{lemma:weight_convergence}, we have
\begin{equation}
    \sum_{t=1}^T h_t^2(x) \Big/ \bE\Big[\sum_{t=1}^T h_t^2(x)\Big]  \xrightarrow{p}1.
\end{equation}

In addition, by condition that $\bE[e_t^{-1}(x,w)]/e_t^{-1}(x,w) \xrightarrow{a.s.}1$, we have 
\begin{equation}
\label{eq:eta_h2c_as}
     h_t^2(x) \sum_{w}\frac{C_1(x,w)\pi^2(x,w)}{e_t(x,w)} \bigg/ \Big\{\bE\Big[\sum_{w}\frac{C_1(x,w)\pi^2(x,w)}{e_t(x,w)} \Big]\Big/\bE\Big[ \sum_{w}\frac{\pi^2(x,w)}{e_t(x,w)}\Big]\Big\} \xrightarrow{a.s.}1,
\end{equation}
by a caculation similar to \eqref{eq:ht_as_converge}.
 Recall the definition of \StableVar weights is $h_t(x) = 1/\sqrt{\sum_w\frac{\pi^2(x,w)}{e_t(x,w)}}$, we thus have $ h_t^2(x) \sum_{w}\frac{C_1(x,w)\pi^2(x,w)}{e_t(x,w)} $ is bounded above and away from zero, thus 
 
 \begin{equation}
 \label{eq:eta_h2c_L1}
     h_t^2(x) \sum_{w}\frac{C_1(x,w)\pi^2(x,w)}{e_t(x,w)} \bigg/ \Big\{\bE\Big[\sum_{w}\frac{C_1(x,w)\pi^2(x,w)}{e_t(x,w)} \Big]\Big/\bE\Big[ \sum_{w}\frac{\pi^2(x,w)}{e_t(x,w)}\Big]\Big\} \xrightarrow{L_1}1.
\end{equation}

Combining \eqref{eq:eta_h2c_as}, \eqref{eq:eta_h2c_L1}, and Lemma \ref{lemma:weight_convergence}, we have
\begin{align*}
     &\frac{\sum_{t=1}^T h_t^2(x) \sum_{w}\frac{C_1(x,w)\pi^2(x,w)}{e_t(x,w)} +C_2(x) \sum_{t=1}^T h_t^2(x)  }{\sum_{t=1}^T \bE[h_t^2(x) \sum_{w}\frac{C_1(x,w)\pi^2(x,w)}{e_t(x,w)}] +C_2(x) \sum_{t=1}^T \bE[ h_t^2(x)] } \\
     =& \frac{(\{\sum_{t=1}^T h_t^2(x) \sum_{w}\frac{C_1(x,w)\pi^2(x,w)}{e_t(x,w)} +C_2(x)h_t^2(x)\}  /M_T-1)M_T + M_T}{(\sum_{t=1}^T \bE[h_t^2(x) \sum_{w}\frac{C_1(x,w)\pi^2(x,w)}{e_t(x,w)}] +C_2(x) h_t^2(x)]/M_T-1)M_T+M_T }\\
     = & \frac{o_p(1)M_T+M_T}{o(1)M_T+M_T}\xrightarrow{p}1,
\end{align*}
where $M_T := \sum_{t=1}^T \bE\Big[\sum_{w}\frac{C_1(x,w)\pi^2(x,w)}{e_t(x,w)} + C_2(x) \Big]\big/\bE[ \sum_{w}\frac{\pi^2(x,w)}{e_t(x,w)}]$, proving \eqref{eq:eta2_target}.

\vspace{\baselineskip} 

To complete the proof, let's show  \eqref{eq:eta_ZT} and \eqref{eq:eta_EZT}. Define the next quantity, which collects the terms we'll prove to be asymptotically negligible relative to the dominant terms in these expressions:
\begin{equation}
\begin{split}
      &  m(x) := \sum_{t=1}^T h_t^2(x)P_x\Big\{\sum_w\frac{\pi^2(x,w)}{e_t(x,w)}[(\hat{\mu}_t(x,w)-\mu_\infty(x,w))^2+2(\hat{\mu}_t(x,w)-\mu_\infty(x,w))(\mu_\infty(x,w)- \mu(x,w))]\\
       & -\big(\sum_w \pi(x,w)(\hat{\mu}_t(x, w) -\mu_\infty(x,w))\big)^2 - 2(\sum_w \pi(x,w)(\hat{\mu}_t(x, w) -\mu_\infty(x,w))(\sum_w \pi(x,w)(\mu_\infty(x, w) -\mu(x,w)) \Big\}.
\end{split}
\end{equation}

For \eqref{eq:eta_ZT}, we need to show $m(x) = o_p(\sum_{t=1}^T h_t^2(x) \sum_{w}\frac{C_1(x,w)\pi^2(x,w)}{e_t(x,w)})$, which is true since each summand in $m_t$ is $o_p(h_t^2(x)\sum_w\frac{\pi^2(x,w)}{e_t(x,w)})$ since  $\sup_{x,w}|\hat{\mu}_{t}(x,w)-\mu_\infty(x,w)|\xrightarrow{a.s.}0$; also notice that $\sum_{t=1}^T h_t^2(x) \sum_{w}\frac{C_1(x,w)\pi^2(x,w)}{e_t(x,w)} = \Theta(T)$. Then invoking Proposition \ref{prop:weighted_convergence}, we have $m(x) / \big\{ \sum_{t=1}^T h_t^2(x) \sum_{w}\frac{C_1(x,w)\pi^2(x,w)}{e_t(x,w)}\big\} = o_p(1)$ and thus \eqref{eq:eta_ZT}.

\vspace{\baselineskip}

For \eqref{eq:eta_EZT}, we  need to show $\bE[m(x)] = o(\bE[\sum_{t=1}^T h_t^2(x) \sum_{w}\frac{C_1(x,w)\pi^2(x,w)}{e_t(x,w)}])$, which is true
since each summand in $\bE[m_t]$ is $o(\bE[h_t^2(x)\sum_w\frac{\pi^2(x,w)}{e_t(x,w)}])$ since  $\sup_{x,w}|\hat{\mu}_{t}(x,w)-\mu_\infty(x,w)|\xrightarrow{L_1}0$; also notice that $\bE[\sum_{t=1}^T h_t^2(x) \sum_{w}\frac{C_1(x,w)\pi^2(x,w)}{e_t(x,w)}] = \Theta(T)$. Then invoking Proposition \ref{prop:weighted_convergence_2}, we have \\$\bE[m(x)] / \big\{ \bE[\sum_{t=1}^T h_t^2(x) \sum_{w}\frac{C_1(x,w)\pi^2(x,w)}{e_t(x,w)}]\big\} = o(1)$ and thus \eqref{eq:eta_EZT}.

\begin{remark}
$\sum_{t=1}^T \bE[\eta_t^2(x)|H_{t-1}]$ is bounded by derivation in \eqref{eq:eta_second_moment}, we thus have $\sum_{t=1}^T \bE[\eta_t^2(x)|H_{t-1}]\xrightarrow{L_p}1$ for any $p>0$.
\end{remark}

\paragraph{Moment decay.}
A routine calculation (similar to what we have done in Appendix~\ref{appendix:clt_xi}) leads to 
\begin{align*}
    \condE{}{(\one\{X_t=x\}\widehat{\Gamma}_t(x, \pi)  - P_xQ_x)^4}{H_{t-1}} &\leq c\cdot \Big(
    1+\sum_w \frac{\pi^2(x,w)}{e_t(x,w)} +\sum_w \frac{\pi^3(x,w)}{e_t^2(x,w)}+\sum_w \frac{\pi^4(x,w)}{e_t^3(x,w)}
    \Big)\leq c'\cdot\sum_w \frac{\pi^4(x,w)}{e_t^3(x,w)},
\end{align*}
for some constants $c', c$. We thus have 
\begin{align*}
    \sum_{t=1}^T  \condE{}{\eta^4_{T,t}(x)}{H_{t-1}} &= \frac{\sum_{t=1}^T h_t^4(x)  \condE{}{(\one\{X_t=x\}\widehat{\Gamma}_t(x, \pi)  - P_xQ_x)^4}{H_{t-1}}}{ 
    \big( \sum_{t=1}^T \bE[h_t(x)^2 (\one\{X_t=x\}\widehat{\Gamma}_t(x, \pi)  - P_xQ_x)^2]\big)^2
    }\\
    &\lesssim \frac{\sum_{t=1}^T \sum_w \frac{\pi^4(x,w)}{e_t^3(x,w)}\big/(\sum_w \frac{\pi^2(x,w)}{e_t(x,w)})^2 }{T^2}\\
    &\leq \frac{\sum_{t=1}^T \sum_w \frac{\pi^4(x,w)}{e_t^3(x,w)}\big/\sum_w \frac{\pi^4(x,w)}{e_t^2(x,w)} }{T^2} \lesssim \frac{\sum_{t=1}^T t^{\alpha}}{T^2} = O(T^{\alpha-1})\xrightarrow{a.s.} 0.
\end{align*}
Invoking Proposition~\ref{prop:martingale_clt}, we have $\sum_{t=1}^T \eta_{T,t}(x)\xrightarrow{d}\N(0,1).$

\subsubsection{Step II: CLT of aggregated MDS \texorpdfstring{$\zeta_{T,t}$}{MDS-C}}
We set out to show that we can aggregate our auxiliary MDS $\eta_{T,t}(x)$ for each content $x$ into an asymptotically normal studentized statistic $\zeta_{T,t}$ for the value of the policy across $x$. To do so, we will again prove that $\zeta_{T,t}$ satisfies the conditions in Proposition~\ref{prop:martingale_clt}. Recall that $ \zeta_{T, t}  =  \sum_x \eta_{T, t}(x) \alpha_T(x)/ (V_T^C)^{-1/2}$, where $\alpha_T(x)$ and $V_T^C$ are defined in \eqref{eq:alpha_T} and \eqref{eq:vc} respectively. We in addition introduce the lemma below that will become handy in the subsequent proof. 

\begin{lemma}\label{lemma:diff-p-conv}
Suppose  $A_T/\E{}{A_T}\xrightarrow{p}1, B_T/\E{}{B_T}\xrightarrow{p}1 $. Also, suppose there exist constants $M_a,M_b>0$ such that $|\E{}{A_T} - \E{}{B_T}|\geq M_a|\E{}{A_T}|$ and $|\E{}{A_T} - \E{}{B_T}|\geq M_b|\E{}{B_T}|$, then we have 
$(A_T-B_T)/\E{}{A_T-B_T}\xrightarrow{p}1$.
\end{lemma}

\begin{proof}
 We have
 \begin{align*}
    \Big| \frac{A_T-B_T}{\E{}{A_T-B_T}} - 1\Big|\leq  \Big| \frac{A_T-\E{}{A_T}}{\E{}{A_T-B_T}}\Big| + \Big| \frac{B_T-\E{}{B_T}}{\E{}{A_T-B_T}}\Big| \leq \frac{1}{M_a}\Big|\frac{A_T}{\bE[A_T]} - 1 \Big| + \frac{1}{M_b}\Big|\frac{B_T}{\bE[B_T]} - 1 \Big|,
 \end{align*}
 which concludes the proof.
\end{proof}

\paragraph{Variance convergence.} We will invoke Lemma~\ref{lemma:diff-p-conv} to prove the result $\sum_{t=1}^T \bE[\zeta_{T,t}^2|H_{t-1}]\xrightarrow{p}1$.  In fact, the numerator of $\sum_{t=1}^T \bE[\zeta_{T,t}^2|H_{t-1}]$ can be written as $A_T-B_T$, and its denominator is $\bE[A_T - B_T]$, where
\begin{equation}
\label{eq:at_bt}
    \begin{split}
       & A_T = \sum_x \frac{\sum_{t=1}^T h_t^2(x)
        \bE[(\one\{X_t=x\}\widehat{\Gamma}_t(x, \pi) -P_xQ_x
)^2|H_{t-1}]}{\bE[\sum_{t=1}^T h_t(x)]^2}, \\
 &B_T = \sum_{x\neq y} \frac{ \sum_{t=1}^T
         h_t(x)h_t(y)P_xP_yQ_xQ_y
        }{\bE[\sum_{t=1}^T h_t(x)]\bE[\sum_{t=1}^T h_t(y)].
        }.
    \end{split}
\end{equation}
It suffices to verify that $A_T, B_T$ satisfy conditions in  Lemma~\ref{lemma:diff-p-conv}.
\begin{itemize}
    \item \emph{Convergence $A_T/\bE[A_T]\xrightarrow{p}1$.} By variance convergence of $\eta_{T,t}(x)$, we have for each $x\in\X$,
\begin{align*}
    \frac{\sum_{t=1}^T h_t^2(x)
        \bE[(\one\{X_t=x\}\widehat{\Gamma}_t(x, \pi) -P_xQ_x
)^2|H_{t-1}]}{\bE[\sum_{t=1}^T h_t(x)]^2} \Big/  \frac{ \sum_{t=1}^T
        \bE[h_t^2(x)(\one\{X_t=x\}\widehat{\Gamma}_t(x, \pi) -P_xQ_x
)^2] }{\bE[\sum_{t=1}^T h_t(x)]^2 } \xrightarrow{p}1.
\end{align*}
    
    \item \emph{Convergence $B_T/\bE[B_T]\xrightarrow{p}1$.} By Lemma \ref{lemma:weight_convergence}, for each $x\neq y$ we have
\begin{align*}
   &  \frac{\sum_{t=1}^T h_t(x)h_t(y) }{\bE[\sum_{t=1}^T h_t(x)h_t(y)]}\xrightarrow{p}1\\
    \Rightarrow \quad & \frac{ \sum_{t=1}^T
         h_t(x)h_t(y)P_xP_yQ_xQ_y
        }{\bE[\sum_{t=1}^T h_t(x)]\bE[\sum_{t=1}^T h_t(y)]
        }\Big/ \frac{ \sum_{t=1}^T \bE[h_t(x)h_t(y)]P_xP_yQ_xQ_y }{\bE[\sum_{t=1}^T h_t(x)]\bE[\sum_{t=1}^T h_t(y)]} \xrightarrow{p}1.
\end{align*}
    
    \item \emph{Show that $|A_T-B_T|\geq M_a A_T$ for some constant $M_a$.} 
       \begin{align*}
       & A_T- B_T \\
             =& \sum_x \frac{\sum_{t=1}^T h_t^2(x)
        \bE[(\one\{X_t=x\}\widehat{\Gamma}_t(x, \pi) -P_xQ_x
)^2|H_{t-1}, X_t=x]}{\bE[\sum_{t=1}^T h_t(x)]^2}
           - \sum_{x\neq y} \frac{ \sum_{t=1}^T
         h_t(x)h_t(y)P_xP_yQ_xQ_y
        }{\bE[\sum_{t=1}^T h_t(x)]\bE[\sum_{t=1}^T h_t(y)]
        }\\
          = &\sum_x \frac{\sum_{t=1}^T h_t^2(x)P_x
        \bE[(\widehat{\Gamma}_t(x, \pi) -Q_x
)^2|H_{t-1}, X_t=x]}{\bE[\sum_{t=1}^T h_t(x)]^2} +\sum_x \frac{\sum_{t=1}^T h_t^2(x)P_x(1-P_x)
        Q_x^2}{\bE[\sum_{t=1}^T h_t(x)]^2}
           \\
           &\quad \qquad- \sum_{x\neq y} \frac{ \sum_{t=1}^T
         h_t(x)h_t(y)P_xP_yQ_xQ_y
        }{\bE[\sum_{t=1}^T h_t(x)]\bE[\sum_{t=1}^T h_t(y)]
        }\\
         =& \sum_x \frac{\sum_{t=1}^T h_t^2(x)P_x
        \bE[(\widehat{\Gamma}_t(x,\pi ) -Q_x
)^2|H_{t-1}, X_t=x]}{\bE[\sum_{t=1}^T h_t(x)]^2} + \sum_{t=1}^T \sum_{x\neq y}P_xP_y\Big(
        \frac{
        h_t(x)Q_x
        }{\bE[\sum_{t=1}^T h_t(x)]} - \frac{
        h_t(y)Q_y
        }{\bE[\sum_{t=1}^T h_t(y)]}
        \Big)^2\\
           & \geq \sum_{t=1}^T \sum_x\frac{h^2_t(x)P_x\bE[(\widehat{\Gamma}_t(x, \pi) -Q_x
)^2|H_{t-1}, X_t=x]}{
\bE[\sum_{t=1}^T h_t(x)]^2 } \gtrsim \sum_{t=1}^T \sum_x\frac{h^2_t(x)P_x}{
\bE[\sum_{t=1}^T h_t(x)]^2 }
      \end{align*}
      Thus, we have
\begin{align*}
         A_T- B_T & \geq \sum_{t=1}^T \sum_x\frac{h^2_t(x)P_x\bE[(\widehat{\Gamma}_t(x, \pi) -Q_x
)^2|H_{t-1}, X_t=x]}{
\bE[\sum_{t=1}^T h_t(x)]^2 } \\
 &\gtrsim \sum_x \frac{\sum_{t=1}^T h_t^2(x)
        \bE[(\one\{X_t=x\}\widehat{\Gamma}_t(x, \pi) -P_xQ_x
)^2|H_{t-1}]}{\bE[\sum_{t=1}^T h_t(x)]^2}= A_T.
\end{align*}

    \item \emph{Show that $|A_T-B_T|\geq M_b B_T$ for some constant $M_b$.} Similarly, 
\begin{align*}
    A_T-B_T\gtrsim \sum_{t=1}^T \sum_x\frac{h^2_t(x)P_x}{
\bE[\sum_{t=1}^T h_t(x)]^2 }\gtrsim \sum_{t=1}^T \sum_x\frac{h^2_t(x)P_x(1-P_x) Q_x^2}{
\bE[\sum_{t=1}^T h_t(x)]^2 }\\
\stackrel{\mbox{(i)}}{\geq} \sum_{t=1}^T \sum_{x\neq y}\frac{h_t(x)h_t(y)P_xP_yQ_xQ_y }{
\bE[\sum_{t=1}^T h_t(x)] \bE[\sum_{t=1}^T h_t(y)] } = B_T,
\end{align*}
where (i) is by Cauchy-Schwartz inequality.
\end{itemize}

\paragraph{Moment decay.} We proceed to show that the fourth moment of $\sum_{t=1}^T\zeta_{T,t}=\sum_x\frac{\alpha_T(x)}{\sqrt{V_T^C}} \sum_{t=1}^T\eta_{T,t}(x)$ decays. To start, note that the aggregation weight $\frac{\alpha_T(x)}{\sqrt{V_T^C}} = \sqrt{\frac{\bE[A_T]}{\bE[A_T]-\bE[B_T]}}$ is bounded, where $A_T, B_T$ are defined in \eqref{eq:at_bt}. The conditional 4th moment of $\zeta_{T,t}$ can be bounded as follows.
\begin{equation}
    \label{eq:zeta_lyap}
    \begin{split}
    \sum_{t=1}^T\bE\Big[\zeta_{T,t}^4|H_{t-1}\Big]^{1/4}
    &= \sum_{t=1}^T\bE\Big[\Big(\sum_x \eta_{T,t}(x) \frac{\alpha_T(x)}{\sqrt{V_T^C}}\Big)^4
    |H_{t-1}\Big]^{1/4} \\
    &\leq \sum_{t=1}^T\sum_x 
     \frac{\alpha_T(x)}{\sqrt{V_T^C}}\left(\condE{}{\eta_{T,t}^4 (x)}{H_{t-1}} \right)^{1/4} \\
    &\lesssim  \sum_x \sum_{t=1}^T
    \condE{}{\eta_{T,t}^4 (x)}{H_{t-1}}^{1/4} \\
    &\leq |\X| \max_{x'} \Big(\sum_{t=1}^T\condE{}{\eta_{T,t}^4 (x')}{H_{t-1}}\Big)^{1/4} \xrightarrow{a.s.}0,
    \end{split}
\end{equation}
where the first inequality is due to Minkowski inequality, the second uses the fact that $\alpha_T(x)/\sqrt{V_T^C}$ is bounded by above; the limit follows from the previous subsection.

Therefore, by Proposition~\ref{prop:martingale_clt}, we have that $\sum_{t=1}^T \zeta_{T,t}\xrightarrow{d} \N(0,1)$.

\subsubsection{Step III: CLT of \texorpdfstring{$\widehat{Q}_T^{C}$}{Qhat-C}}
We now connect the convergence from $\zeta_{T,t}$ to $\widehat{Q}_T^C$. This procedure decomposes into two parts. 
\begin{enumerate}[(a)]
    \item From  $\sum_{t=1}^T \zeta_{T,t}\xrightarrow{d}\N(0,1)$ to $(\widehat{Q}^C_T(\pi)  - Q(\pi))\big/{\sqrt{V^C_T(\pi) }}\xrightarrow{d}\N(0,1)$.
    \item From $(\widehat{Q}^C_T(\pi)  - Q(\pi))\big/{\sqrt{V^C_T(\pi) }}\xrightarrow{d}\N(0,1)$ to  $(\widehat{Q}^C_T(\pi)  - Q(\pi))\big/{\sqrt{\widehat{V}^C_T(\pi) }}\xrightarrow{d}\N(0,1)$.
\end{enumerate}

\paragraph{Part (a).} It suffices to show the difference is vanishing,
\begin{equation}
\begin{split}
    &\sum_{t=1}^T \zeta_{T,t} - \frac{\widehat{Q}^C_T(\pi)  - Q(\pi)}{\sqrt{V^C_T(\pi) }} \stackrel{\Delta}{=} \sum_x \delta(x),\quad \mbox{where}\\
    & \delta(x) = \frac{\sum_{t=1}^T h_t(x)(\one\{X_t=x\}\widehat{\Gamma}_t(x, \pi)  - P_xQ_x)  }{ \sqrt{V^C_T(\pi) }\sum_{t=1}^T \E{}{h_t(x)} } \bbracket{1-\frac{\sum_{t=1}^T \E{}{h_t(x)} }{\sum_{t=1}^T h_t(x) } }.
\end{split}
\end{equation} 
We have
\[
    \begin{split}
        |\delta(x)| &\stackrel{(i)}{\lesssim} \left|
    \frac{\sum_{t=1}^T h_t(x)(\one\{X_t=x\}\widehat{\Gamma}_t(x, \pi)  - P_xQ_x)  }{  \sqrt{\sum_y \alpha^2_T(y)} \sum_{t=1}^T \E{}{h_t(x)} }\right| \cdot\left|1-\frac{\sum_{t=1}^T \E{}{h_t(x)} }{\sum_{t=1}^T h_t(x)} \right|\\
       &\leq \left|\frac{\sum_{t=1}^T h_t(x)(\one\{X_t=x\}\widehat{\Gamma}_t(x, \pi)  - P_xQ_x)  }{ \sqrt{ \alpha^2_T(x)} \sum_{t=1}^T \E{}{h_t(x)} }\right|\cdot \left|1-\frac{\sum_{t=1}^T \E{}{h_t(x)} }{\sum_{t=1}^T h_t(x) } \right|\\
       & = 
      \left| \sum_{t=1}^T \eta_{T,t}(x)\right|\cdot \left|1-\frac{\sum_{t=1}^T \E{}{h_t(x)} }{\sum_{t=1}^T h_t(x) }  \right|\xrightarrow{p}0,
    \end{split}
\]
where in (i), we use that $V_T^C=\bE[A_T-B_T]]\gtrsim \bE[A_T]=\sum_{y}\alpha_T^2(y)$ (where $A_T, B_T$ defined in \eqref{eq:at_bt}); and in the last convergence statement, we use the fact that $ \sum_{t=1}^T \eta_{T,t}(x)$ is asymptotically normal and thus $O_p(1)$, and that  with \texttt{StableVar} weights $\sum_{t=1}^T \bE[h_t(x)] / \sum_{t=1}^T h_t(x) \xrightarrow{p}1$ by Lemma \ref{lemma:weight_convergence}.
Together we have $\delta(x)$ is vanishing.

\paragraph{Part (b).} It suffices to show that the variance estimator converges to its expectation: $\widehat{V}_T^C/{V}_T^C \xrightarrow{p}1$. Expanding the definitions of $\widehat{V}_T^C$ and ${V}_T^C$, we have
\begin{align*}
    \widehat{V}_T^C(\pi)  =& \sum_x \underbracket[0.4pt]{ \frac{\sum_{t=1}^T 
        h_t^2(x)\big(\one\{X_t=x\}\hat{\Gamma}_t(x, \pi) - (\widehat{P_x  Q_x})_T\big)^2
        }{\big(\sum_{t=1}^Th_t(x)\big)^2}}_{\widehat{I}(x)} +\sum_{x\neq y}
   \underbracket[0.4pt]{\frac{\sum_{t=1}^T h_t(x)h_t(y) (\widehat{P_x  Q_x})_T(\widehat{P_y Q_y})_T }{
    \big(\sum_{t=1}^T h_t(x))\big(\sum_{t=1}^Th_t(y) )
    }}_{\widehat{II_a}(x,y)} 
    \\
    &\quad -\sum_{x\neq y} \underbracket[0.4pt]{\frac{ \sum_{t=1}^T h_t(x)h_t(y) \one\{X_t=x\}\hat{\Gamma}_t(x, \pi)  (\widehat{P_y Q_y})_T }{
     \big(\sum_{t=1}^T h_t(x)\big)\big( \sum_{t=1}^Th_t(y)\big)
    } }_{\widehat{II_b}(x,y)}  -\sum_{x\neq y}\underbracket[0.4pt]{ \frac{ \sum_{t=1}^T h_t(x)h_t(y) \one\{X_t=y\}\hat{\Gamma}_t(x, \pi) (\widehat{P_x  Q_x})_T }{
     \big(\sum_{t=1}^T h_t(x)\big)\big( \sum_{t=1}^Th_t(y)\big)
    } }_{\widehat{II_c}(x,y)}.\\
             V_T^C(x, \pi)  =&\sum_x  \underbracket[0.4pt]{ \frac{\sum_{t=1}^T\E{}{
        h_t^2(x)\big(\one\{X_t=x\}\hat{\Gamma}_t(x, \pi) - P_xQ_x\big)^2
      }}{\E{}{\sum_{t=1}^Th_t(x)}^2}}_{I(x)}
       - \sum_{x\neq y}\underbracket[0.4pt]{\frac{\sum_{t=1}^T \E{}{h_t(x)h_t(y)}P_xQ_xP_y Q_y }{\big(\sum_{t=1}^T \E{}{h_t(x)}\big)\bbracket{ \sum_{t=1}^T\E{}{h_t(y)} }}}_{II(x,y)}
\end{align*}
In fact, for any $x,y\in \X$, we have $\frac{\widehat{I}(x)}{I(x)}, \frac{\widehat{II_a}(x,y)}{ II(x,y)}, \frac{\widehat{II_b}(x,y)}{II(x,y)}, \frac{\widehat{II_c}(x,y)}{II(x,y)}$ all converge to $1$ in probability (to be shown shortly). Also, recall that $\sum_x I(x) = \bE[A_T]$ and $\sum_{x\neq y}II(x,y) = \bE[B_T]$ (where $A_T, B_T$ are defined in \eqref{eq:at_bt}); invoking Lemma~\ref{lemma:diff-p-conv} yields $\widehat{V}_T^C/{V}_T^C \xrightarrow{p}1$.

We now show that $\widehat{I}(x), \widehat{II_a}(x,y), \widehat{II_b}(x,y), \widehat{II_c}(x,y)$ converge to $I(x), II(x,y), II(x,y), II(x,y)$ respectively. 
\begin{itemize}
    \item Show $\widehat{I}(x) / I(x) \xrightarrow{p}1$.
    \begin{align*}
        \frac{\widehat{I}(x)}{I(x)} = \frac{\sum_{t=1}^T 
        h_t^2(x)\big(\one\{X_t=x\}\hat{\Gamma}_t(x,\pi) - (\widehat{P_x  Q_x})_T\big)^2
        }{\sum_{t=1}^T\E{}{
        h_t^2(x)\big(\one\{X_t=x\}\hat{\Gamma}_t(x,\pi) - P_xQ_x\big)^2
      }} \frac{\E{}{\sum_{t=1}^Th_t(x)}^2}{\big(\sum_{t=1}^Th_t(x)\big)^2} .
    \end{align*}
    For \StableVar weights, by Lemma~\ref{lemma:weight_convergence}, $\frac{\E{}{\sum_{t=1}^Th_t(x)}^2}{\big(\sum_{t=1}^Th_t(x)\big)^2} \xrightarrow{p}1$.
    On the other hand, following the same steps as in Appendix~\ref{appendix:clt_nc_final}, one can show that
    \begin{align}
    \label{eq:x_var_consistent}
        \frac{\sum_{t=1}^T 
        h_t^2(x)\big(\one\{X_t=x\}\hat{\Gamma}_t(x, \pi) - (\widehat{P_x  Q_x})_T\big)^2
        }{\sum_{t=1}^T\E{}{
        h_t^2(x)\big(\one\{X_t=x\}\hat{\Gamma}_t(x, \pi) - P_xQ_x\big)^2
      }}\xrightarrow{p}1,
    \end{align}
    by observing that $(\widehat{P_x  Q_x})_T$ is consistent to $P_xQ_x$ (shown in Appendix~\ref{appendix:consistency_of_contextual_weighting}), and that $\sum_{t=1}^T \bE[\eta_{T,t}^2|H_{t-1}]\xrightarrow{L_1}1$, $\sum_{t=1}^T \bE[\eta_{T,t}^4|H_{t-1}]\xrightarrow{a.s.}0$; we defer the \hyperlink{proof:theproof}{proof} to later text. 
    
    \item Show $\widehat{II_a}(x,y) / II(x,y)\xrightarrow{p}1$. By Lemma \ref{lemma:weight_convergence}, with \StableVar weights, we have 
    \begin{align*}
      \frac{  \sum_{t=1}^{T}h_t(x)}{ \sum_{t=1}^{T}\bE[h_t(x) ] } \xrightarrow{p}1,\quad  \frac{  \sum_{t=1}^{T}h_t(x)h_t(y) }{ \sum_{t=1}^{T}\bE[h_t(x)h_t(y) ] }  \xrightarrow{p}1.
    \end{align*}
    Meanwhile, $(\widehat{P_x  Q_x})_T, (\widehat{P_y Q_y})_T$ are consistent for $P_xQ_x, P_y Q_y$ respectively, yielding the desired result.
    
    \item Show $\widehat{II_b}(x,y) / II(x,y)\xrightarrow{p}1$. We have
    \begin{align*}
        \frac{\widehat{II_b}(x,y) }{II(x,y)} &= \frac{ \sum_{t=1}^T h_t(x)h_t(y) \one\{X_t=x\}\hat{\Gamma}_t(x, \pi)  (\widehat{P_y Q_y})_T }{
     \big(\sum_{t=1}^T h_t(x)\big)\big( \sum_{t=1}^Th_t(y)\big)
    }  \Big / \frac{\sum_{t=1}^T \E{}{h_t(x)h_t(y)}P_xQ_xP_y Q_y }{\big(\sum_{t=1}^T \E{}{h_t(x)}\big)\bbracket{ \sum_{t=1}^T\E{}{h_t(y)} }} \\
    & = \frac{\sum_{t=1}^T\E{}{h_t(x)}}{\sum_{t=1}^Th_t(y)} \cdot  \frac{\sum_{t=1}^T\E{}{h_t(y)}}{\sum_{t=1}^Th_t(y)} \cdot \frac{ (\widehat{P_y Q_y})_T }{P_y Q_y} \cdot \frac{\sum_{t=1}^Th_t(x)h_t(y)}{\sum_{t=1}^T \bE[h_t(x)h_t(y)]}\\
    &\quad \quad \times \frac{ \sum_{t=1}^T h_t(x)h_t(y) \one\{X_t=x\}\hat{\Gamma}_t(x, \pi) }{\sum_{t=1}^T h_t(x)h_t(y)} \Big/ \Big(P_xQ_x  \Big).
    \end{align*}
    We claim that 
    \begin{equation}
    \label{eq:cross_weights_consistency}
        \frac{ \sum_{t=1}^T h_t(x)h_t(y) \one\{X_t=x\}\hat{\Gamma}_t(x, \pi) }{\sum_{t=1}^T h_t(x)h_t(y)}\xrightarrow{p} P_xQ_x,
    \end{equation}
    which concludes $\widehat{II_b}(x,y) / II(x,y)\xrightarrow{p}1$. We defer the \hyperlink{proof:theproof2}{proof} of \eqref{eq:cross_weights_consistency} to later text.
    
        \item Show $\widehat{II_c}(x,y) / II(x,y) \xrightarrow{p}1$, which holds by the same argument applied to showing $\widehat{II_b}(x,y) / II(x,y) \xrightarrow{p}1$ above.
\end{itemize}

\vspace{1cm}

\begin{proof}[\hypertarget{proof:theproof}{Proof of \eqref{eq:x_var_consistent}}]
We now prove \eqref{eq:x_var_consistent}. 
\begin{align*}
   & \frac{\sum_{t=1}^T 
        h_t^2(x)\big(\one\{X_t=x\}\hat{\Gamma}_t(x, \pi) - (\widehat{P_x  Q_x})_T\big)^2
        }{\sum_{t=1}^T\E{}{
        h_t^2(x)\big(\one\{X_t=x\}\hat{\Gamma}_t(x, \pi) - P_xQ_x\big)^2
      }}\\
=&\underbracket[0.4pt]{ \frac{\sum_{t=1}^T 
        h_t^2(x)\big(\one\{X_t=x\}\hat{\Gamma}_t(x, \pi) - P_xQ_x\big)^2
        }{\sum_{t=1}^T\E{}{
        h_t^2(x)\big(\one\{X_t=x\}\hat{\Gamma}_t(x, \pi) - P_xQ_x\big)^2
      }}}_{A} + \underbracket[0.4pt]{\frac{\sum_{t=1}^T 
        h_t^2(x)\big(P_xQ_x - (\widehat{P_x  Q_x})_T\big)^2
        }{\sum_{t=1}^T\E{}{
        h_t^2(x)\big(\one\{X_t=x\}\hat{\Gamma}_t(x, \pi) - P_xQ_x\big)^2
      }} }_{B}\\
      &+\underbracket[0.4pt]{ 2 \frac{\sum_{t=1}^T 
        h_t^2(x)\big(P_xQ_x - (\widehat{P_x  Q_x})_T\big)\big(\one\{X_t=x\}\hat{\Gamma}_t(x, \pi) - P_xQ_x\big)
        }{\sum_{t=1}^T\E{}{
        h_t^2(x)\big(\one\{X_t=x\}\hat{\Gamma}_t(x, \pi) - P_xQ_x\big)^2
      }}}_{C}
\end{align*}
\begin{itemize}
    \item \emph{Term A.} We show this term converges to $1$ in $L_1$ and thus in probability.
\begin{align*}
&  \bE\big[\Big|  \frac{\sum_{t=1}^T 
        h_t^2(x)\big(\one\{X_t=x\}\hat{\Gamma}_t(x, \pi) - P_xQ_x\big)^2
        }{\sum_{t=1}^T\E{}{
        h_t^2(x)\big(\one\{X_t=x\}\hat{\Gamma}_t(x, \pi) - P_xQ_x\big)^2
      }} \Big|\big]\\
\leq & \bE[|\sum_{t=1}^T \bE[\eta_{T,t}^2|H_{t-1}]|]\\
      &\quad   +\bE\Big[\Big|  \frac{\sum_{t=1}^T 
        h_t^2(x)\big(\one\{X_t=x\}\hat{\Gamma}_t(x, \pi) - P_xQ_x\big)^2
        - \sum_{t=1}^T 
        \bE[h_t^2(x)\big(\one\{X_t=x\}\hat{\Gamma}_t(x, \pi) - P_xQ_x\big)^2|H_{t-1}]
        }{\sum_{t=1}^T\E{}{
        h_t^2(x)\big(\one\{X_t=x\}\hat{\Gamma}_t(x, \pi) - P_xQ_x\big)^2
      }} \Big|\Big] \\
      \leq & \bE[|\sum_{t=1}^T \bE[\eta_{T,t}^2|H_{t-1}]|]+\bE\big[  \frac{\sum_{t=1}^T 
        h_t^4(x)\big(\one\{X_t=x\}\hat{\Gamma}_t(x, \pi) - P_xQ_x\big)^4
        }{\sum_{t=1}^T\E{}{
        h_t^2(x)\big(\one\{X_t=x\}\hat{\Gamma}_t(x, \pi) - P_xQ_x\big)^4
      }} \big]^{1/2}\\
      = &\bE[|\sum_{t=1}^T \bE[\eta_{T,t}^2|H_{t-1}]|] + \bE[\sum_{t=1}^T \bE[\eta_{T,t}^4|H_{t-1}]]^{1/2}\rightarrow 1.
\end{align*}
 
    \item \emph{Term B.} We show this term converges to $0$ in probability.
\begin{align*}
    \frac{\sum_{t=1}^T 
        h_t^2(x)\big(P_xQ_x - (\widehat{P_x  Q_x})_T\big)^2
        }{\sum_{t=1}^T\E{}{
        h_t^2(x)\big(\one\{X_t=x\}\hat{\Gamma}_t- P_xQ_x\big)^2
      }}\lesssim \frac{T }{T}\big(P_xQ_x - (\widehat{P_x  Q_x})_T\big)^2 \xrightarrow{p}0.
\end{align*}

    \item \emph{Term C.} This term is vanished by Cauchy-Schwartz inequality. 
\end{itemize}

\noindent Collectively, we conclude the proof and have
\begin{equation*}
    \frac{\sum_{t=1}^T 
        h_t^2(x)\big(\one\{X_t=x\}\hat{\Gamma}_t(x, \pi) - (\widehat{P_x  Q_x})_T\big)^2
        }{\sum_{t=1}^T\E{}{
        h_t^2(x)\big(\one\{X_t=x\}\hat{\Gamma}_t(x, \pi) - P_xQ_x\big)^2
      }} \xrightarrow{p}1.
\end{equation*}
\end{proof}

\vspace{1cm}
\begin{proof}[\hypertarget{proof:theproof2}{Proof of \eqref{eq:cross_weights_consistency}}]
    We have
    \begin{equation}
           \begin{aligned}
         &\Big| \frac{ \sum_{t=1}^T h_t(x)h_t(y) \one\{X_t=x\}\hat{\Gamma}_t(x, \pi) }{\sum_{t=1}^T h_t(x)h_t(y)} - P_xQ_x\Big |=  \Big|\frac{
         \sum_{t=1}^T h_t(x)h_t(y)(\one\{X_t=x\}\hat{\Gamma}_t(x, \pi)  - P_xQ_x  )
        }{\sum_{t=1}^T h_t(x)h_t(y)} \Big|\\
        = & \Big| \frac{
         \sum_{t=1}^T h_t(x)h_t(y)(\one\{X_t=x\}\hat{\Gamma}_t(x, \pi)  - P_xQ_x  )
        }{\sqrt{\bE[\sum_{t=1}^T h_t^2(x)h_t^2(y)(\one\{X_t=x\}\hat{\Gamma}_t(x, \pi)  - P_xQ_x  )^2]}} \Big| \cdot \frac{\sqrt{\bE[\sum_{t=1}^T h_t^2(x)h_t^2(y)(\one\{X_t=x\}\hat{\Gamma}_t(x, \pi)  - P_xQ_x  )^2]}}{\sum_{t=1}^T h_t(x)h_t(y)} \\
        \stackrel{(i)}{\lesssim} & \Big| \frac{
         \sum_{t=1}^T h_t(x)h_t(y)(\one\{X_t=x\}\hat{\Gamma}_t(x, \pi)  - P_xQ_x  )
        }{\sqrt{\bE[\sum_{t=1}^T h_t^2(x)h_t^2(y)(\one\{X_t=x\}\hat{\Gamma}_t(x, \pi)  - P_xQ_x  )^2]}} \Big| \frac{\sqrt{T}}{\sum_{t=1}^T t^{-\alpha}}\\
        =  &  \Big|\frac{
         \sum_{t=1}^T h_t(x)h_t(y)(\one\{X_t=x\}\hat{\Gamma}_t(x, \pi)  - P_xQ_x  )
        }{\sqrt{\bE[\sum_{t=1}^T h^2_t(x)h^2_t(y)(\one\{X_t=x\}\hat{\Gamma}_t(x, \pi)  - P_xQ_x  )^2]}} \Big| O(T^{\alpha-\frac{1}{2}}),
    \end{aligned}
    \end{equation}
    where in (i)  we invoke \StableVar weights such that $h_t(x) = 1/\sqrt{\sum_w\pi^2(x,w)/e_t(x,w)}, h_t(y) = 1/\sqrt{\sum_w \pi^2(y,w)/e_t(y,w)}$ such that $h_t(x), h_t(y)\in [C\cdot t^{-\alpha/2}, 1]$. Thus, for any $\epsilon>0$,
    \begin{equation}
            \begin{aligned}
        & \p\Big(\Big|
        \frac{ \sum_{t=1}^T h_t(x)h_t(y) \one\{X_t=x\}\hat{\Gamma}_t(x, \pi) }{\sum_{t=1}^T h_t(x)h_t(y)} - P_xQ_x
        \Big|>\epsilon\Big) \leq \epsilon^{-2}\bE\Big[\Big|  \frac{ \sum_{t=1}^T h_t(x)h_t(y) \one\{X_t=x\}\hat{\Gamma}_t(x, \pi) }{\sum_{t=1}^T h_t(x)h_t(y)} - P_xQ_x\Big|^2 \Big]\\
        \leq & \epsilon^{-2}
       \bE\Big[ \Big|\frac{
         \sum_{t=1}^T h_t(x)h_t(y)(\one\{X_t=x\}\hat{\Gamma}_t(x, \pi)  - P_xQ_x  )
        }{\sqrt{\bE[\sum_{t=1}^T h^2_t(x)h^2_t(y)(\one\{X_t=x\}\hat{\Gamma}_t(x, \pi)  - P_xQ_x  )^2]}} \Big|^2 \Big]
       O(T^{2\alpha-1}) =  \epsilon^{-2}O(T^{2\alpha-1})\rightarrow 0,
    \end{aligned}
    \end{equation}
    where we use the condition that $\alpha<\frac{1}{2}$ in Assumption~\ref{assu:clt_condition}.
    
\end{proof}

\section{Classification dataset}
\label{sec:dataset}
The full list of the $86$  datasets from OpenML \cite{OpenML2013} is:

\hfill

\begin{verbatim}
 CreditCardSubset, 
 GAMETES_Epistasis_2-Way_20atts_0_4H_EDM-1_1,
 GAMETES_Heterogeneity_20atts_1600_Het_0_4_0_2_75_EDM-2_001, 
 LEV, Long, MagicTelescope, PhishingWebsites, PizzaCutter3, 
 SPECT, Satellite, abalone, allrep, artificial-characters, 
 autoUniv-au1-1000, balance-scale, banknote-authentication,
 blood-transfusion-service-center, boston, 
 boston_corrected, car, cardiotocography, chatfield_4, 
 chscase_census2, chscase_census6,
 cmc, coil2000, collins, credit-g, delta_ailerons, diabetes, dis,
 disclosure_x_noise, eeg-eye-state, eye_movements, fri_c0_1000_5,
 fri_c1_1000_25, fri_c1_1000_5, fri_c1_250_10, fri_c1_500_10,
 fri_c1_500_25, fri_c2_1000_25, fri_c2_1000_50, fri_c3_1000_25,
 fri_c3_250_10, fri_c3_250_5, fri_c3_500_5, fri_c4_1000_50, 
 haberman, heart-statlog, houses, ionosphere, 
 jEdit_4_2_4_3, jungle_chess_2pcs_endgame_elephant_elephant,
 jungle_chess_2pcs_endgame_panther_lion, kc3, kr-vs-kp, 
 mammography, mfeat-morphological, monks-problems-1, 
 monks-problems-2, monks-problems-3, nursery, oil_spill, 
 ozone_level, page-blocks, plasma_retinol, prnn_fglass, 
 qualitative-bankruptcy, ringnorm, rmftsa_sleepdata, 
 segment, solar-flare, spambase, splice, threeOf9,
 tic-tac-toe, vertebra-column, volcanoes-a1, 
 volcanoes-a2, volcanoes-b3, volcanoes-d4, volcanoes-e5, 
 wdbc, wilt, xd6, yeast
\end{verbatim}

\end{document}